\documentclass[11pt]{article}
\usepackage{fullpage}
\usepackage{amsmath,amsthm,amssymb}
\usepackage{algorithm}
\usepackage{algorithmic}
\usepackage{subfigure}
\usepackage{hyperref}
\usepackage{graphicx}
\usepackage{pbox}
\usepackage{multirow}
\usepackage{macro}
\usepackage{filecontents}

\RequirePackage{epsfig}
\RequirePackage{natbib}

\newcommand{\reals}{\mathbb{R}}

\def\by{{\bar y}}

\newcommand{\ba}{\begin{array}}
\newcommand{\ea}{\end{array}}
\newcommand{\beq}{\begin{equation}}
\newcommand{\eeq}{\end{equation}}
\newcommand{\beqa}{\begin{eqnarray}}
\newcommand{\eeqa}{\end{eqnarray}}
\newcommand{\beqas}{\begin{eqnarray*}}
\newcommand{\eeqas}{\end{eqnarray*}}
\newcommand{\bi}{\begin{itemize}}
\newcommand{\ei}{\end{itemize}}

\def\eqref#1{(\ref{#1})}

\def\bx{{\bar x}}
\def\by{{\bar y}}
\def\bu{{\bar u}}
\def\bv{{\bar v}}

\def\tx{{\tilde x}}

\def\hx{{\hat x}}
\def\hy{{\hat y}}
\def\bbx{{\mathbf x}}
\def\bby{{\mathbf y}}
\def\bbA{{\mathbf A}}
\def\bbB{{\mathbf B}}
\def\bbC{{\mathbf C}}
\def\bbD{{\mathbf D}}
\def\bbW{{\mathbf W}}
\def\bbZ{{\mathbf Z}}

\graphicspath{{figs/}}

\title{Doubly Stochastic Primal-Dual Coordinate Method for  Bilinear Saddle-Point Problem}
\date{}

\author{
Adams Wei Yu \\
Machine Learning Department\\
Carnegie Mellon University\\
Pittsburgh, PA 15213 \\
\texttt{weiyu@cs.cmu.edu} \\
\and
Qihang Lin \\
Tippie College of Business\\
The University of Iowa\\
Iowa City, IA 52245 \\
\texttt{qihang-lin@uiowa.edu}\\
\and
Tianbao Yang\\
Computer Science Department\\
The University of Iowa\\
Iowa City, IA 52245 \\
\texttt{tianbao-yang@uiowa.edu}
}

\begin{document}
\maketitle

\begin{abstract}
We propose a doubly stochastic primal-dual coordinate optimization algorithm for empirical risk minimization, which can be formulated as a bilinear saddle-point problem. In each iteration, our method randomly samples a block of coordinates of the primal and dual solutions to update. The linear convergence of our method could be established in terms of 1)
the distance from the current iterate to the optimal solution
and 2) the primal-dual objective gap. We show that the proposed method has a lower overall complexity than existing coordinate methods when either the data matrix has a factorized structure or the proximal mapping on each block is computationally expensive, e.g., involving an eigenvalue decomposition. The efficiency of the proposed method is confirmed by empirical studies on several real applications, such as the multi-task large margin nearest neighbor problem.
\end{abstract}

\section{Introduction}
We consider regularized \emph{empirical risk minimization} (ERM) problems of the following form:
\begin{eqnarray}
\label{eq:erm}
\min_{x\in\mathbb{R}^p}\left\{ P(x)\equiv\frac{1}{n}\sum_{i=1}^n\phi_i(a_i^Tx)+g(x)\right\},
\end{eqnarray}
where $a_1,\dots,a_n\in\mathbb{R}^p$ are $n$ data points with $p$ features, $\phi_i:\mathbb{R}\rightarrow \mathbb{R}$ is a convex loss function of the linear predictor $a_i^Tx$, for $i=1,\dots,n$, and $g:\mathbb{R}^p\rightarrow \mathbb{R}$ is a convex regularization function for the coefficient vector $x\in\mathbb{R}^p$ in the linear predictor. We assume $g$ has a \emph{decomposable structure}, namely,
\begin{eqnarray}
\label{eq:decomp_g}
g(x)=\sum_{j=1}^pg_j(x_j),
\end{eqnarray}
where $g_j:\mathbb{R}\rightarrow \mathbb{R}$ is only a function of $x_j$, the $j$-th coordinate of $x$. For simplicity, we consider a univariate $g_j$ at this moment.  In Section~\ref{sec:block}, the proposed method will be generalized for the problems having a block-wise decomposable structure with multivariate $g_j$.
We further make the following assumptions:
\begin{assumption}
\label{assume1}
For any $\alpha,\beta\in\mathbb{R}$,
\begin{itemize}
\item $g_j$ is $\lambda$-strongly convex for $j=1,2,\dots,p$, i.e.,
$g_j(\alpha)\geq g_j(\beta)+g'_j(\beta)(\alpha-\beta)+\frac{\lambda}{2}(\alpha-\beta)^2$;
\item $\phi_i$ is $(1/\gamma)$-smooth for $i=1,2,\dots,n$, i.e., $\phi_i(\alpha)\leq \phi_i(\beta)+\nabla\phi_i(\beta)(\alpha-\beta)+\frac{1}{2\gamma}(\alpha-\beta)^2$.
\end{itemize}
\end{assumption}

The problem \eqref{eq:erm} captures many applications in business analytics, statistics, machine learning and data mining, and has triggered many studies in the optimization community. Typically, for each data point $a_i$, there is an associated response value $b_i\in\mathbb{R}$, which can be continuous (in regression problems) or discrete (in classification problems). The examples of loss function 
$\phi_i(\cdot)$
associated to $(a_i,b_i)$ include:

\begin{itemize}
\setlength\itemsep{0em}
\item \textit{Square Loss}, where $a_i\in\mathbb{R}^p$, $b_i\in\mathbb{R}$ and $\phi_i(z)=\frac{1}{2}(z-b_i)^2$, which corresponds to \textit{linear regression} problem;
\item \textit{Sigmoid Loss}, where $a_i\in\mathbb{R}^p$, $b_i\in\{1,-1\}$ and $\phi_i(z)=\log(1+\exp(-b_iz))$, which corresponds to \textit{logistic regression} problem;
\item \textit{Smooth Hinge Loss}, where $a_i\in\mathbb{R}^p$, $b_i\in\{1,-1\}$ and
\begin{eqnarray}
\label{eq:ssvmloss}
\phi_i(z)=
\left\{
\begin{array}{ll}
0 & \text{if }b_iz\geq1\\
\frac{1}{2}-b_iz & \text{if }b_iz\leq0\\
\frac{1}{2}(1-b_iz)^2 & \text{otherwise.}
\end{array}
\right.
\end{eqnarray}
which corresponds to the smooth \textit{support vector machine} problem.
\end{itemize}
In fact, if appropriate reformulation is conducted, many other problems can also be reduced to \eqref{eq:erm}, for example, the \textit{multi-task large margin nearest neighbor metric learning} (MT-LMNN) problem (See Section \ref{sec:mt_lmnn}).

The commonly used regularization terms include the $\ell_2$-regularization $g_j(x)=\frac{\lambda x^2}{2}$ with $\lambda>0$ and $\ell_2+\ell_1$-regularization $g_j(x)=\frac{\lambda_2 x^2}{2}+\lambda_1|x|$ with $\lambda_1,\lambda_2>0$.

We often call \eqref{eq:erm} the \emph{primal problem} and its conjugate \emph{dual problem} is
\begin{eqnarray}
\label{eq:erm_dual}
\max_{y\in\mathbb{R}^n} \left\{D(y)\equiv -g^*\left(-\frac{A^Ty}{n}\right)-\frac{1}{n}\sum_{i=1}^n\phi_i^*(y_i)\right\},
\end{eqnarray}
where $A=[a_1,a_2,\dots,a_n]^T\in\mathbb{R}^{n\times p}$ and $\phi_i^*$ and $g^*$ are the convex conjugates of $\phi_i$ and $g$, respectively, meaning that $g^*(v)=\max_{u\in\mathbb{R}^p} \langle u,v\rangle-g(u)$ and $\phi_i^*(\alpha)=\max_{\beta\in\mathbb{R}} \alpha\beta-\phi_i(\beta)$. It is well-known in convex analysis that, under Assumption~\ref{assume1}, $g^*$ is $\frac{1}{\lambda}$-smooth and $\phi_i^*$ is $\gamma$-strongly convex.
In this paper, instead of considering purely \eqref{eq:erm} or \eqref{eq:erm_dual}, we are interested in their associated \emph{saddle-point} problem:
\begin{eqnarray}
\label{eq:sdp}
\min_{x\in\mathbb{R}^p}\max_{y\in\mathbb{R}^n}\left\{ g(x)+\frac{1}{n}y^TAx-\frac{1}{n}\sum_{i=1}^n\phi_i^*(y_i)\right\}.
\end{eqnarray}
Let $x^\star$ and $y^\star$ be the optimal solutions of \eqref{eq:erm} and \eqref{eq:erm_dual}, respectively. It is known that the pair $(x^\star,y^\star)$ is a \emph{saddle point} of \eqref{eq:sdp} in the sense that
\begin{eqnarray}
\label{eq:def_sp}
x^\star=\argmin_{x\in\mathbb{R}^p}\left\{g(x)+\frac{1}{n}(y^\star)^TAx-\frac{1}{n}\sum_{i=1}^n\phi_i^*(y_i^\star)\right\},\\
y^\star=\argmax_{y\in\mathbb{R}^n}\left\{g(x^\star)+\frac{1}{n}y^TAx^\star-\frac{1}{n}\sum_{i=1}^n\phi_i^*(y_i)\right\}.
\end{eqnarray}



The contributions of this paper can be highlighted as follow:
\begin{itemize}
\item We propose a \emph{doubly stochastic primal-dual coordinate} (DSPDC) method for solving problem~\eqref{eq:sdp} that randomly samples $q$ out of $p$ primal and $m$ out of $n$ dual coordinates to update in each iteration.
\item We show that DSPDC method generates a sequence of primal-dual iterates that \textit{linearly} converges to $(x^\star,y^\star)$ and the primal-dual objective gap along this sequence also \textit{linearly} converges to zero.
\item We generalize this approach to bilinear saddle-point problems with a block-wise decomposable structure, and show a similar iteration complexity for finding an $\epsilon$-optimal solution.
\item We show that the proposed method has a lower \textit{overall complexity} than existing coordinate methods when either the data matrix has a factorized structure or the proximal mapping on each block is computationally expensive, e.g., involving an eigenvalue decomposition.
\item Our experiments confirm the efficiency of DSPDC on both synthetic and real datasets in various scenarios. A notable application is the \emph{multi-task large margin nearest neighbor} (MT-LMNN) metric learning problem.

\end{itemize}



\paragraph{Notation}
\label{sec:notation}
Before presenting our approach, we first introduce the notations that will be used throughout the paper.
Let $[d]$ represent the set $\{1,2,...,d\}$.
For $v\in\mathbb{R}^d$, let $v_i$ be its $i$-th coordinate for $i\in[d]$ and $v_I$ be a sub-vector of $v$ that consists of the coordinates of $v$ indexed by a set $I\subset[d]$. Given an $n\times p$ matrix $W$, we denote its $i$-th row and $j$-th column by $W_i$ and $W^j$, respectively. For $I\subset[n]$ and $J\subset[p]$, the matrices $W_I$ and $W^J$ represent sub-matrices of $W$ that consist of the rows indexed by $I$ and columns indexed by $J$, respectively. We denote the entry of $W$ in $i$-th row and $j$-th column by $W_i^j$ and let $W_I^J$ be sub-matrix of $W$ where the rows indexed by $I$ intersect with the columns indexed by $J$.



Let $\left\langle\cdot,\cdot\right\rangle$ be the inner product in a Euclidean space, $\|\cdot\|$ be the $\ell_2$-norm of a vector and $\|\cdot\|_2$ and $\|\cdot\|_F$ be the spectral norm and the Frobenius norm of a matrix, respectively. For integers $q\in[p]$ and $m\in[n]$, we define $\Lambda_{q,m}$ as a \emph{scale constant} of the data as follows
\begin{eqnarray}
\label{eq:Lambda}
\Lambda_{q,m}\equiv\max_{I\subset[n],J\subset[p],|I|=m,|J|=q}\|A_I^J\|_2^2.
\end{eqnarray}
The maximum $\ell_2$ norm of data points is therefore $\sqrt{\Lambda_{p,1}}$.
The \emph{condition number} of problems~\eqref{eq:erm},\eqref{eq:erm_dual}, and \eqref{eq:sdp} is usually defined as
\begin{equation}\label{eq:kappa}
\kappa\equiv\frac{\Lambda_{p,1}}{\lambda\gamma},
\end{equation}
which affects the iteration complexity of many first-order methods.

\paragraph{Paper Outline} The rest of this paper is organized as follows. Section~\ref{relatedwork} briefly introduces the existing work that are related to ours. Section~\ref{sec:summary} immediately summarizes the results of this paper, along with rough comparisons against the existing methods. In Section~\ref{sec:DSPDC}, we propose the DSPDC algorithm, followed by its theoretical convergence analysis and an efficient implementation for factorized data. The algorithm is extended to the block coordinate update scheme in Section~\ref{sec:block}, which can be applied to two important problems including the multi-task large margin nearest neighbor. We conduct empirical studies in Section~\ref{sec:numerical} to confirm the efficiency of the proposed method, and conclude our paper in Section~\ref{sec:conclusion}. All the proofs are deferred 
to appendix.
\section{Related Work}
\label{relatedwork}


To find an $\epsilon$-optimal solution of problem~\eqref{eq:erm}, \eqref{eq:erm_dual} or \eqref{eq:sdp}, the \emph{overall complexity} of an iterative method is defined as the per-iteration computational cost multiplied by the total number of required iterations (called \emph{iteration complexity}). Deterministic first-order methods such as~\citet{Nesterov04book,Nes:05,Nemirovski:04,Chambolle:2011,YuKC14} have to compute a full gradient in each iteration by going through all $p$ features of all $n$ instances at a per-iteration cost of $O(np)$, which can be inefficient for big data. Therefore, stochastic optimization methods that sample one instance or one feature in each iteration become more popular. There are two major categories of stochastic optimization algorithms that are studied actively in recent years: stochastic gradient methods and stochastic coordinate methods. The DSPDC method we propose belongs to the second category.



Recently, there have been increasing interests in \emph{stochastic variance reduced gradient} (SVRG) methods~\citep{JohnsonZhang13,XiaoZhang14,Nitanda:14,Konecny:15b,allen2016katyusha}. SVRG runs in multiple stages. At each stage, it computes a full gradient and then performs $O(\kappa)$ iterative updates with stochastic gradients constructed by sampled instances. Since the full gradient is computed only once in each stage, SVRG has a per-iteration cost of $O(p)$, which is lower than deterministic gradient methods, and it needs $O((n+\kappa)\log(1/\epsilon))$ iterations
to find an $\epsilon$-optimal solution for problem \eqref{eq:erm}, so that the overall complexity of SVRG is $O((np+\kappa p)\log(1/\epsilon))$.
Recently, an accelerated SVRG method, named Katyusha~\citep{allen2016katyusha}, further reduces the iteration complexity of SVRG to $O((n+\sqrt{n\kappa})\log(1/\epsilon))$ while maintains the $O(p)$ per-iteration cost so that it achieves an overall complexity of $O((np+\sqrt{n\kappa}p)\log(1/\epsilon))$. The aforementioned overall complexities are obtained when a \emph{uniform sampling} scheme is applied in the construction of stochastic gradient. One can further reduce the $\kappa$ term in these complexities by using a \emph{non-uniform sampling} scheme as pointed out, for example, by~\citet{XiaoZhang14}. However, in this paper, the complexity of each algorithm we present and compare is based on a uniform sampling scheme unless stated otherwise. After the earlier version of our draft was posted online\footnote{\url{https://arxiv.org/pdf/1508.03390v2.pdf}}, \citet{palaniappan2016stochastic} developed an accelerated SVRG method (ASVRG-SP) for solving the saddle-point formulation \eqref{eq:sdp}, which has a complexity\footnote{This complexity is achieved by the individual-split version of ASVRG-SP.} of $\tilde O((np+np\sqrt{\frac{\max\{\Lambda_{p,1},\Lambda_{1,n}\}}{\lambda\gamma}})\log(1/\epsilon))$ by uniform sampling and $\tilde O((np+\sqrt{np}\sqrt{\frac{\max\{n,p\}\|A\|_F^2}{\lambda\gamma}})\log(1/\epsilon))$ by non-uniform sampling. Here and in the rest of the paper, $\tilde O$ contains some logarithmic factors.

\emph{Stochastic incremental gradient} methods~\citep{schmidt2013minimizing,LeRouxSchmidtBach12,Defazio:14a,FINITO,MISO:15,Lan:15a,palaniappan2016stochastic} is also widely studied in recent literature. Different from SVRG, stochastic incremental gradient method computes a full gradient only once at the beginning, but maintains and updates the average of historical stochastic gradients using one sampled instance per iteration. Standard stochastic incremental gradient methods~\citep{schmidt2013minimizing,LeRouxSchmidtBach12,Defazio:14a,FINITO,MISO:15} have a per-iteration cost of $O(p)$ just as SVRG and need $O((n+\kappa)\log(1/\epsilon))$ iterations to find an $\epsilon$-optimal solution so that their overall complexity is the same as SVRG. 
Moreover, an accelerated stochastic incremental gradients method, named RPDG~\citep{Lan:15a}, achieves an iteration complexity of only $O((n+\sqrt{n\kappa})\log(1/\epsilon))$ and a per-iteration cost of $O(p)$ so that its overall complexity is the same as Katyusha. The iteration complexity of RPDF and Katyusha is proved to be optimal by~\citet{Lan:15a}.


In contrast to stochastic gradient methods, \emph{stochastic coordinate} method works by updating randomly sampled coordinates of decision variables~\citep{Nesterov12rcdm,RichtarikTakac12,conf/icml/Shalev-ShwartzT09,journals/corr/FercoqR13a,LuXiao13analysis,Dang:14,Lin:15a,Deng:15,allen2016even,nesterov2016efficiency,qu2016coordinate,Qu:2016:CDA:2996036.2996038,richtarik2016optimal}. 
\citet{SSZhang13SDCA,NIPS2013_4938,SSZhang13SDCA} proposed a stochastic dual coordinate ascent (SDCA) method to solve the dual formulation~\eqref{eq:erm_dual}. SDCA has an iteration complexity of $O((n+\kappa)\log(1/\epsilon))$ and has been further improved to the accelerated SDCA (ASDCA) method~\citep{NIPS2013_4938} that achieves an iteration complexity of $\tilde O((n+\sqrt{n\kappa})\log(1/\epsilon))$. The optimal iteration complexity $O((n+\sqrt{n\kappa})\log(1/\epsilon))$ is obtained by the accelerated proximal coordinate gradient (APCG) method~\citep{Lin:15a} when it is applied to the dual problem \eqref{eq:erm_dual}. Extending the deterministic algorithm by~\citet{Chambolle:2011} for saddle-point problems,
\citet{ZhangXiao:14} recently proposed a stochastic primal-dual coordinate (SPDC) method for~\eqref{eq:sdp}, which alternates between maximizing over a randomly chosen dual variable and minimizing over all primal variables and also achieves the optimal $O((n+\sqrt{n\kappa})\log(1/\epsilon))$ iteration complexity. The per-iteration cost is $O(p)$ in all of these coordinate methods.
Note that,  when applied to the primal problem \eqref{eq:erm}, APCG samples a feature of data in each iterative update and find an $\epsilon$-optimal solution with a per-iteration cost of $O(n)$ in $O((p+p\sqrt{\frac{\Lambda_{1,n}}{n\lambda\gamma}})\log(1/\epsilon))$ iterations,
which is also optimal according to~\cite{Lan:15a}.

Some recent works~\citep{conf/nips/ZhaoYWAL14,conf/nips/DaiXHLRBS14,KonecnyRichtarik14,Matsushima:14,Dang13,Deng:15} made attempts in combining stochastic gradient and stochastic coordinate. \citet{conf/nips/ZhaoYWAL14,Matsushima:14,Dang13} proposed randomized block coordinate methods, which utilize stochastic partial gradient of the selected block based on randomly sampled instances and features in each iteration. However, these methods face a constant variance of stochastic partial gradient so that they need $O(1/\epsilon)$ iterations. These techniques are further improved in~\citet{KonecnyRichtarik14,conf/nips/ZhaoYWAL14} with the stochastic variance reduced partial gradient but only obtain the sub-optimal $O((n+\kappa)\log(1/\epsilon))$ iteration complexity.




\section{Summary of Results}\label{sec:summary}

Although the aforementioned stochastic coordinate methods have achieved great performances on ERM problem~\eqref{eq:sdp}, they either only sample over primal coordinates or only sample over dual coordinates to update in each iteration. Therefore, it is natural to ask the following questions.

\begin{itemize}
\item \emph{What is the iteration complexity of a coordinate method for problem~\eqref{eq:sdp} that samples both primal and dual coordinates to update in each iteration?}
\item \emph{When is this type of algorithm has a lower overall complexity than purely primal and purely dual coordinate methods?}
\end{itemize}



To contribute to the answers to these questions, we propose the DSPDC method in Section~\ref{sec:DSPDC} that samples over both features and instances of dataset by randomly choosing the associated primal and dual coordinates to update in each iteration.


To answer the first question, we show in Theorem~\ref{thm:conv_strong_astrsym_switch} and \ref{thm:conv_strong_astrsym_switch_gap} that, if $q$ primal and $m$ dual coordinates are uniformly sampled and updated in each iteration, the number of iterations DSPDC needs to find an $\epsilon$-optimal solution for~\eqref{eq:sdp} is $O((\sqrt{\frac{\Lambda_{q,m}}{\lambda\gamma n}}\frac{np}{mq}+\max\{\frac{n}{m},\frac{p}{q}\})\log(1/\epsilon))$. This iteration complexity is interesting since it matches the optimal $O((n+\sqrt{n\kappa})\log(1/\epsilon))$ iteration complexity of dual coordinate methods~\citep{NIPS2013_4938,Lin:15a,ZhangXiao:14} when $(q,m)=(p,1)$, and also matches the optimal $O((p+p\sqrt{\frac{\Lambda_{1,n}}{n\lambda\gamma}})\log(1/\epsilon))$ iteration complexity of primal coordinate methods~\citep{LuXiao13analysis,Lin:15a} when $(q,m)=(1,n)$.
In Section~\ref{sec:block}, we further generalize DSPDC and its complexity to a bilinear saddle-point problem with a block-wise decomposable structure.

To study the second question, we compare different coordinate algorithms based on the overall complexity for finding an $\epsilon$-optimal solution. For most ERM problems, the per-iteration cost of SPDC is $O(p)$ and its the overall complexity is $O(np+p\sqrt{n\kappa}\log(1/\epsilon))$. When $(q,m)=(1,1)$ and without any assumptions on the sparsity of data, the per-iteration cost of DSPDC is $O(\min\{n,p\})$ due to a full-dimensional inner product in the algorithm. If $n\geq p$, which is true for most ERM problems, the overall complexity of DSPDC becomes $O((np+\sqrt{\frac{n\Lambda_{1,1}}{\lambda\gamma }}p^2)\log(1/\epsilon))$, which is not lower than that of SPDC in general\footnote{\label{note1}Note that $\Lambda_{p,1}\geq\Lambda_{1,1}\geq\frac{\Lambda_{p,1}}{p}$ and $\kappa\geq\frac{\Lambda_{1,1}}{\lambda\gamma}\geq\frac{\kappa}{p}$.}.
Nevertheless, we identify two important cases where DSPDC has a lower overall complexity than SPDC and other existing coordinate methods.

The first case is when data $A$ has a \emph{factorized structure}, namely, $A=UV$ with $U\in\mathbb R^{n\times d}$, $V\in\mathbb R^{d\times p}$ and $d< \min\{n, p\}$. The ERM problem with factorized data arises when (random) dimension/instance reduction or matrix sketching/factorization techniques are applied to $A$ in order to reduce the storage and computational cost. More examples are provided in Section~\ref{sec:implementation}. In this case, choosing $(q,m)=(1,1)$ and using an efficient implementation, our DSPDC has an overall complexity of $O((nd+\sqrt{\frac{n\Lambda_{1,1}}{\lambda\gamma }}pd) \log (1/\epsilon))$, better than the $O((npd+\sqrt{\kappa n}pd) \log (1/\epsilon))$ complexity of SPDC with the same efficient implementation. See Table~\ref{table:complexity} for comparisons with more existing techniques for this class of problems.

The second case is when solving a block-wise decomposable bilinear saddle-point problem where the proximal mapping on each block is computationally expensive. The applications with this property include trace regression~\citep{Slawski:15} and distance metric learning~\citep{Weinberger:2008,Weinberger:2009,ParameswaranW10}, where each block of variables needs to be a $d\times d$ positive semi-definite matrix so that the proximal mapping involves an eigenvalue decomposition with a complexity of $O(d^3)$. When $(q,m)=(1,1)$ and $n\geq p$, DSPDC requires solving eigenvalue decomposition only for one block of variables so that its overall complexity is $O((d^3+pd^2)(n+\sqrt{\frac{n\Lambda_{1,1}}{\lambda\gamma}}p)
\log(1/\epsilon))$ as shown in Section~\ref{sec:block}, which is lower than the $O((npd^3+\sqrt{n\kappa}pd^3)\log(1/\epsilon))$ overall complexity of SPDC when $\sqrt{\Lambda_{1,1}}p\leq\sqrt{\Lambda_{p,1}}d$. See Table~\ref{table:complexity_matrix} for comparisons with more existing techniques for this class of problems.

Although it is not our main focus, we note that applying a non-uniform sampling on the primal and dual coordinates can further reduce the overall complexity of our DSPDC just as other coordinate methods~\citep{ZhangXiao:14,csiba2016importance,richtarik2016optimal,csiba2015stochastic,qu2016coordinate,Qu:2016:CDA:2996036.2996038,allen2016even} .



\section{Doubly Stochastic Primal-Dual Coordinate Method}\label{sec:DSPDC}
\subsection{Algorithm and Convergence Properties}

\begin{algorithm}[t]
\caption{Doubly Stochastic Primal-Dual Coordinate (DSPDC) Method}
\label{alg:aspdc_general}
\begin{algorithmic}[1]
		\REQUIRE $x^{(-1)}=x^{(0)}=\bx^{(0)}\in\mathbb{R}^p$, $y^{(-1)}=y^{(0)}=\by^{(0)}\in\mathbb{R}^n$, and parameters $(\theta,\tau,\sigma)$.
		\ENSURE $x^{(T)}$ and $y^{(T)}$
        \FOR{$t=0,1,2,\ldots,T-1$}
		\STATE Uniformly and randomly choose $I\subset[n]$ and $J\subset[p]$ of sizes $m$ and $q$, respectively.
\STATE Update the primal and dual coordinates
\begin{eqnarray}
\label{eq:y_update}
y_i^{(t+1)}&=&\left\{\begin{array}{ll}
\argmax_{\beta\in\mathbb{R}}\left\{\frac{1}{n}\langle A_i,\bx^{(t)}\rangle\beta-\frac{\phi^*_i(\beta)}{n}-\frac{1}{2\sigma}(\beta-y_i^{(t)})^2\right\}&\text{if }i\in I,\\
y_i^{(t)}&\text{if }i\notin I,
\end{array}
\right.\\\label{eq:by_update}
\by^{(t+1)}&=&y^{(t)}+\frac{n}{m}(y^{(t+1)}-y^{(t)}),\\\label{eq:x_update}
x_j^{(t+1)}&=&\left\{\begin{array}{ll}
\argmin_{\alpha\in\mathbb{R}}\left\{\frac{1}{n}\langle A^j,\by^{(t+1)}\rangle\alpha+g_j(\alpha)+\frac{1}{2\tau}(\alpha-x_j^{(t)})^2\right\}&\text{if }j\in J,\\
x_j^{(t)}&\text{if }j\notin J,
\end{array}
\right.\\\label{eq:bx_update}
\bx^{(t+1)}&=&x^{(t)}+(\theta+1)(x^{(t+1)}-x^{(t)}).
\end{eqnarray}
        \ENDFOR
	\end{algorithmic}
\end{algorithm}

In this section, we propose the doubly stochastic primal-dual coordinate (DSPDC) method in Algorithm~\ref{alg:aspdc_general} for  problem \eqref{eq:sdp}.
In Algorithm~\ref{alg:aspdc_general}, the primal and dual solutions $(x^{(t+1)},y^{(t+1)})$ are updated as~\eqref{eq:x_update} and~\eqref{eq:y_update} in the randomly selected $q$ and $m$ coordinates indexed by $J$ and $I$, respectively\footnote{Here, we hide the dependency of $I$ and $J$ on $t$ to simplify the notation.}. These updates utilize the first-order information provided by the vectors $A\bx^{(t)}$ and $A^T\by^{(t+1)}$ where  $(\bx^{(t)},\by^{(t+1)})$ are updated using the momentum steps~\eqref{eq:bx_update} and~\eqref{eq:by_update} which are commonly used to accelerate gradient (AG) methods~\citep{Nesterov04book,Nes:05}. Algorithm~\ref{alg:aspdc_general} requires three control parameters $\theta$, $\tau$ and $\sigma$ and its convergence is ensured after a proper choice of these parameters as shown in Theorem~\ref{thm:conv_strong_astrsym_switch}. The proofs of all  theorems are deferred to the Appendix.

\begin{theorem}
\label{thm:conv_strong_astrsym_switch}
Suppose $\theta$, $\tau$ and $\sigma$ in Algorithm~\ref{alg:aspdc_general} are chosen so that
\begin{eqnarray}
\label{eq:thetatausigma}
\theta=\frac{p}{q}-\frac{p/q}{\sqrt{\frac{\Lambda}{\lambda\gamma n}}\frac{np}{mq}+\max\{\frac{n}{m},\frac{p}{q}\}},
~
\tau\sigma=\frac{nmq}{4p\Lambda},
~
\frac{p}{2q\lambda\tau}+\frac{p}{q}=\frac{n^2}{2m\gamma\sigma}+\frac{n}{m}
\end{eqnarray}
\normalsize
where $\Lambda$ is any constant such that $\Lambda\geq\Lambda_{q,m}$.
For each $t\geq0$, Algorithm~\ref{alg:aspdc_general} guarantees
\small
\begin{eqnarray*}
&&\left(\frac{p}{2q\tau}+\frac{p\lambda}{q}\right)\mathbb{E}\|x^\star-x^{(t)}\|^2
+\left(\frac{n}{4m\sigma}+\frac{\gamma}{m}\right)\mathbb{E}\|y^\star-y^{(t)}\|^2\\
&\leq&\left(1-\frac{1}{\max\left\{\frac{p}{q},\frac{n}{m}\right\}+\sqrt{\frac{\Lambda}{\lambda\gamma n}}\frac{np}{mq}}\right)^t\bigg[\left(\frac{p}{2q\tau}+\frac{p\lambda}{q}\right)\|x^\star-x^{(0)}\|^2+\left(\frac{n}{2m\sigma}+\frac{\gamma}{m}\right)\|y^\star-y^{(0)}\|^2\bigg].
\end{eqnarray*}
\normalsize
\end{theorem}

\noindent\textbf{Remark 1}
\emph{
For a given $\Lambda$, the values of $\tau$ and $\sigma$ can be solved from the last two equations of \eqref{eq:thetatausigma} in closed forms:
\begin{eqnarray}
\label{eq:opttau2}
\tau&=&
\frac{p}{q\lambda}\left(\left(\frac{n}{m}-\frac{p}{q}\right)+\sqrt{\left(\frac{n}{m}-\frac{p}{q}\right)^2+\frac{4(np)^2\Lambda}{(mq)^2n\lambda\gamma}}\right)^{-1},\\
\label{eq:optsigma2}
\sigma&=&
\frac{n^2}{m\gamma}\left(\left(\frac{p}{q}-\frac{n}{m}\right)+\sqrt{\left(\frac{n}{m}-\frac{p}{q}\right)^2+\frac{4(np)^2\Lambda}{(mq)^2n\lambda\gamma}}\right)^{-1},
\end{eqnarray}
which are referred to as the primal and dual step size, respectively.
If both primal and dual coordinates are sampled at the same ratio, i.e., ${q\over p}={m\over n}$, then we have the following simplified version:
\begin{eqnarray}
\label{eq:opttau3}
\tau= {m\over 2} \sqrt{\gamma\over \Lambda n\lambda},\quad
\sigma = {m\over 2} \sqrt{n\lambda\over \Lambda\gamma}.
\end{eqnarray}
\normalsize
According to the convergence rate above, the best choice of $\Lambda$ is $\Lambda_{q,m}$. Although the exact computation of $\Lambda_{q,m}$ by definition \eqref{eq:Lambda} may be costly, for instance, when $q\approx \frac{p}{2}$ or $m\approx \frac{n}{2}$, it is tractable when $q$ and $m$ are close to 1 or close to $p$ and $n$. In practice, we suggest choosing $\Lambda=\frac{mqR^2\Lambda_{p,1}}{p}$ as an approximation of $\Lambda_{q,m}$, which provides reasonably good empirical performance (see Section~\ref{sec:numerical}).
}

Besides the distance to the saddle-point $(x^\star,y^\star)$, a useful quality measure for the solution $(x^{(t)},y^{(t)})$ is its primal-dual objective gap, $P(x^{(t)})-D(y^{(t)})$, because it can be evaluated in each iteration and used as a stopping criterion in practice. The next theorem establishes the convergence rate of the primal-dual objective gap ensured by DSPDC.

\begin{theorem}
\label{thm:conv_strong_astrsym_switch_gap}
Suppose $\tau$ and $\sigma$ are chosen as~\eqref{eq:thetatausigma} while $\theta$ is replaced by
\begin{eqnarray}
\label{eq:thetagap}
\theta=\frac{p}{q}-\frac{p/q}{2\sqrt{\frac{\Lambda}{\lambda\gamma n}}\frac{np}{mq}+2\max\{\frac{n}{m},\frac{p}{q}\}}
\end{eqnarray}
\normalsize
in Algorithm~\ref{alg:aspdc_general}.
For each $t\geq0$, Algorithm~\ref{alg:aspdc_general} guarantees
\small
\begin{eqnarray*}
&&\mathbb{E}\left[P(x^{(t)})-D(y^{(t)})\right]\\
&\leq&
\left(1-\frac{1}{2\sqrt{\frac{\Lambda}{\lambda\gamma n}}\frac{np}{mq}+2\max\{\frac{n}{m},\frac{p}{q}\}}\right)^t\times
\left\{\frac{1}{\min\left\{\frac{p}{q},\frac{n}{m}\right\}}+\frac{\max\left\{\frac{\|A\|^2}{n\gamma},\frac{\|A\|^2}{\lambda n^2}\right\}}{\min\left\{\frac{\lambda p}{q},\frac{\gamma }{m}\right\}}\right\}\times\\
&&\bigg[\left(\frac{p}{2q\tau}+\frac{p\lambda}{2q}\right)\|x^{(0)}-x^\star\|^2+\left(\frac{n}{2m\sigma}+\frac{\gamma}{2m}\right)\|y^{(0)}-y^\star\|^2+\max\left\{\frac{p}{q},\frac{n}{m}\right\}\left(P(x^{(0)})-D(y^{(0)})\right)\bigg].
\end{eqnarray*}
\normalsize
\end{theorem}

According to Theorem~\ref{thm:conv_strong_astrsym_switch} and \ref{thm:conv_strong_astrsym_switch_gap}, in order to obtain a pair of primal and dual solutions with an expected $\epsilon$ distance to $(x^\star,y^\star)$, i.e., $\mathbb{E}[\|x^{(t)}-x^\star\|^2]\leq\epsilon$ and $\mathbb{E}[\|y^{(t)}-y^\star\|^2]\leq\epsilon$, or with an expected $\epsilon$ objective gap, Algorithm~\ref{alg:aspdc_general} needs
\begin{eqnarray*}
\label{eq:xyepsilont}
t=O\left(\Big(\max\left\{\frac{p}{q},\frac{n}{m}\right\}+\sqrt{\frac{\Lambda_{q,m}}{n\lambda\gamma}}\frac{pn}{qm}\Big)
\log\big(\frac{1}{\epsilon}\big)\right)
\end{eqnarray*}
iterations when $\Lambda=\Lambda_{q,m}$. 
This iteration complexity is interesting since it matches the optimal $O((n+\sqrt{n\kappa})\log\left(\frac{1}{\epsilon}\right))$ iteration complexity of dual coordinate methods such as SPDC~\citep{ZhangXiao:14} and others~\citep{NIPS2013_4938,Lin:15a} when $(q,m)=(p,1)$, and also matches the optimal $O((p+p\sqrt{\frac{\Lambda_{1,n}}{n\lambda\gamma}})\log\left(\frac{1}{\epsilon}\right))$ iteration complexity of primal coordinate methods~\citep{LuXiao13analysis,Lin:15a} when $(q,m)=(1,n)$.

To efficiently implement Algorithm~\ref{alg:aspdc_general}, we just need to maintain and efficiently update either $A\bx^{(t)}$ or $A^T\by^{(t)}$, depending on whether $\frac{n}{m}$ or $\frac{p}{q}$ is larger. If $\frac{n}{m}\geq\frac{p}{q}$, we should maintain $A^T\by^{(t)}$ during the algorithm, which is used in~\eqref{eq:x_update} and can be updated in $O(mp)$ time. We will then directly compute $\langle A_i,\bx^{(t)}\rangle$ for $i\in I$ in~\eqref{eq:y_update} in $O(mp)$ time. In fact, this is how SPDC is implemented in~\cite{ZhangXiao:14} where $q=p$.
On the other hand, if $\frac{n}{m}\leq\frac{p}{q}$, it is more efficient to maintain $A\bx^{(t)}$ and update it in $O(qn)$ time and compute $\langle A^j,\by^{(t+1)}\rangle$ for $j\in J$ in \eqref{eq:x_update} in $O(qn)$ time. Hence, the overall complexity for DSPDC to find an $\epsilon$-optimal solution
is $O((np+\sqrt{\frac{n\Lambda_{q,m}}{\lambda\gamma} }\frac{p^2}{q})\log\left(\frac{1}{\epsilon}\right))$ when $\frac{n}{m}\geq\frac{p}{q}$ and $O((np+\sqrt{\frac{n\Lambda_{q,m}}{\lambda\gamma} }\frac{np}{m})\log\left(\frac{1}{\epsilon}\right))$ when $\frac{n}{m}\leq\frac{p}{q}$. Since the overall complexity of SPDC is $O\left(\left(np+\sqrt{\kappa nm}p\right)\log\left(\frac{1}{\epsilon}\right)\right)$ when $\frac{n}{m}\geq\frac{p}{q}$, DSPDC method is not more efficient for general data matrix. However, in the next section, we show that DSPDC has an efficient implementation for factorized data matrix which leads to a lower overall complexity than SPDC with the same implementation.

\subsection{Efficient Implementation for Factorized Data Matrix}
\label{sec:implementation}


In this section, we assume that the data matrix $A$ in~\eqref{eq:sdp} has a factorized structure $A=UV$ where $U\in\mathbb{R}^{n\times d}$ and $V\in\mathbb{R}^{d\times p}$ with $d<\min\{n,p\}$. Such a matrix $A$ is often obtained as a low-rank or denoised approximation of raw data matrix. Recently, there emerges a surge of interests of using factorized data to alleviate the computational cost for big data. For example, \citet{Pham:15} proposed to use a low-rank approximation $X\approx UV=A$ for data matrix $X$ to solve multiple instances of lasso problems. For solving big data kernel learning problems, the Nystr\"{o}m methods, that approximates a $n\times n$ kernel matrix $K$ by $US^{\dagger}U^{\top}$ with $U\in\mathbb R^{n\times d}$, $S\in\mathbb R^{d\times d}$ and $d< n$, has become a popular method~\citep{Yang12}. Moreover, recent advances on fast randomized algorithms~\citep{Halko:2011} for finding a low-rank approximation of a matrix render the proposed coordinate optimization algorithm more attractive for tackling factorized big data problems.

The factorized $A$ also appears often in the problem of sparse recovery from the randomized feature reduction or randomized instance reduction of \eqref{eq:erm}. The sparse recovery problem from randomized feature reduction can be also formulated into \eqref{eq:sdp} as
\begin{eqnarray}
\label{eq:srir}
\min_{x\in\mathbb{R}^p}\max_{y\in\mathbb{R}^n} \left\{\frac{\lambda_2}{2}\|x\|_2^2+\lambda_1\|x\|_1+\frac{1}{n}y^TXG^TGx-\frac{1}{n}\sum_{i=1}^n\phi_i^*(y_i)\right\}
\end{eqnarray}
where $X$ is the original $n\times p$ raw data matrix, $G$ is a $d \times p$ random measurement matrix with $d< p$, and the actual data matrix for~\eqref{eq:sdp} is $A=XG^TG$ with $U=XG^T$ and $V=G$. This approximation approach has been employed to reduce the computational cost of solving underconstrained least-squares problem~\citep{mahoney-2011-randomized,wang2016sketching}.
Similarly, the randomized instance reduction~\citep{Drineas:2011:FLS:1936922.1936925,wang2016sketching} can be applied by replacing $XG^TG$ in~\eqref{eq:srir} with $G^TGX$, where $G$ is a $d \times n$ random measurement matrix with $d< n$, and the data matrix $A=G^TGX$ with $U=G^T$ and $V=GX$.

To solve \eqref{eq:sdp} with $A=UV$, we implement DSPDC by maintaining the vectors $\bu^{(t)}=U^T\by^{(t)}$ and $\bv^{(t)}=V\bx^{(t)}$ and updating them in $O(dm)$ and $O(dq)$ time, respectively, in each iteration. Then, we can obtain $\left\langle A_i,\bx^{(t)}\right\rangle$ in~\eqref{eq:y_update} in $O(dm)$ time by evaluating $\left\langle U_i,\bv^{(t)}\right\rangle$ for each $i\in I$, where $U_i$ is the $i$th row of $U$. Similarly, we can obtain $\left\langle A_j,\by^{(t+1)}\right\rangle$ in~\eqref{eq:x_update} in $O(dq)$ time by taking $\left\langle V^j,\bv^{(t)}\right\rangle$ for each $j\in J$, where $V^j$ is the $j$th column of $V$. This leads to an efficient implementation of DSPDC described as in  Algorithm~\ref{alg:aspdc_general_factor} whose per-iteration cost is $O(dm+dq)$, lower than the $O(mp)$ or $O(qn)$ cost when $A$ is not factorized. 

To make a clear comparison between DSPDC and other methods when applied to factorized data, in Table~\ref{table:complexity}, we summarize their numbers of iterations and per-iteration costs (when $A=UV$)\footnote{For SVRG and ASVRG-SP, we present their numbers of outer and inner iterations and per-iteration costs separately.}. For all methods in comparison, we assume $A$ is too large so that only $U$ and $V$ are stored in memory, which is the typical situation when applying random reduction. Moreover, the aforementioned efficient implementation in DSPDC (if applicable) has been also applied to other methods to reduce their per-iteration cost. In Table~\ref{table:complexity}, we assume $n\geq p$ and $(q,m)=(1,1)$ and omit all the big-$O$ notations for simplicity. For ASVRG-SP, we present the complexity of its individual-split version with uniform sampling. According to the last column of Table~\ref{table:complexity}, our DSPDC with efficient implementation has the lowest overall complexity among these methods.


\begin{algorithm}[t]
\caption{Efficient Implementation of Algorithm~\ref{alg:aspdc_general} for Factorized Data}
\label{alg:aspdc_general_factor}
\textbf{Input:} $x^{(-1)}=x^{(0)}=\bx^{(0)}\in\mathbb{R}^p$, $y^{(-1)}=y^{(0)}=\by^{(0)}\in\mathbb{R}^n$, and parameters $(\theta,\tau,\sigma)$\\[0.5ex]
\textbf{Initialize:} $u^{(0)}=U^Ty^{(0)}$, $v^{(0)}=Vx^{(0)}$,$\bu^{(0)}=U^T\by^{(0)}$, $\bv^{(0)}=V\bx^{(0)}$ \\[0.5ex]
\textbf{Iterate:}\\[0.5ex] For $t=0,1,2,\ldots,T-1$
\begin{enumerate}  \itemsep 0pt
\item[]
Uniformly and randomly choose $I\subset[n]$ and $J\subset[p]$ of sizes $m$ and $q$, respectively.
\begin{eqnarray}
y_i^{(t+1)}&=&\left\{\begin{array}{ll}
\argmax_{\beta\in\mathbb{R}}\left\{\frac{1}{n}\langle U_i,\bv^{(t)}\rangle\beta-\frac{\phi^*_i(\beta)}{n}-\frac{1}{2\sigma}(\beta-y_i^{(t)})^2\right\}&\text{if }i\in I,\\
y_i^{(t)}&\text{if }i\notin I,
\end{array}
\right.\\
u^{(t+1)}&=&u^{(t)}+U^T(y^{(t+1)}-y^{(t)}),\\
\bu^{(t+1)}&=&u^{(t)}+\frac{n}{m}U^T(y^{(t+1)}-y^{(t)}),\\
x_j^{(t+1)}&=&\left\{\begin{array}{ll}
\argmin_{\alpha\in\mathbb{R}}\left\{\frac{1}{n}\langle V^j,\bu^{(t+1)}\rangle\alpha+g_j(\alpha)+\frac{1}{2\tau}(\alpha-x_j^{(t)})^2\right\}&\text{if }j\in J,\\
x_j^{(t)}&\text{if }j\notin J,
\end{array}
\right.\\
v^{(t+1)}&=&v^{(t)}+V(x^{(t+1)}-x^{(t)}),\\
\bv^{(t+1)}&=&v^{(t)}+(\theta+1)V(x^{(t+1)}-x^{(t)}).
\end{eqnarray}
\end{enumerate}
\textbf{Output:} $x^{(T)}$ and $y^{(T)}$
\end{algorithm}

\begin{table}[t]
\begin{center}
    \begin{tabular}{| c | c | c | c|}
    \hline
    Algorithm & Num. of Iter. $(\times\log(\frac{1}{\epsilon}))$   & Per-Iter. Cost   &Overall Compl. $(\times\log(\frac{1}{\epsilon}))$ \\\hline
    DSPDC & $n+p\sqrt{\frac{\Lambda_{1,1} n}{\lambda\gamma}}$ & $d$&$nd+pd\sqrt{\frac{\Lambda_{1,1} n}{\lambda\gamma}}$\\\hline
    SPDC & \multirow{4}{*}{$n+\sqrt{\frac{\Lambda_{p,1} n}{\lambda\gamma}}$} & \multirow{4}{*}{$pd$} & \multirow{4}{*}{$npd+pd\sqrt{\frac{\Lambda_{p,1} n}{\lambda\gamma}}$}\\
    ASDCA& && \\
    APCG& & & \\
    RPDG& & & \\\hline
    SDCA& \multirow{2}{*}{$n+\frac{\Lambda_{p,1}}{\lambda\gamma}$} & \multirow{2}{*}{$pd$} &  \multirow{2}{*}{$npd+pd\frac{\Lambda_{p,1}}{\lambda\gamma}$}  \\
    SAGA& &&\\\hline
    \multirow{2}{*}{SVRG} & Outer: $1$ & $nd$ &  \multirow{2}{*}{$nd+pd\frac{\Lambda_{p,1}}{\lambda\gamma}$}  \\
    &Inner: $\frac{\Lambda_{p,1}}{\lambda\gamma}$ &$pd$&\\\hline
    \multirow{2}{*}{ASVRG-SP}&Outer: $\sqrt{\frac{p\max\{\Lambda_{p,1},\Lambda_{1,n}\}}{\lambda\gamma}}+1$& $nd$&  \multirow{2}{*}{$nd+nd\sqrt{\frac{p\max\{\Lambda_{p,1},\Lambda_{1,n}\}}{\lambda\gamma}}$}  \\
    &Inner: $n\sqrt{\frac{p\max\{\Lambda_{p,1},\Lambda_{1,n}\}}{\lambda\gamma}}$&$d$& \\\hline
    \end{tabular}
    \caption{The overall complexity of finding an $\epsilon$-optimal solution when $A=UV$, $n\geq p$ and $U$ and $V$ (but not $A$) are stored in memory. We choose $(q,m)=(1,1)$ in DSPDC.
    \label{table:complexity}
    }
\end{center}
\vspace{-15pt}
\end{table}

\section{Extension with Block Coordinate Updates}
\label{sec:block}
With block-wise sampling and updates, DSPDC can be easily generalized and applied to the bilinear saddle-point problem \eqref{eq:sdp} with a block-decomposable structure and a similar linear convergence rate can be obtained. Although this is a straightforward extension, it is worth showing that, when the proximal mapping on each block is computationally expensive, DSPDC can achieve a lower complexity than other coordinate methods. In this section, we first extend DSPDC to its block coordinate update version, and then identify the scenarios where such an extension has a lower overall complexity than other methods.
\subsection{Algorithm and Convergence Properties}



We partition the space $\mathbb{R}^{\bar p}$ into $p$ subspaces as $\mathbb{R}^{\bar p}=\mathbb{R}^{q_1}\times \mathbb{R}^{q_2}\times\cdots\times \mathbb{R}^{q_p}$ such that $\sum_{j=1}^pq_j=\bar p$ and partition the space $\mathbb{R}^{\bar n}$ into $n$ subspaces as $\mathbb{R}^{\bar n}=\mathbb{R}^{m_1}\times \mathbb{R}^{m_2}\times\cdots\times \mathbb{R}^{m_n}$ such that $\sum_{i=1}^nm_i=\bar n$. With a little abuse of notation, we represent the corresponding partitions of $\bbx\in\mathbb{R}^{\bar p}$ and $\bby\in\mathbb{R}^{\bar n}$ as $\bbx=(\bbx_1,\bbx_2,\dots,\bbx_p)$ with $\bbx_j\in\mathbb{R}^{q_j}$ for $j=1,\dots,p$ and  $\bby=(\bby_1,\bby_2,\dots,\bby_n)$ with $\bby_i\in\mathbb{R}^{m_i}$ for $i=1,\dots,n$, respectively.

We consider the following bilinear saddle-point problem
\begin{eqnarray}
\label{eq:sdp_gen}
\min_{\bbx\in\mathbb{R}^{\bar p}}\max_{\bby\in\mathbb{R}^{\bar n}}\left\{ \sum_{j=1}^pg_j(\bbx_j)+\frac{1}{n}\bby^T\bbA\bbx-\frac{1}{n}\sum_{i=1}^n\phi_i^*(\bby_i)\right\},
\end{eqnarray}
where $g_j:\mathbb{R}^{q_j}\rightarrow \mathbb{R}$ and $\phi_i^*:\mathbb{R}^{m_i}\rightarrow \mathbb{R}$ are functions of $\bbx_j$ and $\bby_i$, respectively. Moreover, we assume $g_j$ and $\phi_i^*$ are strongly convex with strong convexity parameters of $\lambda>0$ and $\gamma>0$, respectively. Due to the partitions on $\bbx\in\mathbb{R}^{\bar p}$ and $\bby\in\mathbb{R}^{\bar n}$, we partition the matrix $\bbA$ into blocks accordingly so that
$$
\bby^T\bbA\bbx=\sum_{j=1}^p\sum_{i=1}^n\bby_i^T\bbA_i^j\bbx_j,
$$
where $\bbA_i^j\in\mathbb{R}^{m_i\times q_j}$ is the block of $\bbA$ corresponding to $\bbx_j$ and $\bby_i$.



It is easy to see that the problem \eqref{eq:sdp} is a special case of \eqref{eq:sdp_gen} when $q_j=m_i=1$ for $j=1,\dots,p$ and $i=1,\dots,n$, $\bar p=p$ and $\bar n=n$. The scale constant defined in \eqref{eq:Lambda} can be similarly generalized as
\begin{eqnarray}
\label{eq:Lambda_gen}
\mathbf{\Lambda}_{q,m}\equiv\max_{I\subset[n],J\subset[p],|I|=m,|J|=q}\|\bbA_I^J\|_2^2,
\end{eqnarray}
where $\bbA_I^J$ is sub-matrix of $\bbA$ consisting of each block $\bbA_i^j$ with $i\in I$ and $j\in J$.

Let $\bbA_i=(\bbA_i^1,\cdots,\bbA_i^p)$ and $\bbA^j=((\bbA^j_1)^T,\cdots,(\bbA^j_n)^T)^T$. Given these correspondings between \eqref{eq:sdp} and \eqref{eq:sdp_gen}, DSPDC can be easily extended for solving \eqref{eq:sdp_gen} by replacing \eqref{eq:y_update} and \eqref{eq:x_update} with
\begin{eqnarray}
\label{eq:y_update_block}
\bby_i^{(t+1)}&=&\left\{\begin{array}{ll}
\argmax_{\beta\in\mathbb{R}^{m_i}}\left\{\frac{1}{n}\beta^T\bbA_i\bar\bbx^{(t)}
-\frac{\phi^*_i(\beta)}{n}-\frac{1}{2\sigma}\|\beta-\bby_i^{(t)}\|^2\right\}&\text{if }i\in I,\\
\bby_i^{(t)}&\text{if }i\notin I,
\end{array}
\right.\\\label{eq:x_update_block}
\bbx_j^{(t+1)}&=&\left\{\begin{array}{ll}
\argmin_{\alpha\in\mathbb{R}^{q_j}}\left\{\frac{1}{n}\alpha^T( \bbA^j)^T\bar\bby^{(t+1)}+g_j(\alpha)+\frac{1}{2\tau}\|\alpha-\bbx_i^{(t)}\|^2\right\}&\text{if }j\in J,\\
\bbx_j^{(t)}&\text{if }j\notin J,
\end{array}
\right.
\end{eqnarray}
\normalsize
respectively, and $\bar\bby^{(t)}$ and $\bar\bbx^{(t)}$ are updated in the same way as \eqref{eq:by_update} and \eqref{eq:bx_update}.

For this extension, the convergence results similar to Theorem~\ref{thm:conv_strong_astrsym_switch} and Theorem~\ref{thm:conv_strong_astrsym_switch_gap} can be easily derived with almost the same proof. We skip the proofs but directly state the results. To find a pair of primal-dual solutions for \eqref{eq:sdp_gen} which either has an $\epsilon$-distance to the optimal solution or has an $\epsilon$-primal-dual objective gap, the number of iterations Algorithm \ref{alg:aspdc_general} (with  by \eqref{eq:y_update} and \eqref{eq:x_update} replaced by \eqref{eq:y_update_block} and \eqref{eq:x_update_block}) needs is
\begin{eqnarray*}
t=O\left(\Big(\max\Big\{\frac{n}{m},\frac{p}{q}\Big\}+\sqrt{\frac{\mathbf{\Lambda}_{q,m}}{\lambda\gamma n}}\frac{np}{mq}\Big)
\log\Big(\frac{1}{\epsilon}\Big)\right).
\end{eqnarray*}

\subsection{Matrix Risk Minimization}\label{sec:matrix_risk}
In this section, we study the theoretical performance of DSPDC method when the block updating step \eqref{eq:y_update_block} or \eqref{eq:x_update_block} has a high computational cost due to eigenvalue decomposition. Let $\mathbb{S}_+^d$ be the set of $d\times d$ positive semi-definite matrices. The problem we consider is a general multiple-matrix risk minimization which is formulated as
\begin{eqnarray}
\label{eq:matrixerm}
\min_{X_j\in\mathbb{S}_+^d,j=1,\dots,p}\left\{\frac{1}{n}\sum_{i=1}^n\phi_i\left(\sum_{j=1}^p\langle\bbD_i^j, X_j\rangle\right)+\frac{\lambda}{2}\sum_{j=1}^p\|X_j\|_F^2\right\},
\end{eqnarray}
where $\bbD_i^j$ is a $d\times d$ data matrix, $\phi_i$ is $(1/\gamma)$-smooth convex loss function applied to the linear prediction $\sum_{j=1}^p\langle\bbD_i^j, X_j\rangle$ and $\lambda$ is a regularization parameter.
The associated saddle-point formulation of \eqref{eq:matrixerm} is
\begin{eqnarray}
\label{eq:matrixspd}
\min_{X_j\in\mathbb{S}_+^d,j=1,\dots,p}\max_{y\in\mathbb{R}^n}\left\{\frac{\lambda}{2}\sum_{j=1}^p\|X_j\|_F^2
+\frac{1}{n}\sum_{i=1}^n\sum_{j=1}^py_i\langle\bbD_i^j, X_j\rangle-\frac{1}{n}\sum_{i=1}^n\phi_i^*(y_i)\right\},
\end{eqnarray}
\normalsize
which is a special case of \eqref{eq:sdp_gen} where $q_j=d^2$ and $m_i=1$, $\bbx_j\in\mathbb{R}^{d^2}$ and $\bbA_i^j\in\mathbb{R}^{1\times d^2}$ are the vectorization of the matrices $X_j$ and $\bbD_i^j$ respectively, and $g_j(X_j)=\frac{\lambda}{2}\|X_j\|_F^2$ if $X_j\in\mathbb{S}_+^d$ and $g_j(X_j)=+\infty$ if $X_j\notin\mathbb{S}_+^d$.
The applications of this model include matrix trace regression~\citep{Slawski:15} and distance metric learning~\citep{Weinberger:2008,Weinberger:2009,ParameswaranW10}.

In Table~\ref{table:complexity_matrix}, we compare DSPDC with various methods on the numbers of iterations and per-iteration costs when applied to problem~\eqref{eq:matrixspd}. We assume $n\geq \max\{p,d\}$ and $(q,m)=(1,1)$ and omit all the big-$O$ notations for simplicity. For ASVRG-SP, we present the complexity of its individual-split version with uniform sampling. When applied to \eqref{eq:matrixspd} with $(q,m)=(1,1)$, DSPDC requires solving \eqref{eq:x_update_block} in each iteration which involves the eigenvalue decomposition of one $d\times d$ matrix with complexity of $O(d^3)$. To efficiently implement DSPDC, we need to maintain and efficiently update either $\bbA\bar\bbx^{(t)}$ or $\bbA^T\bar\bby^{(t)}$ with complexity of $O(d^2\min\{n,p\})$. When $p\leq n$, the per-iteration cost of DSPDC in this case is therefore $O(d^3+pd^2)$ so that the overall complexity for DSPDC to find an $\epsilon$-optimal solution of \eqref{eq:matrixspd} is $O((d^3+pd^2)(n+\sqrt{\frac{n\mathbf{\mathbf{\Lambda}}_{1,1}}{\lambda\gamma}}p)
\log(1/\epsilon))$. On the contrary, SPDC, ASDCA, APCG and RPDG need to solve $p$ eigenvalue decompositions per iteration so that its the overall complexity is $O(pd^3(n+\sqrt{\frac{n\mathbf{\mathbf{\Lambda}}_{p,1}}{\lambda\gamma}})
\log(1/\epsilon))$ which is higher than that of DSPDC when $\sqrt{\mathbf{\Lambda}_{1,1}}p\leq\sqrt{\mathbf{\Lambda}_{p,1}}d$. Without this condition, according to the last column of Table~\ref{table:complexity_matrix}, DSPDC still has a lower overall complexity than SDCA, SAGA, SVRG and ASVRG-SP.

\begin{table}[t]    
\begin{center}
    \begin{tabular}{| c | c | c | c|}
    \hline
    Algorithm & Num. of Iter. $(\times\log(\frac{1}{\epsilon}))$   & Per-Iter. Cost   &Overall Compl. $(\times\log(\frac{1}{\epsilon}))$ \\\hline
    DSPDC & $n+p\sqrt{\frac{\mathbf{\Lambda}_{1,1} n}{\lambda\gamma}}$ & $pd^2+d^3$&$(d^2p+d^3)(n+p\sqrt{\frac{\mathbf{\Lambda}_{1,1} n}{\lambda\gamma}})$\\\hline
    SPDC & \multirow{4}{*}{$n+\sqrt{\frac{\mathbf{\Lambda}_{p,1} n}{\lambda\gamma}}$} & \multirow{4}{*}{$pd^3$} & \multirow{4}{*}{$npd^3+pd^3\sqrt{\frac{\mathbf{\Lambda}_{p,1} n}{\lambda\gamma}}$}\\
    ASDCA& && \\
    APCG& & & \\
    RPDG& & & \\\hline
    SDCA& \multirow{2}{*}{$n+\frac{\mathbf{\Lambda}_{p,1}}{\lambda\gamma}$} & \multirow{2}{*}{$pd^3$} &  \multirow{2}{*}{$npd^3+pd^3\frac{\mathbf{\Lambda}_{p,1}}{\lambda\gamma}$}  \\
    SAGA& &&\\\hline
    \multirow{2}{*}{SVRG}&Outer: $1$& $pnd^2$&  \multirow{2}{*}{$npd^2+pd^3\frac{\mathbf{\Lambda}_{p,1}}{\lambda\gamma}$}  \\
    &Inner: $\frac{\mathbf{\Lambda}_{p,1}}{\lambda\gamma}$&$pd^3$& \\\hline
    \multirow{2}{*}{ASVRG-SP}&Outer: $\sqrt{\frac{d\max\{\mathbf{\Lambda}_{p,1},\mathbf{\Lambda}_{1,n}\}}{\lambda\gamma}}+1$& $pnd^2$&  \multirow{2}{*}{$npd^2+pnd^2\sqrt{\frac{d\max\{\mathbf{\Lambda}_{p,1},\mathbf{\Lambda}_{1,n}\}}{\lambda\gamma}}$}  \\
    &Inner: $np\sqrt{\frac{\max\{\mathbf{\Lambda}_{p,1},\mathbf{\Lambda}_{1,n}\}}{d\lambda\gamma}}$&$d^3$& \\\hline
    \end{tabular}
    \caption{The overall complexity of finding an $\epsilon$-optimal solution for \eqref{eq:matrixspd} when $n\geq p$. 
    }
    \label{table:complexity_matrix}
\end{center}
\vspace{-15pt}
\end{table}

\subsection{Multi-task  Large Margin Nearest Neighbor Problem}\label{sec:mt_lmnn}

In this section, we show that DSPDC can be applied to the Multi-task Large Margin Nearest Neighbor (MT-LMNN) problem~\citep{ParameswaranW10}.
The key is to appropriately reduce the original form to the matrix risk minimization \eqref{eq:matrixspd}.

\paragraph{Problem Reformulation} To make the paper self-contained, we include the introduction of MT-LMNN here. Interested readers can find more background in~\citep{ParameswaranW10}.
Suppose there are $p>1$ tasks, each being a multi-class classification problem. For example, in our empirical study (Section~\ref{sec:mt_lmnn_exp}), we have $p=100$ tasks, each being a 10-class image classification problem. MT-LMNN aims to learn one Mahalanobis distance metric (defined as a positive semi-definite matrix) for each task, so there are totally $p$ metric matrices to be learned. With those metrics, the label of a testing point in task $j$ is determined by the majority vote of its 
$\ell$-nearest neighbors defined by the $j$-th metric.
We further assume that the tasks are correlated that all the metrics share a common component, in addition to their own matrix.
The original formulation of MT-LMNN is the following:
\small
\begin{eqnarray}\nonumber
&\min\limits_{X_j\in\Sb^d_+,j=0,1,...,p, \xi\in\Rb^n} & {\lambda_0\over 2}\|X_0-I\|_F^2 + \sum_{j=1}^{p}{\lambda_j\over 2} \|X_j\|_F^2 + {1\over n}\sum_{j=1}^{p}\sum_{(u,v)\in \Nc_j} d_j^2(z_u,z_v) + {1\over n}\sum_{j=1}^{p}\sum_{(u,v,w)\in \Sc_j} \xi_{uvw} \\\label{eq:original}
&\st & d_j^2(z_u,z_w) - d_j^2(z_u,z_v) \ge 1 - \xi_{uvw}, \quad
\forall j\in [p], ~\forall (u,v,w)\in \Sc_j\\\nonumber
&& \xi_{uvw} \ge 0, \quad\forall (u,v,w)\in \Sc_j.
\end{eqnarray}
\normalsize
Now we interpret the notations above. We let $z$ denote a training data point indexed by subscript $u,v,w$ etc,    $X_j\in\Sb^{d}_+$ be the metric matrix for task $j=1,2,...,p$ and $X_0$  the common component shared by all the tasks to reflect the correlations among them. Let $\Nc_j$ be the set of every ordered pair $(u,v)$ for task $j$ such that $z_v$ is among the $\ell$ closest points of $z_u$ that has the same label as $z_u$, and $\Sc_j$ be the set consisting of the triples $(u,v,w)$ such that $(u,v)\in\Nc_j$ and $z_w$ is the closest point to $z_u$ that has a different label. The aforementioned closeness can be measured in Euclidean distance or other appropriate methods in the original feature space. We use $n:=\sum_{j=1}^{p}|\Nc_j| = \sum_{j=1}^{p}|\Sc_j|$ to denote the total number of constraints in \eqref{eq:original} excluding the non-negativity constraints. Let $\bbZ_{j,uv}:=  (z_u-z_v)(z_u-z_v)^{\top}$ for all $(u,v)\in\Nc_j$ and $\bbZ_{j,uvw}:= \bbZ_{j,uw} - \bbZ_{j,uv}$ for all $(u,v,w)\in\Sc_j$ so that the distance $d_j(z_u,z_v)$ in task $j$ is defined as
$$d_j(z_u,z_v) = \sqrt{(z_u-z_v)^{\top}(X_j+X_0)(z_u-z_v)} = \sqrt{\langle \bbZ_{j,uv}, X_j+X_0\rangle},$$
Note that the metric matrix for task $j$ is $X_j+X_0$, the sum of the individual matrix $X_j$ and the shared component $X_0$ among all the tasks.
We can see that in each task, the goal of formulation~\eqref{eq:original} is essentially to minimize the distances of points with the same label (the objective) while enforcing the points with different labels to stay away from each other (the constraints).
The slack variables $\xi_{uvw}$ allow for soft constraints in the problem. The regularization term $\lambda_j\|X_j\|_F^2,\forall j\in[p]$ controls the magnitude of $X_j$  and $\lambda_0\|X_0-I\|_F^2$ tunes how close $X_0$ to the identity $I$. 

Following the same convention of support vector machine, we can transform the problem~\eqref{eq:original} to
an unconstrained form:
\begin{eqnarray}\label{eq:unconstrain}
\min_{X_0,...,X_p\in \Sb_+^d} && {\lambda_0\over 2}\|X_0-I\|_F^2 + \sum_{j=1}^{p}{\lambda_j\over 2} \|X_j\|_F^2
+ {1\over n}\sum_{j=1}^{p}\sum_{(u,v)\in \Nc_j} \langle \bbZ_{j,uv},X_j+X_0\rangle\\\nonumber
&&+ {1\over n}\sum_{j=1}^{p}\sum_{(u,v,w)\in S_j}\phi(\langle \bbZ_{j,uvw}, X_j+X_0\rangle).
\end{eqnarray}
where $\phi(\cdot)$ is the hinge loss and we adopt its smoothed version
\eqref{eq:ssvmloss} with $b=1$.

By introducing the dual variable $y$ we can obtain the following equivalent saddle-point formulation of \eqref{eq:unconstrain}:
\begin{eqnarray}
\label{eq:saddle1}
\min_{X_0,...,X_p\in \Sb_+^d} \max_{y\in\Rb^n} && {\lambda_0\over 2}\|X_0-I\|_F^2 + \sum_{j=1}^{p}{\lambda_j\over 2} \|X_j\|_F^2 + {1\over n}\sum_{j=1}^{p}\sum_{(u,v)\in \Nc_j} \langle \bbZ_{j,uv}, X_j+X_0\rangle \\
\nonumber
&&+ {1\over n}\sum_{j=1}^{p}\sum_{(u,v,w)\in \Sc_j} y_{j,uvw}\langle \bbZ_{j,uvw}, X_j+X_0\rangle
- {1\over n}\sum_{j=1}^{p}\sum_{(u,v,w)\in \Sc_j}\phi^*(y_{j,uvw}) .
\end{eqnarray}
Here, each dual variable $y_{j,uvw}$ corresponds to the matrix $\bbZ_{j,uvw}$ and the constraint $d_j^2(z_u,z_w) - d_j^2(z_u,z_v) \ge 1 - \xi_{uvw}$ in  \eqref{eq:original} for all $j\in[p]$ and $(u,v,w)\in\Sc_j$. We stack all the dual variables $y_{j,uvw}$ into a single column vector $y\in\Rb^{n}$ and $y_s$ represents the $s$th coordinate of $y$. Let $\Tc(s)$ and $\bbZ_s$ represent the task and the outer product $y_s$ corresponds to, namely, $\Tc(s)=j$ and $\bbZ_s:= \bbZ_{j,uw} - \bbZ_{j,uv}$ if the new index $s$ corresponds to the original index $(j,uvw)$. Then we have the following more compact formulation: 
\begin{eqnarray}
\label{eq:mt_lmnn}
\min_{X_0,...,X_p\in \Sb_+^d}  \max_{y\in\Rb^n}~~ \sum_{j=0}^{p} g_j(X_j)
+ {1\over n}\sum_{i=1}^{n} y_i\langle \bbZ_i, X_{\Tc(i)}+X_0\rangle
- {1\over n}\sum_{i=1}^{n}\phi^*(y_i),
\end{eqnarray}
where
$g_j(X_j) :=  {\lambda_j\over 2} \|X_j\|_F^2 +{1\over n} \left\langle \bbC_j, X_j\right\rangle$, with $\bbC_j=\sum_{(u,v)\in \Nc_j} \bbZ_{j,uv}$ for $j\in[p]$, and
$g_0(X_0) := {\lambda_0\over 2} \|X_0-I\|_F^2 +{1\over n}\left\langle  \bbC_0, X_0\right\rangle$ with $\bbC_0=\sum_{j=1}^{p}\bbC_j$. Now, we have reduced the MT-LMNN problem to the form of \eqref{eq:matrixspd} and the customized method is shown in Algorithm \ref{alg:mt_lmnn}. The convergence properties similar to Theorem~\ref{thm:conv_strong_astrsym_switch} and \ref{thm:conv_strong_astrsym_switch_gap} immediately follow.

\begin{algorithm}[t]
\caption{DSPDC Customized for MT-LMNN}
\label{alg:mt_lmnn}
\textbf{Input:} $X^{(-1)}=X^{0}=\bar X^{0}\in\mathbb{R}^{d\times d}$, $y^{(-1)}=y^{0}=\by^{0}\in\mathbb{R}^n$, step sizes $\tau,\sigma$, parameter $\theta$, total iteration $S$, sample sizes $1\leq m\leq n$ and $1\leq q\leq p+1$.\\
\textbf{Output:} $X^{S}$ and $y^{S}$;\\
$\bbW_j^{0}=\sum_{i:i\in [n], \Tc(i)=j} \by_{i}^{0}\bbZ_{i}$ for $j=1,2,...,p$ and $\bbW_0^{0}=\sum\limits_{j=1}^p\bbW_j^0=\sum_{i=1}^n \by_{i}^{0}\bbZ_{i}$.\\
$\bbB_j^0=0_{d\times d}$ for $j=1,2,...,p$.\\
For {$t=0,1,2,\ldots,S-1$:}
\begin{enumerate}
\item[]
Randomly choose $\Ic_t\subset\{1,2,...,n\}$ and $\Jc_t\subset\{0,1,...,p\}$ with $|\Ic_t|=m$ and  $|\Jc_t|=q$.
Perform the following updates:
\begin{eqnarray}
y_{i}^{t+1}&=&\left\{\begin{array}{ll}
\argmax\limits_{\beta\in\mathbb{R}}\left\{
{\beta\over n} \left\langle
\bbZ_{i}, \bar X_{\Tc(i)}^{t}+\bar X_0^{t} \right\rangle
-{1\over n}\phi^*(\beta) -{1\over 2\sigma} (\beta - y_{i}^{t})^2 \right\}&\text{if } i\in \Ic_t,\\
y_{i}^{t}&\text{if } i\notin \Ic_t
\end{array}
\right.\\ 
\bbB_j^{t+1}&=&\sum\limits_{i:i\in\Ic_t,\Tc(i)=j} (y_{i}^{t+1}-y_{i}^{t})\bbZ_{i},~~j=1,2,...,p\\
\bbW_j^{t+1}&=&
\left\{
\begin{array}{ll}
\bbW_j^{t}+\frac{n}{m}\bbB_j^{t+1}
+ \frac{m-n}{m} \bbB_j^{t}
&\text{if}~ j\ne 0,\\
\sum_{l=1}^{T}\bbW_l^{t+1} & \text{if}~ j=0.\\
\end{array}\right.\\
\label{eq:M_update}
X_j^{t+1}&=&\left\{\begin{array}{ll}
\argmin\limits_{Q\in\Sb_+^{d}}\left\{
{1\over n}\langle \bbW_j^{t+1}, Q\rangle
+g_j(Q)+\frac{1}{2\tau}\|Q-X_j^{t}\|_F^2\right\}&\text{if }j\in \Jc_t,\\
X_j^{t}&\text{if }j\notin \Jc_t
\end{array}
\right.\\
\label{eq:bM_update}
\bar X_j^{t+1}&=&X_j^{t}+(\theta+1)(X_j^{t+1}-X_j^{t}), ~\forall j=0,1,...,p.
\end{eqnarray}
\end{enumerate}
\end{algorithm}


\section{Numerical Experiments}
\label{sec:numerical}
In this section, we conduct numerical experiments to compare the DSPDC method with  two other popular stochastic coordinate methods, SPDC~\citep{ZhangXiao:14} and SDCA~\citep{SSZhang13SDCA}
on three scenarios. The first two are empirical risk minimizations, with one applied on factorized data (see Section~\ref{sec:implementation}) and the other using matrices as decision variables (see Section~\ref{sec:matrix_risk}), respectively. Those experiments are run on somewhat synthetic data and serve as the first step of sanity check for the convergence speed. The third is a multi-task large margin nearest neighbor metric learning problem (see Section~\ref{sec:mt_lmnn}) on a real dataset. In a nutshell, we show that DSPDC outperforms the competitors in terms of running time in all the experiments.

\subsection{Learning with factorized data}
We first consider the binary classification problem with smoothed hinge loss under the sparse recovery setting. Besides, we work on a low-dimensional feature space where random feature reduction is applied.
That being said, we are solving the problem~\eqref{eq:srir}
with $\phi_i(z)$ given by \eqref{eq:ssvmloss}.

For the experiments over synthetic data, we first generate a random matrix $X\in\reals^{n\times p}$ with $X_{ij}$ following i.i.d.\ standard normal distribution. We sample a random vector $\beta\in\mathbb{R}^p$ with $\beta_j=1$ for $j=1,2,\dots,50$ and $\beta_j=0$ for $j=51,52,\dots,p$ and use $\beta$ to randomly generate $b_i$ with the distribution $\mbox{Pr}(b_i=1|\beta)=1/(1+e^{-X_i^T\beta})$ and $\mbox{Pr}(b_i=-1|\beta)=1/(1+e^{X_i^T\beta})$. To construct factorized data, we
generate a random matrix $G\in\reals^{d\times p}$ with $d<p$ and $G_{ij}$ following i.i.d.\ normal distribution $\mathcal{N}(0,1/d)$. Then, the factorized data $A=UV$ for \eqref{eq:erm} is constructed with $U=XG^T$ and $V=G$.

To demonstrate the effectiveness of these three methods under different settings,  we choose different values for $(n,m,p,q,d)$ and the regularization parameters $(\lambda_1,\lambda_2)$ in~\eqref{eq:srir}. The numerical results are presented in Figure~\ref{fig:synthetic} with the choices of parameters stated at its bottom. Here, the horizontal axis represents the
running time of an algorithm while the vertical axis represents the  primal gap in logarithmic scale.
According to Figure~\ref{fig:synthetic},
DSPDC is significantly faster than both SPDC and SCDA, under these settings.

\begin{figure*}[htp]
\centering
    \subfigure{
        \centering
        \includegraphics[width=0.31\columnwidth]{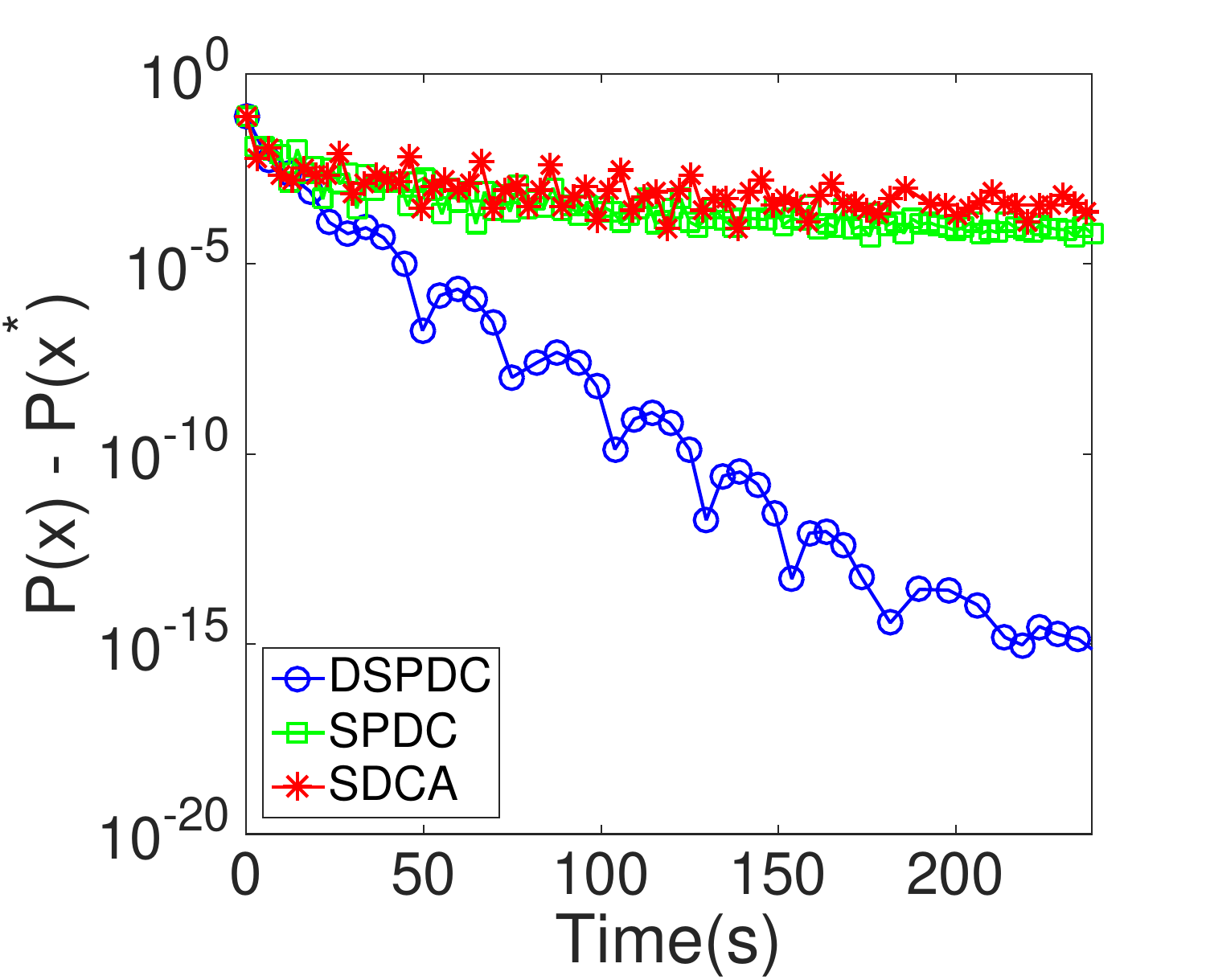}
        \label{fig:n5000lambda01}
    }
    \subfigure{
        \centering
        \includegraphics[width=0.31\columnwidth]{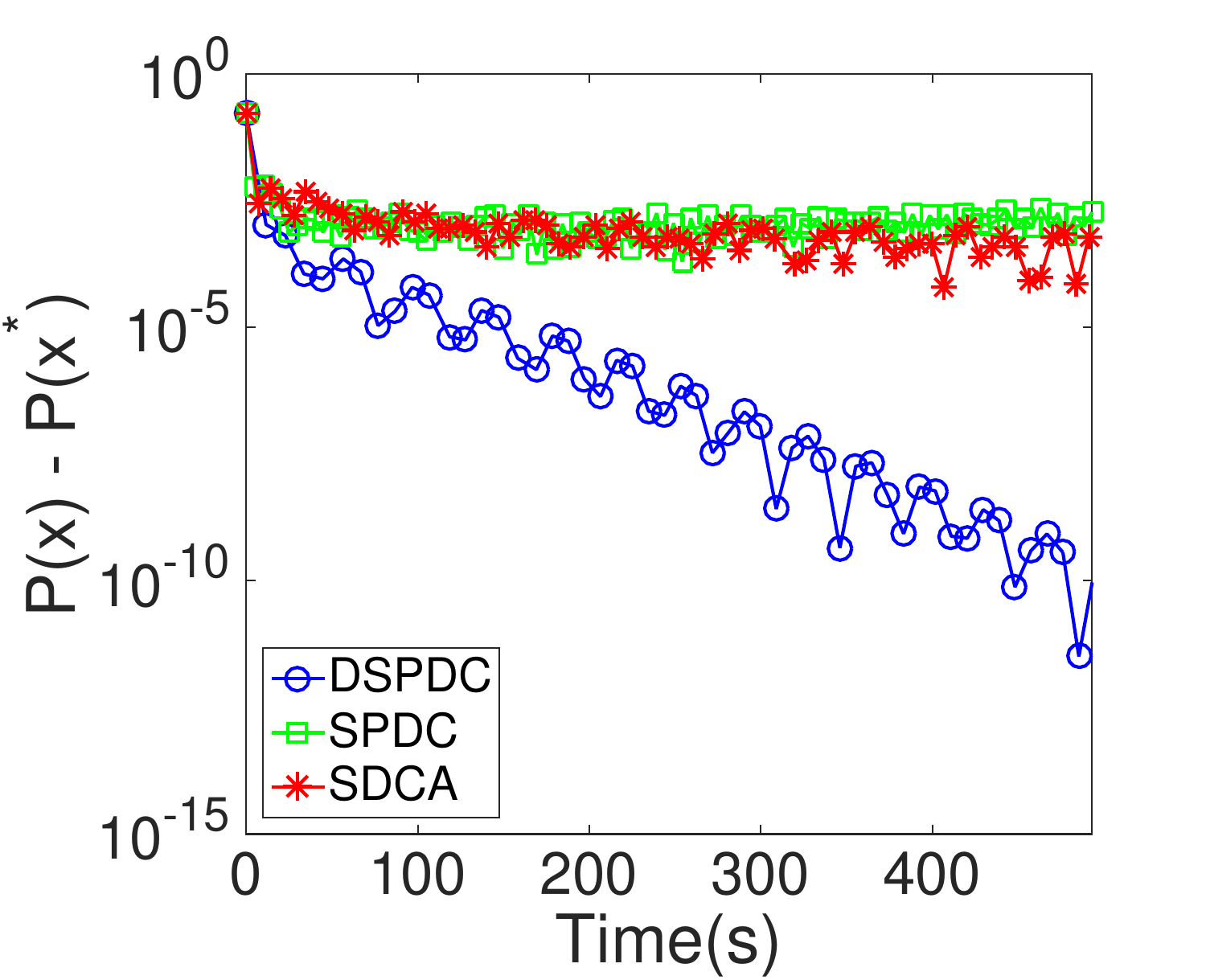}
        \label{fig:n10000lambda01}
    }
    \subfigure{
        \centering
        \includegraphics[width=0.31\columnwidth]{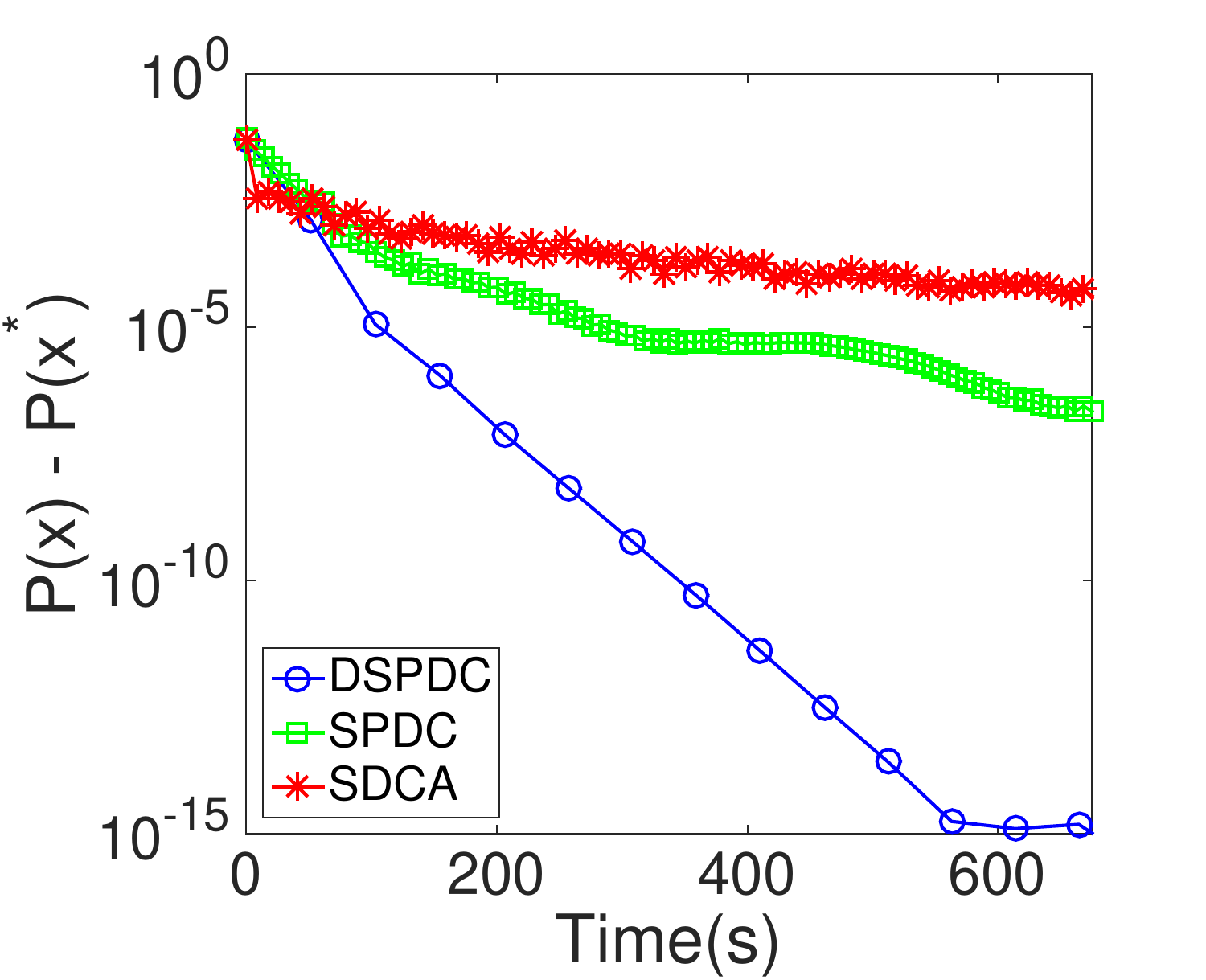}
        \label{fig:n50000lambda01}
    }
    \subfigure{
        \centering
        \includegraphics[width=0.31\columnwidth]{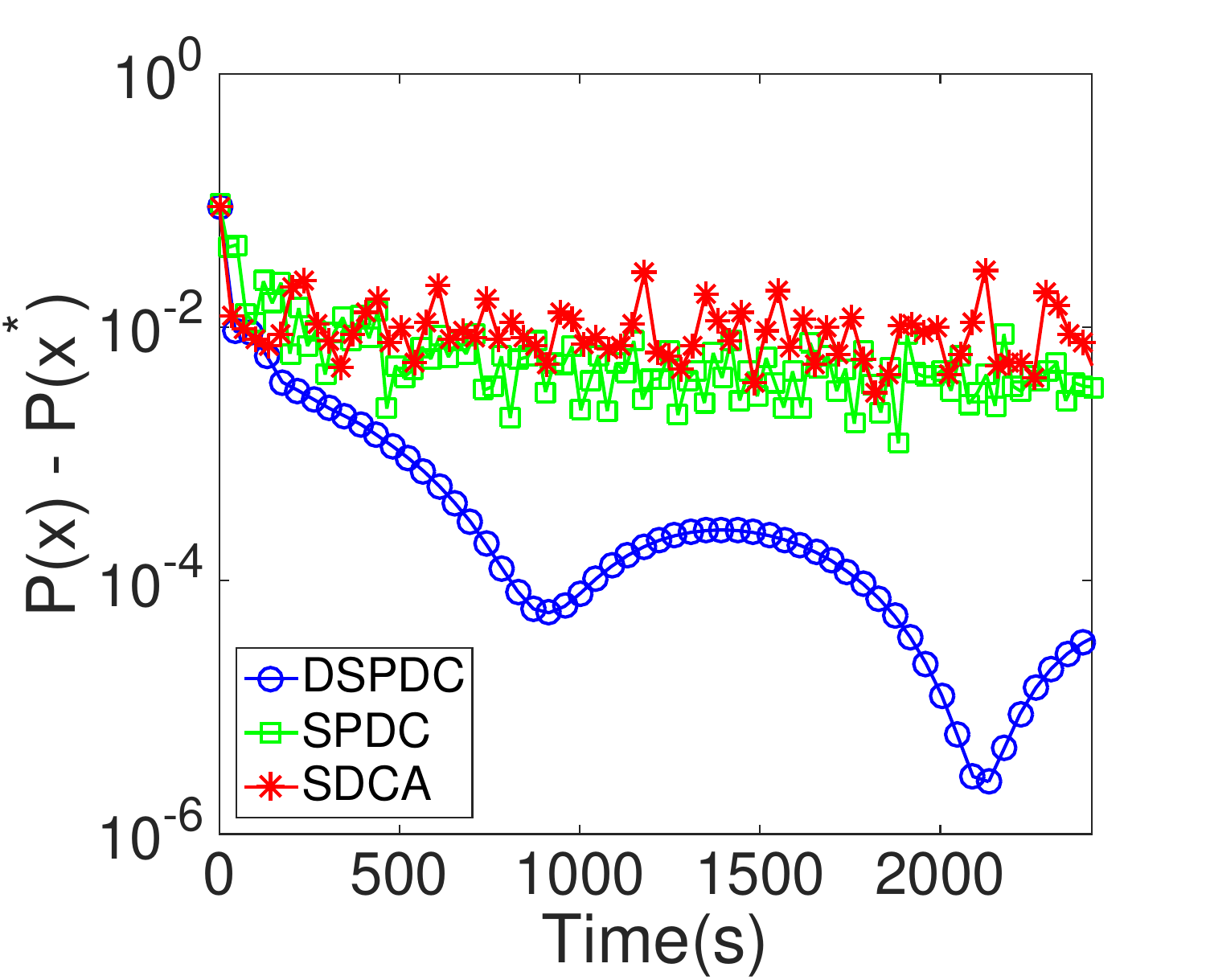}
        \label{fig:n5000lambda1e-6}
    }
    \subfigure{
        \centering
        \includegraphics[width=0.31\columnwidth]{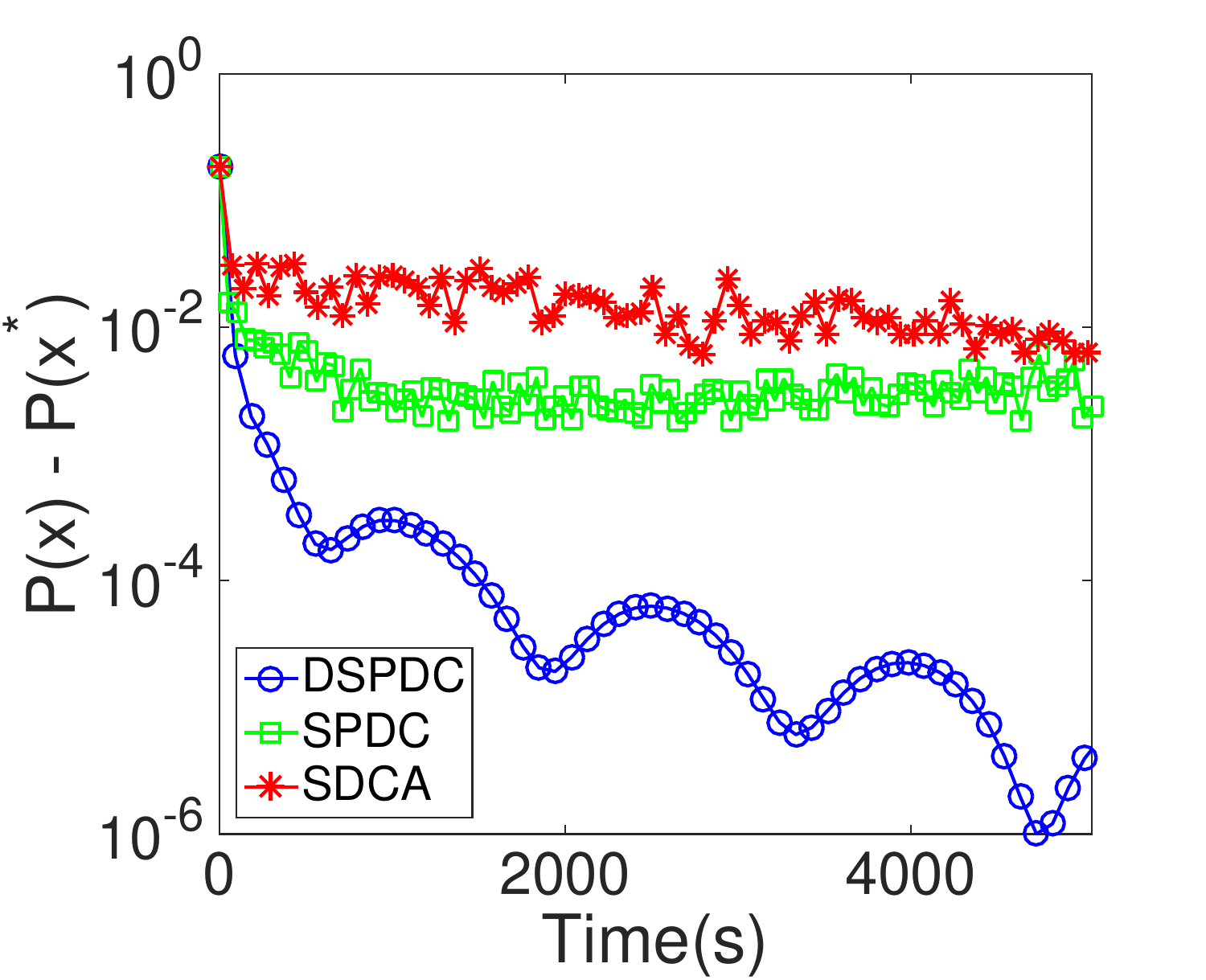}
        \label{fig:n10000lambda1e-6}
    }
    \subfigure{
        \centering
        \includegraphics[width=0.31\columnwidth]{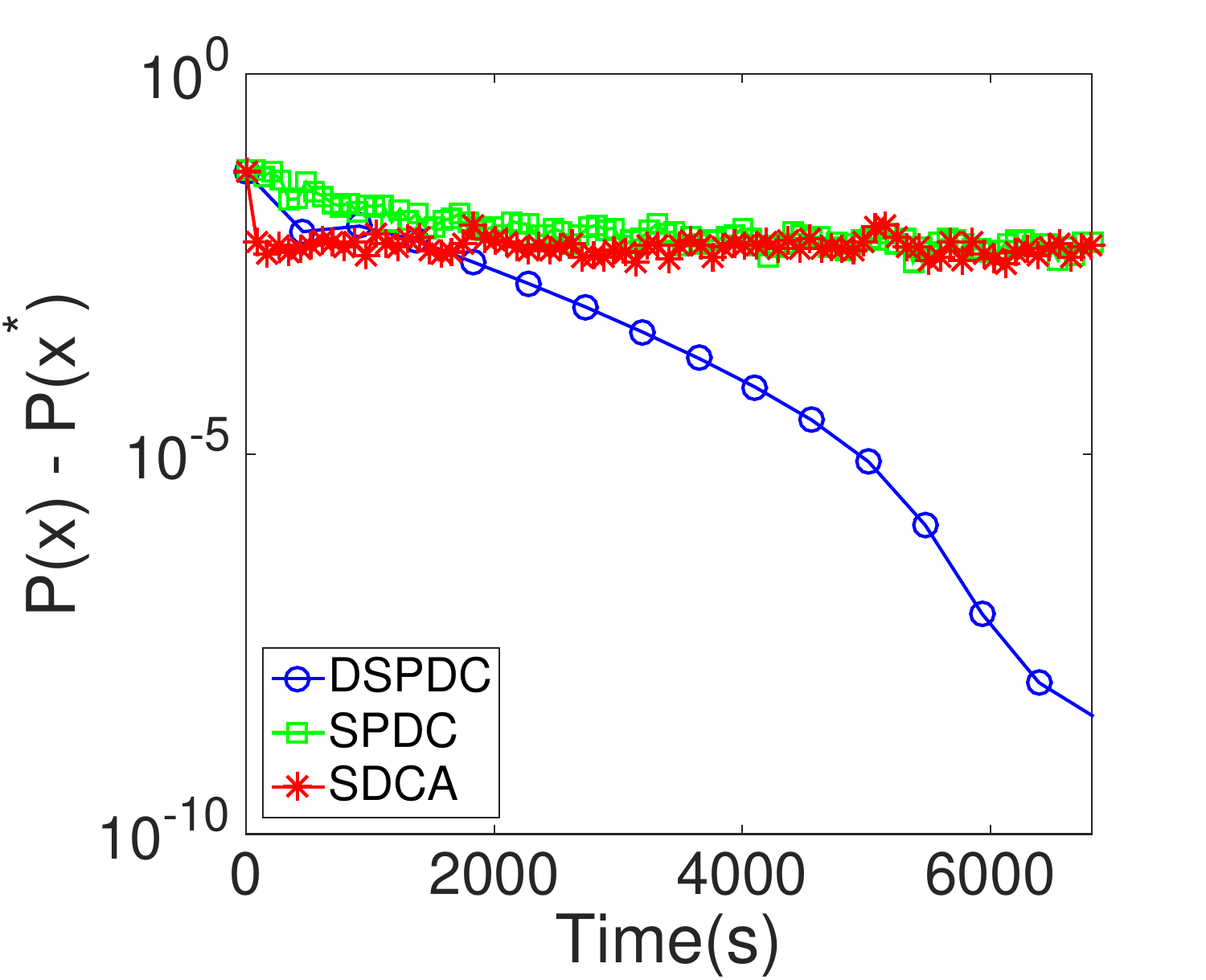}
        \label{fig:n50000lambda1e-6}
    }
    \vspace{-10pt}
  \caption{For all cases, $m=1$. First row: $\lambda_1=10^{-3}, \lambda_2 = 10^{-2}$; Second row: $\lambda_1=10^{-6},\lambda_2=10^{-5}$. First column: $(n,p,q,d)=(5000,100, 50,20)$; Second column: $(n,p,q,d)=(10000,100, 50,50)$; Third column: $(n,p,q,d)=(10000,500, 50,50)$.}
   \label{fig:synthetic}
\end{figure*}

\begin{figure}[htp]
	\centering
	\includegraphics[width=0.31\columnwidth]{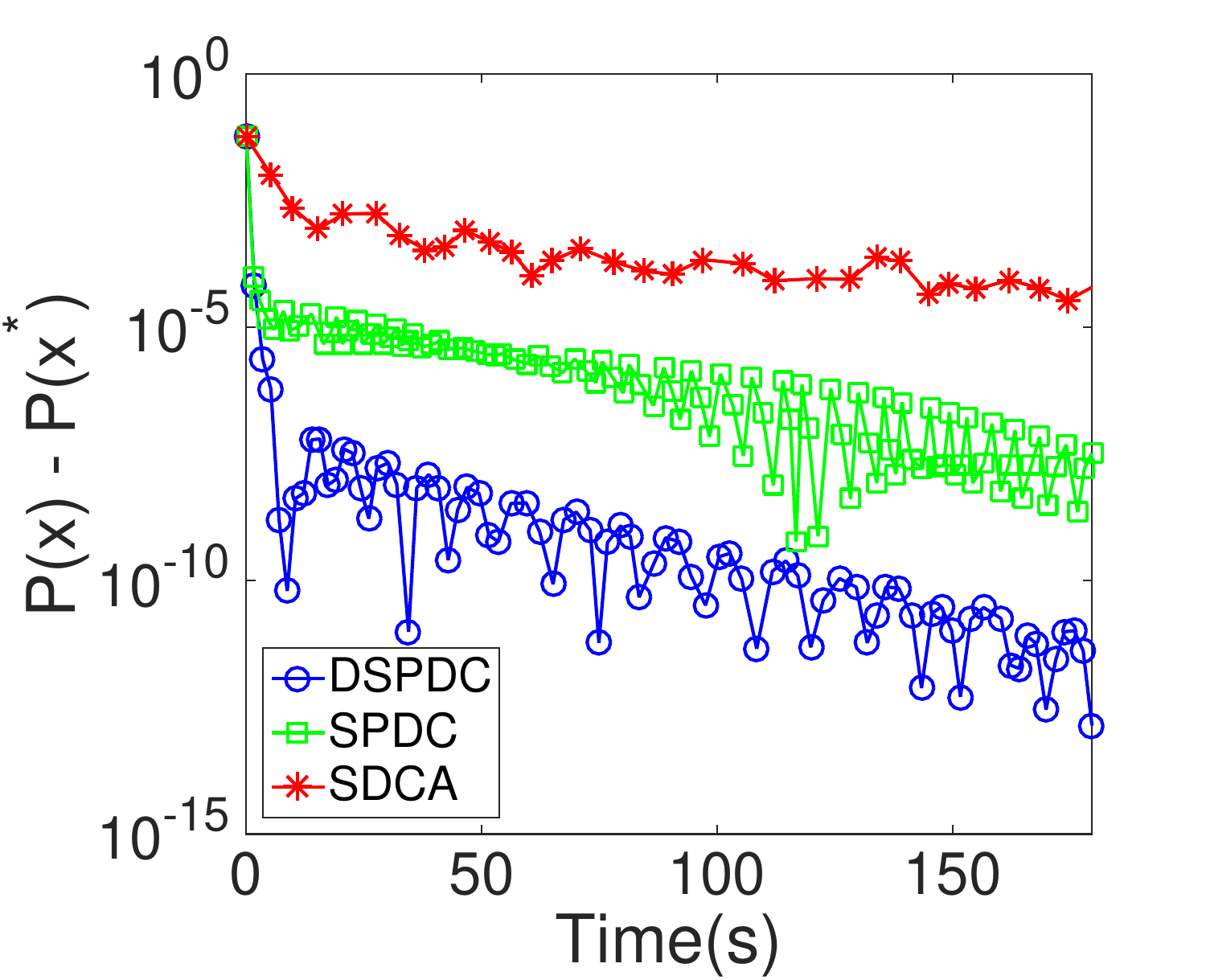}
	\includegraphics[width=0.31\columnwidth]{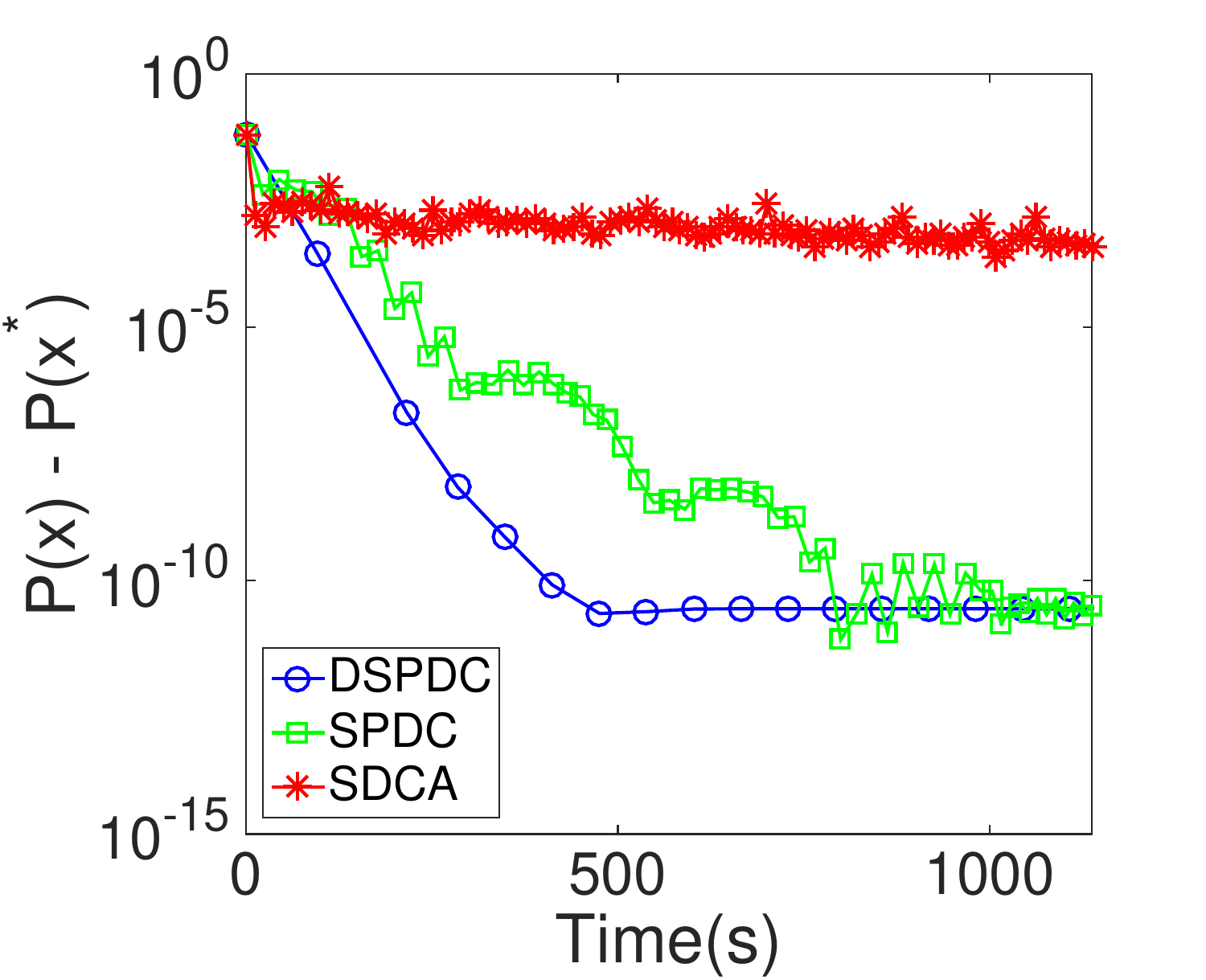}
	\includegraphics[width=0.31\columnwidth]{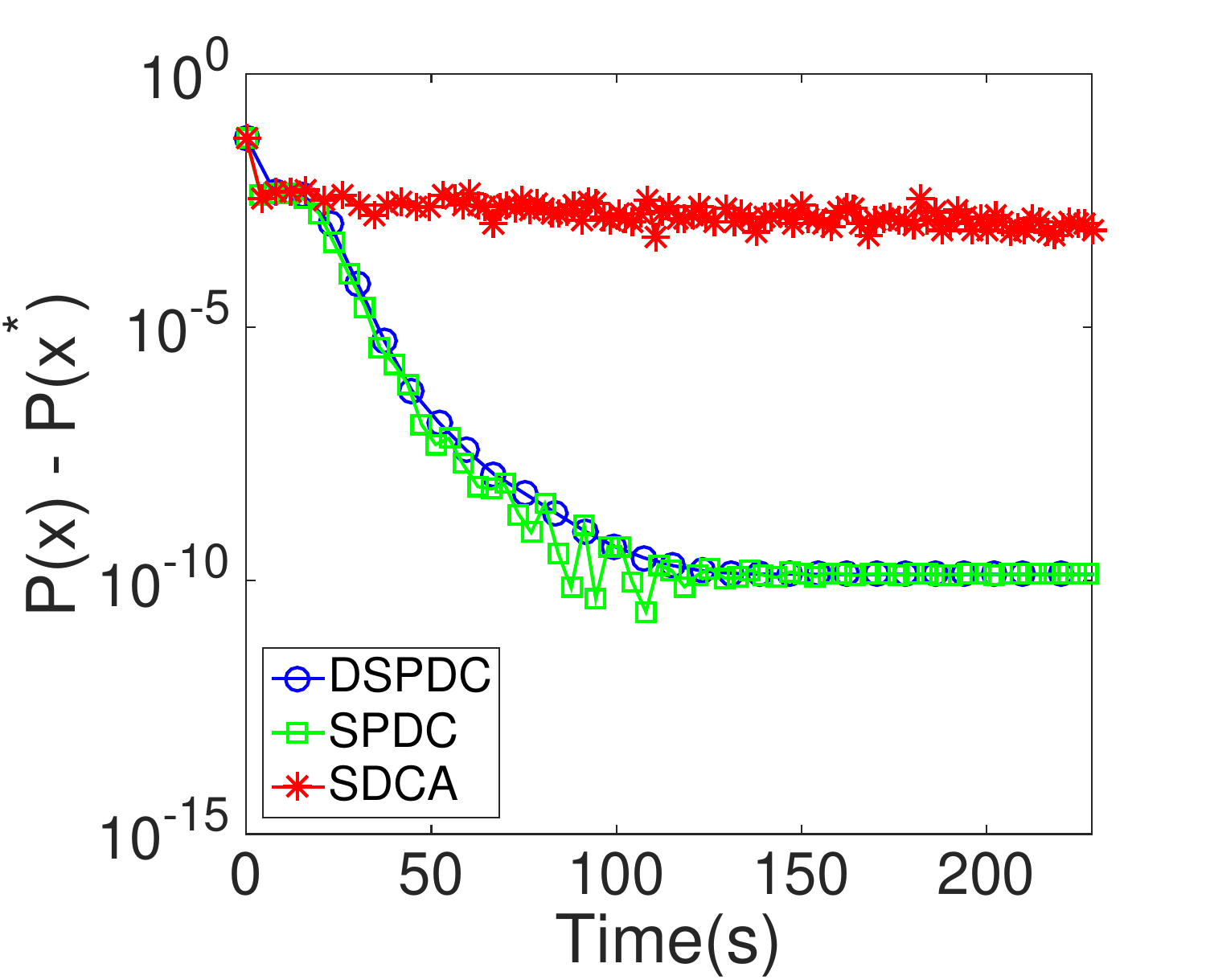}
	\vspace{-5pt}
	\caption{Performance on real datasets. Left: Covtype. Middle: RCV1. Right: Real-sim.}
	\label{fig:realdata}
	\vspace{-5pt}
\end{figure}





We then conduct the comparison of these methods over three real datasets\footnote{\url{http://www.csie.ntu.edu.tw/~cjlin/libsvmtools/datasets/binary.html}}: Covtype ($n=581012,p=54$), RCV1 ($n=20242,p=47236$), and Real-sim ($n=72309,p=20958$). We still consider the sparse recovery problem from feature reduction which is formulated as~\eqref{eq:srir} with $\phi_i$ defined as ~\eqref{eq:ssvmloss}. In all experiments, we choose $d=20$ to generate the random matrix $G$ and set $\lambda_1=10^{-4}$, $\lambda_2=10^{-2}$ in~\eqref{eq:srir}.
We choose $m$ and $q$ so that $n$ and $p$ can be either dividable by them or has a small division remainder. The numerical performances of the three methods are shown in Figure~\ref{fig:realdata}.
In these three examples, SPDC and DSPDC both outperform SDCA significantly. Compared to SPDC, DSPDC is even better on the first two datasets and has the same efficiency on the third.

\subsection{Matrix Risk Minimization}
Next we study the performance of DSPDC for solving the multiple-matrix risk minimization problem \eqref{eq:matrixerm}. We choose $\phi_i$ in \eqref{eq:matrixerm} to be \eqref{eq:ssvmloss} and generate $\bbD_i^j$ as a $d\times d$ matrix with entry sampled from a standard Gaussian distribution for $i=1,2,...,n$ and $j=1,2,...,p$.
Then we generate the true parameter matrix $\bar X_j$ as a $d\times d$ identity matrix for $j=1,2,...,p$. Then we use $\bar X_j$ and $\bbD_i^j$ to generate $b_i$ such that
$b_i = 1$ if ${1\over (1+\exp\{-\langle\bbD_i^j, \bar X_j\rangle \})}> 0.5$ or $b_i = -1$ otherwise.
In this experiment, we set $d=100$ or 200, $p = 100, n=100, \lambda = 0.01$.

We compare the performance of DSPDC, SPDC~\citep{ZhangXiao:14} and SDCA~\citep{SSZhang13SDCA} with various sampling settings, and the results are shown in Fig \ref{fig:matrix}.
It can be easily seen that DSPDC converges much faster than both SPDC and SDCA, in terms of running time. The behaviors of these algorithms are due to the fact that, in each iteration, both SPDC and SDCA need to take $p$ eigenvalue decompositions of $d\times d$ matrix while DSPDC only needs $q$ such operations. Since the cost of each eigenvalue decomposition is as expensive as $O(d^3)$, the total computation cost saved by DSPDC is thus significant.
\begin{figure}[htp]
\centering
\includegraphics[width=0.31\columnwidth]{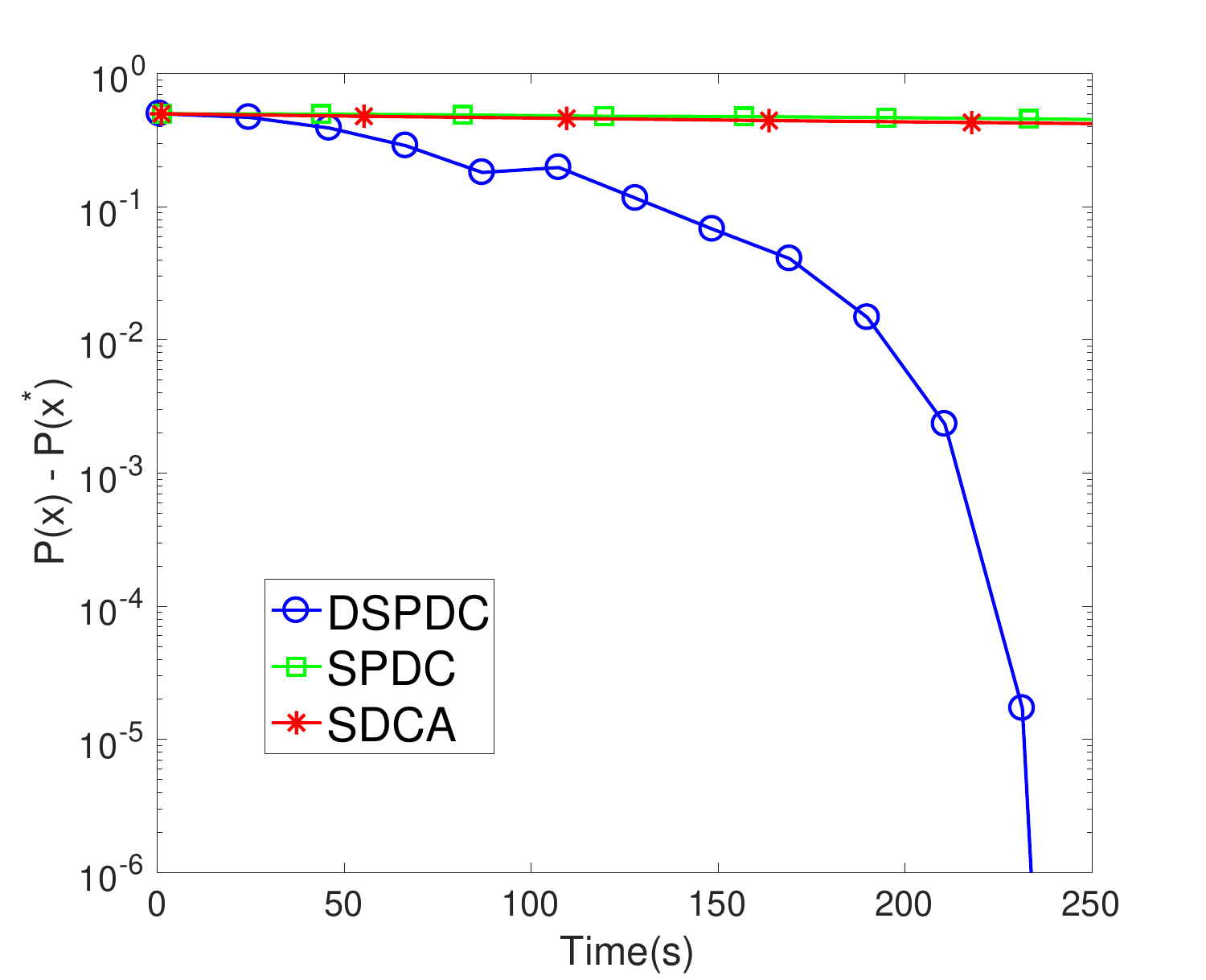}
\includegraphics[width=0.31\columnwidth]{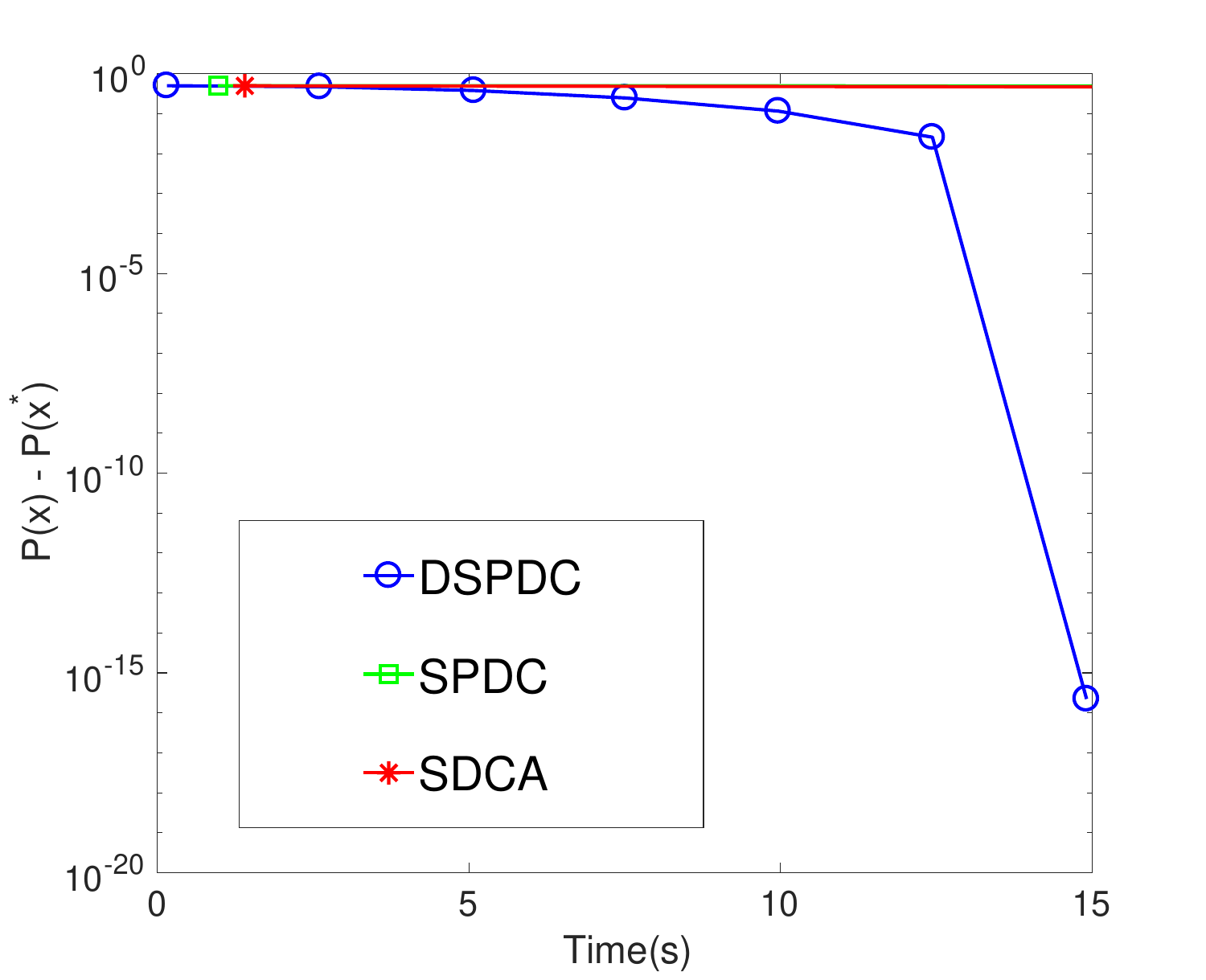}
\includegraphics[width=0.31\columnwidth]{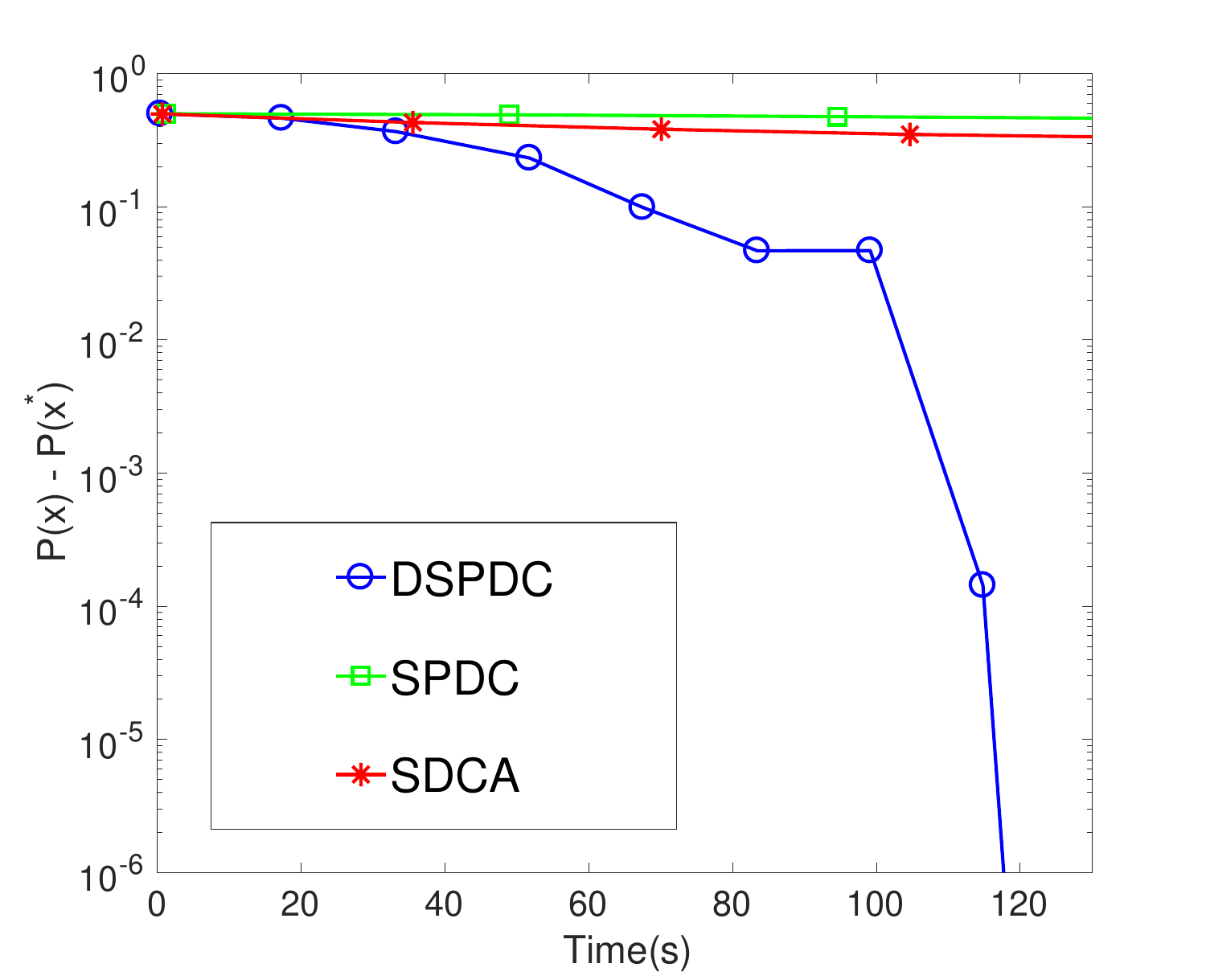}
\includegraphics[width=0.31\columnwidth]{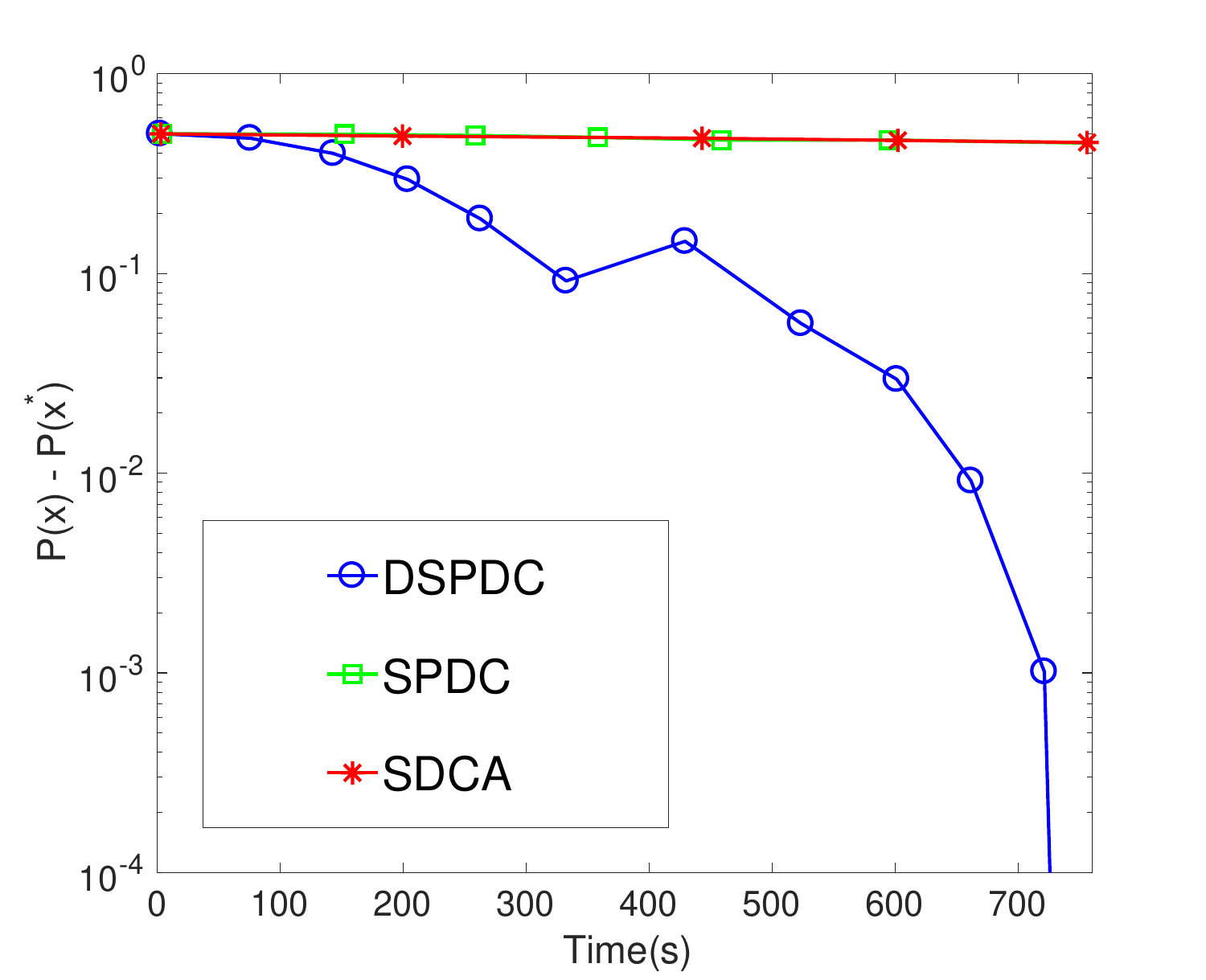}
\includegraphics[width=0.31\columnwidth]{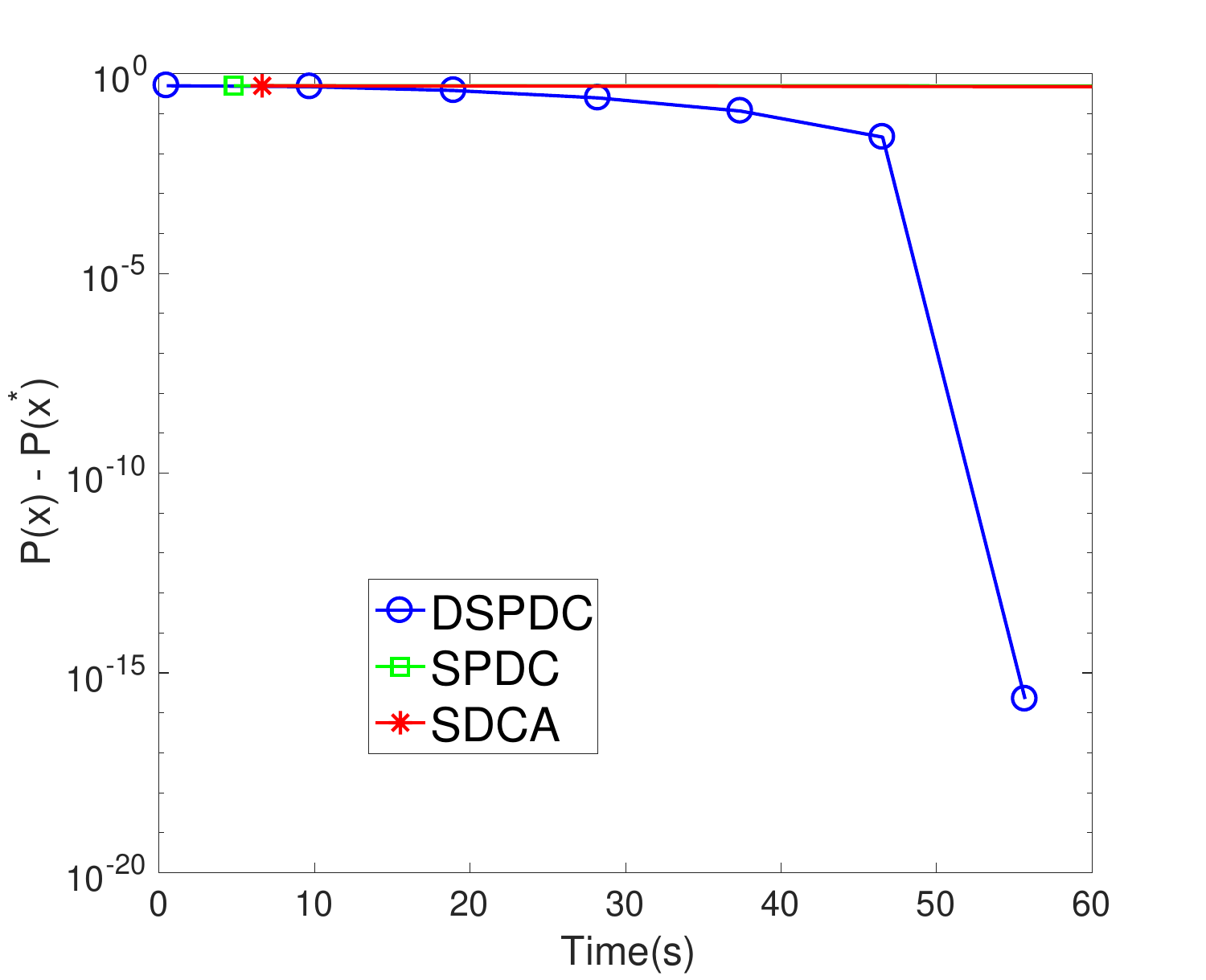}
\includegraphics[width=0.31\columnwidth]{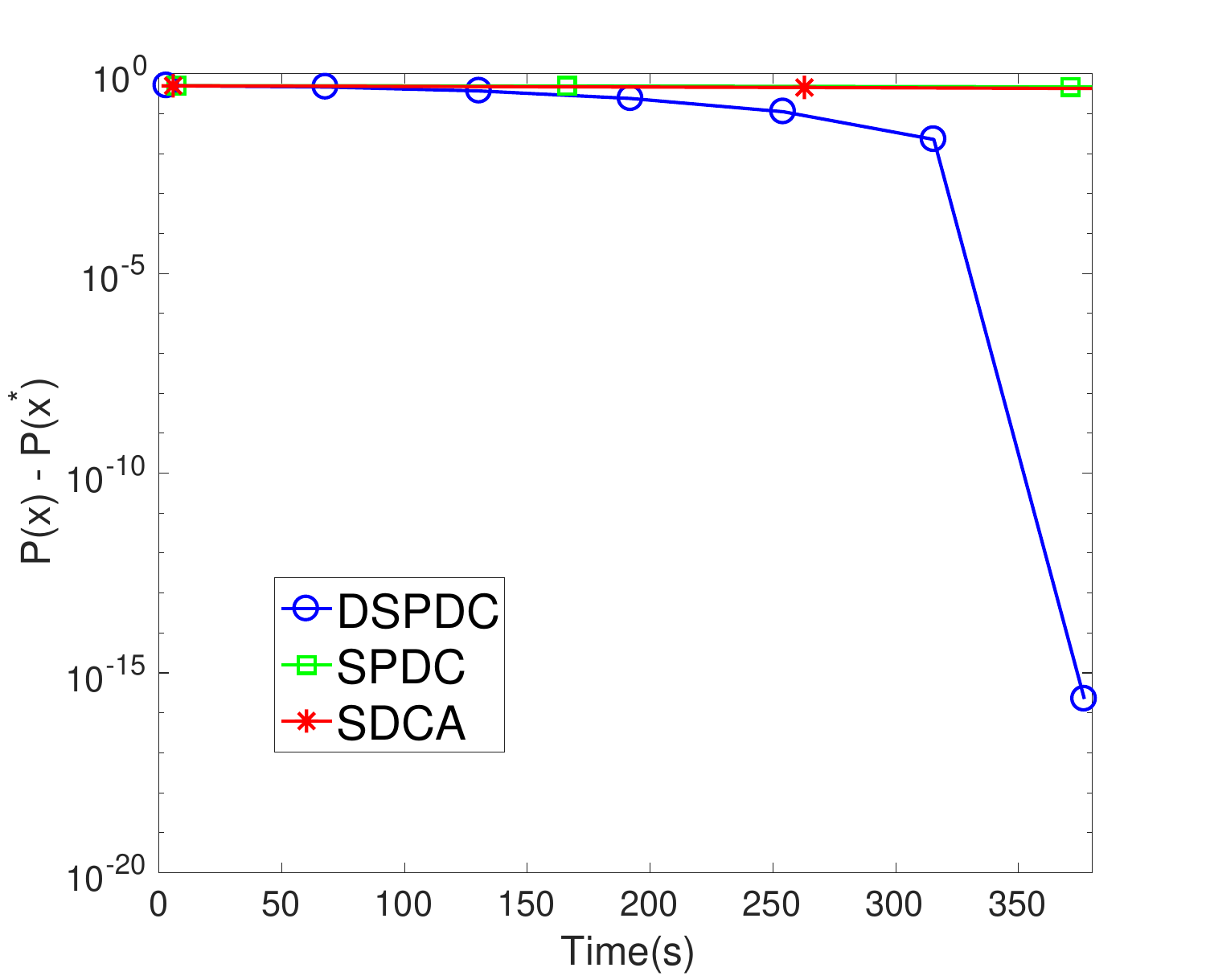}
\vspace{-5pt}
\caption{Performance on matrix risk minimization. First row: $d=100$. Second row: $d=200$. Left: $(m,q)=(5,50)$. Middle: $(m,q)=(50,5)$. Right: $(m,q)=(20,20)$. }
\label{fig:matrix}
\end{figure}


\subsection{Multi-task  Large Margin Nearest Neighbor Problem}\label{sec:mt_lmnn_exp}
Finally, we compare the performance of different algorithms on the MT-LMNN problem \eqref{eq:mt_lmnn}.
The dataset we are using is Amsterdam library of objects ALOI\footnote{\url{https://www.csie.ntu.edu.tw/~cjlin/libsvmtools/datasets/multiclass.html\#aloi}}, a collection of 108,000 images for small objects with 1,000 class labels.  Each image contains one small object which can be expressed as an extended color histogram of $d=128$ dimensions. We adopt the approach similar to~\citep{ParameswaranW10} to generate classification tasks. More specifically, 
we divide the class labels into 100 pieces, each having 10 labels. In other words, we have $100$ metric matrices to learn during training, as well as the one shared by all the tasks. The neighborhood size is $\ell=3$.
For each task, we randomly select 60\% of the data for training, resulting in a training set of 63936 instances. Under this setting, the total number of triplet constraints is $n=1142658$, which is also the number of dual variables.
We set $\lambda_0 = 0.01, \lambda_1=\cdots=\lambda_p=0.1$.

The comparison is again between DSPDC, SPDC and SDCA under different sampling schemes, which is shown in Figure~\ref{fig:mt-lmnn}. We can observe the similar trends as Figure~\ref{fig:matrix}. In particular, DSPDC converges much faster to the optimal solution in terms of running time than both SPDC and SDCA, under all the sampling settings. The superiority of DSPDC in terms of running time is again due to the much less eigenvalue decomposition it does per iteration, which is the benefit brought by primal sampling. Indeed, as both SPDC and SDCA need to do full primal coordinate update, they have to carry out ${p\over q}$ times more eigenvalue decompositions than DSPDC. While all those methods have similar linear convergence rates, the computational cost per iteration dominates the performance.

\begin{figure}[htb]
\begin{center}
\includegraphics[width=0.32\columnwidth]{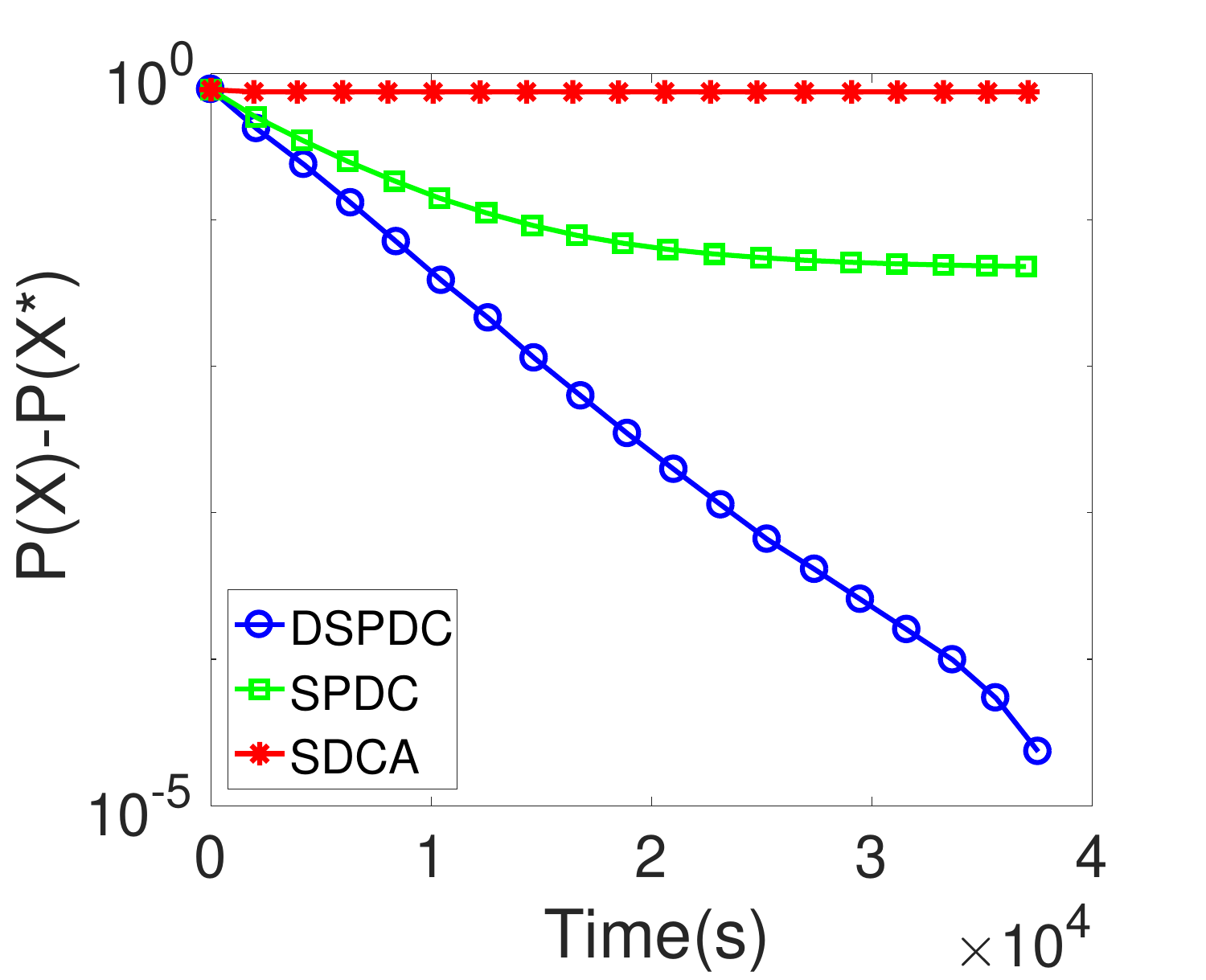}
\includegraphics[width=0.32\columnwidth]{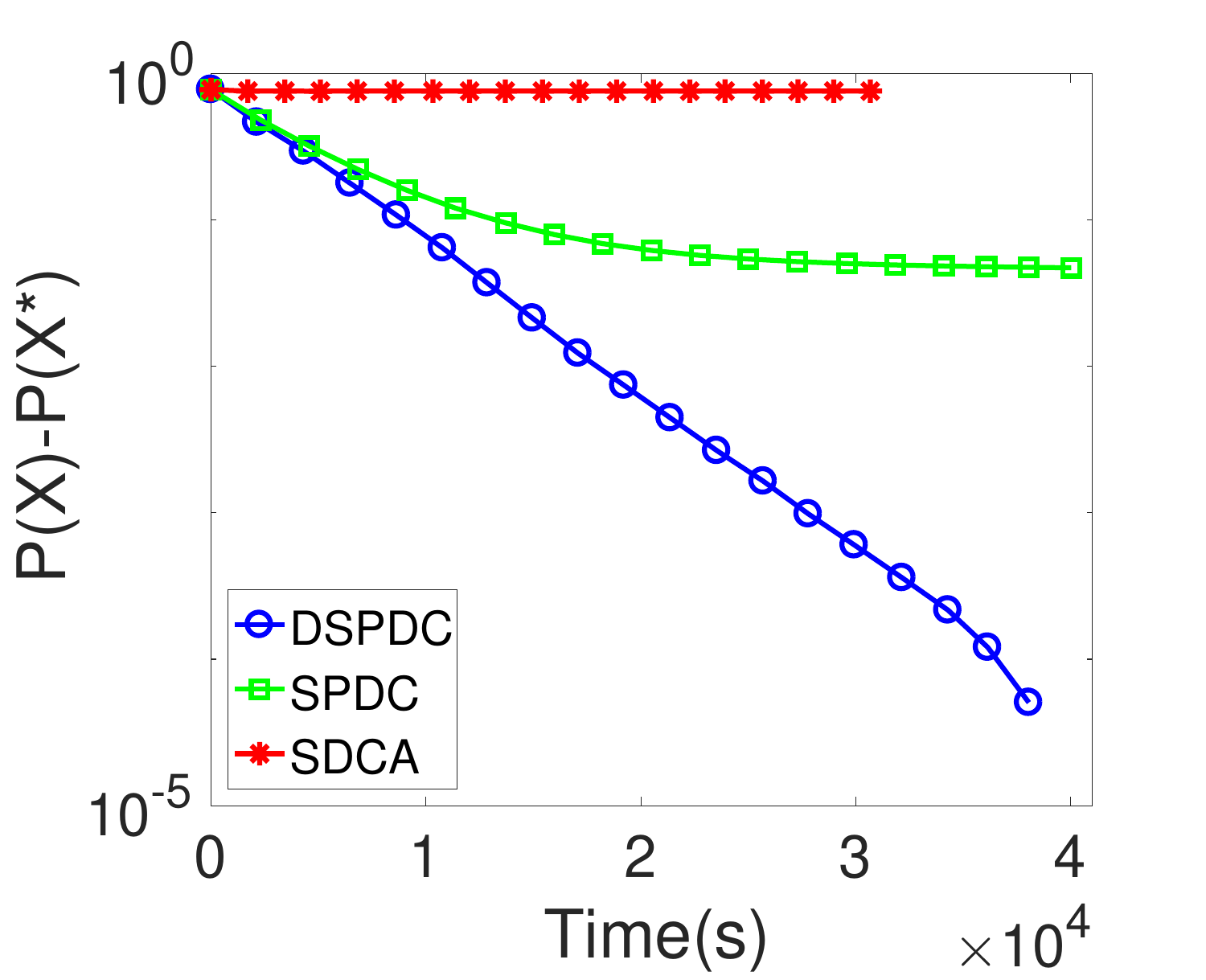}
\includegraphics[width=0.32\columnwidth]{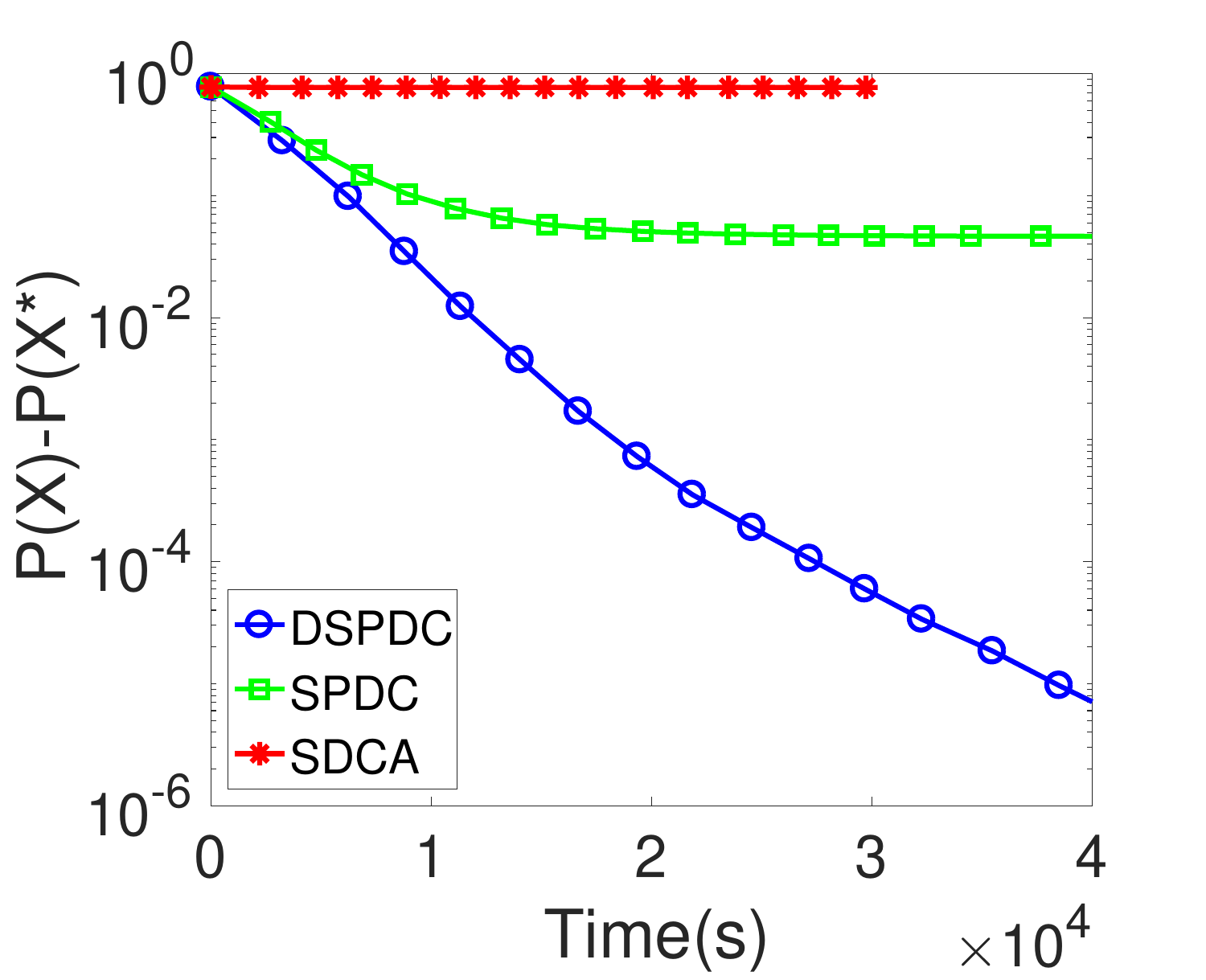}
\caption{The result of different methods on large margin multi-task metric learning problem with ALOI data. There are $p=100$ tasks and thus $p+1=101$ metric matrices to be learned, each being $128\times 128$. The dual variable (triplet constraint) size is $n=1142658$. For left to right, the primal and dual sampling sizes of DSPDC are respectively $(q,m)=(20, 2000), (q,m)=(20, 4000)$ and $(q,m)=(40, 8000)$. For SPDC and SDCA, the dual sampling sizes are the same as DSPDC while they conduct full primal coordinate update ($q=101$).}
\label{fig:mt-lmnn}
\end{center}

\end{figure}

\section{Conclusion}\label{sec:conclusion}
We propose a doubly stochastic primal dual coordinate (DSPDC) method for bilinear saddle point problem, which captures an important class of regularized empirical risk minimization (ERM) problems in statistical learning. We establish the iteration complexity of DSPDC for finding a pair of primal and dual solutions with a $\epsilon$-distance to the optimal solution or with a $\epsilon$-objective gap. When applied to ERM with factorized data or matrix variables with costly prox-mapping, such as the multi-task large margin nearest neighbor metric learning problem, our method achieves a lower overall complexity than existing coordinate methods.


\newpage
\appendix
\section{Convergence Analysis}
\label{sec:proofs}
In this section, we provide the detailed proof of the main theoretical results in Section \ref{sec:DSPDC}.
\subsection{Some technical lemmas}

In order to prove Theorem \ref{thm:conv_strong_astrsym_switch}, we first present the following two technical lemmas which are extracted but extended from \cite{ZhangXiao:14}. In particular, the second inequality in both lemmas are given in \cite{ZhangXiao:14} while the first inequality is new and is the key to prove the convergence in objective gap. These lemmas establish the relationship between two consecutive iterates, $(x^{(t)},y^{(t)})$ and $(x^{(t+1)},y^{(t+1)})$.
\begin{lemma}
\label{thm:ZhangXiaolemma}
Given any $\bx\in\mathbb{R}^p$ and $v\in\mathbb{R}^p$, if we uniformly and randomly choose a set of indices $J\subset\{1,2,\dots,p\}$ with $|J|=q$ and solve an $\hx\in\mathbb{R}^p$ with
\begin{eqnarray}
\label{eq:x_subproblem}
\hx_j=\left\{\begin{array}{ll}
\argmin_{\alpha\in\mathbb{R}}\left\{\frac{1}{n}v_j\alpha+g_j(\alpha)+\frac{1}{2\tau}(\alpha-\bx_j)^2\right\}&\text{if }j\in J\\
\bx_j&\text{if }j\notin J,
\end{array}
\right.
\end{eqnarray}
then, any  $x\in\mathbb{R}^p$, we have
\begin{eqnarray*}
\nonumber
&&\left(\frac{p}{2q\tau}+\frac{(p-q)\lambda}{2q}\right)\|x-\bx\|^2+\frac{p-q}{q}\left(g(\bx)-g(x)\right)\\
&\geq&\left(\frac{p}{2q\tau}+\frac{p\lambda}{2q}\right)\mathbb{E}\|\hx-x\|^2
+\frac{p}{2q\tau}\mathbb{E}\|\hx-\bx\|^2+\frac{p}{q}\mathbb{E}\left(g(\hx)-g(x)\right)
+\frac{1}{n}\mathbb{E}\left\langle v,\bx+\frac{p}{q}(\hx-\bx)-x\right\rangle
\end{eqnarray*}
\normalsize
and
\begin{eqnarray*}
\nonumber
&&\left(\frac{p}{2q\tau}+\frac{(p-q)\lambda}{q}\right)\|x^\star-\bx\|^2\\
&\geq&\left(\frac{p}{2q\tau}+\frac{p\lambda}{q}\right)\mathbb{E}\|\hx-x^\star\|^2
+\frac{p}{2q\tau}\mathbb{E}\|\hx-\bx\|^2+\frac{1}{n}\mathbb{E}\left\langle v-A^Ty^\star,\bx+\frac{p}{q}(\hx-\bx)-x^\star\right\rangle,
\end{eqnarray*}
\normalsize
where the expectation $\mathbb{E}$ is taken over $J$.
\end{lemma}
\begin{proof}
We prove the first conclusion first. Let $\tx$ defined as
\begin{eqnarray*}
\label{eq:proxtx}
\tx=\argmin_{x\in\mathbb{R}^n}\bigg\{\frac{1}{n}v^T x+g(x)+\frac{1}{2\tau}\|x-\bx\|^2\bigg\}.
\end{eqnarray*}
\normalsize
Therefore, according to \eqref{eq:x_subproblem}, $\hx_j=\tx_j$ if $j\in J$ and $\hx_j=\bx_j$ if $j\notin J$.
Due to the decomposable structure~\eqref{eq:decomp_g} of $g(x)$, each coordinate $\tx_j$ of $\tx$ can be solved independently. Since $g_j$ is $\lambda$-strongly convex, the optimality of $\tx_j$ implies that, for any $x_j\in\mathbb{R}$,
\begin{eqnarray}
\label{eq:lemmaoptcond}
\frac{v_j x_j}{n}+g_j(x_j)+\frac{(x_j-\bx_j)^2}{2\tau}
\geq\frac{v_j\tx_j}{n}+g_j(\tx_j)+\frac{(\tx_j-\bx_j)^2}{2\tau}+\left(\frac{1}{2\tau}+\frac{\lambda}{2}\right)(\tx_j-x_j)^2.
\end{eqnarray}
\normalsize
Since each index $j$ is contained in $J$ with a probability of $\frac{q}{p}$, we have the following equalities
\begin{eqnarray}
\label{eq:exp1}
\mathbb{E}(\hx_j-x_j)^2&=&\frac{q}{p}(\tx_j-x_j)^2+\frac{p-q}{p}(\bx_j-x_j)^2,\\
\label{eq:exp2}
\mathbb{E}(\hx_j-\bx_j)^2&=&\frac{q}{p}(\tx_j-\bx_j)^2,\\
\label{eq:exp3}
\mathbb{E}\hx_j&=&\frac{q}{p}\tx_j+\frac{p-q}{p}\bx_j,\\
\label{eq:exp4}
\mathbb{E}g_j(\hx_j)&=&\frac{q}{p}g_j(\tx_j)+\frac{p-q}{p}g_j(\bx_j).
\end{eqnarray}
\normalsize
Using these equalities, we can represent all the terms in \eqref{eq:lemmaoptcond} involving $\tx_j$ by the terms that only contains $\hx_j$, $\bx_j$ and $x_j$. By doing so and organizing terms, we obtain
\begin{eqnarray*}
&&\left(\frac{p}{2q\tau}+\frac{(p-q)\lambda}{2q}\right)(x_j-\bx_j)^2+\frac{p-q}{q}\left(g_j(\bx_j)-g_j(x_j)\right)- \frac{1}{n}v_j\left(\bx_j+\frac{p}{q}(\mathbb{E}\hx_j-\bx_j)-x_j\right)\\
&\geq&\left(\frac{p}{2q\tau}+\frac{p\lambda}{2q}\right)\mathbb{E}(\hx_j-x_j)^2
+\frac{p}{2q\tau}\mathbb{E}(\hx_j-\bx_j)^2+\frac{p}{q}\mathbb{E}\left(g_j(\hx_j)-g_j(x_j)\right),
\end{eqnarray*}
\normalsize
for any $x_j\in\mathbb{R}$. Then, the first conclusion of Lemma \ref{thm:ZhangXiaolemma} is obtained by summing up the inequality above over the indices $j=1,\dots,p$.

In the next, we prove the second conclusion of Lemma \ref{thm:ZhangXiaolemma}. Choosing $x_i=x_i^\star$ in \eqref{eq:lemmaoptcond}, we obtain
\begin{eqnarray}
\label{eq:lemma11}
\frac{v_j x_j^\star}{n}+g_j(x_j^\star)+\frac{(x_j^\star-\bx_j)^2}{2\tau}\geq\frac{v_j\tx_j}{n}+g_j(\tx_j)+\frac{(\tx_j-\bx_j)^2}{2\tau}+\left(\frac{1}{2\tau}+\frac{\lambda}{2}\right)(\tx_j-x_j^\star)^2.
\end{eqnarray}
\normalsize
According to the property \eqref{eq:def_sp} of a saddle point, the primal optimal solution $x^\star$ satisfies
$x^\star=\argmin_{x\in\mathbb{R}^n}\bigg\{\frac{1}{n}(y^{\star})^T Ax+g(x)\bigg\}$.
Due to the decomposable structure~\eqref{eq:decomp_g} of $g(x)$, each coordinate $x_j^\star$ of $x^\star$ can be solved independently. Since $g_j$ is $\lambda$-strongly convex, the optimality of $x_j^\star$ implies
\begin{eqnarray}
\label{eq:lemma12}
\frac{1}{n}\langle A^j,y^\star\rangle \tx_j+g_j(\tx_j)&\geq&\frac{1}{n}\langle A^j,y^\star\rangle x_j^\star+g_j(x_j^\star)+\frac{\lambda}{2}(\tx_j-x_j^\star)^2.
\end{eqnarray}
\normalsize
Summing up~\eqref{eq:lemma11} and~\eqref{eq:lemma12} gives us
\begin{eqnarray}
\label{eq:lemmaoptcond_star}
\frac{1}{2\tau}(x_j^\star-\bx_j)^2\geq\left(\frac{1}{2\tau}+\lambda\right)(\tx_j -x_j^\star)^2+\frac{1}{2\tau}(\tx_j -\bx_j)^2+\frac{1}{n}(v_j-\langle A^j,y^\star\rangle)(\tx_j-x_j^\star).
\end{eqnarray}
\normalsize
By equalities \eqref{eq:exp1}, \eqref{eq:exp2} and \eqref{eq:exp3}, we can represent all the terms in \eqref{eq:lemmaoptcond_star} that involve $\tx_j$ by the terms that only contain $\hx_j$, $\bx_j$ and $x_j$. Then, we obtain
\begin{eqnarray*}
&&\left(\frac{p}{2q\tau}+\frac{(p-q)\lambda}{q}\right)(x_j^\star-\bx_j)^2-\left(\frac{p}{2q\tau}+\frac{p\lambda}{q}\right)\mathbb{E}(\hx_j-x_j^\star)^2\\
&\geq&\frac{p}{2q\tau}\mathbb{E}(\hx_j-\bx_j)^2
+\frac{1}{n}(v_j-\langle A^j,y^\star\rangle)\left(\bx_j+\frac{p}{q}(\mathbb{E}\hx_j-\bx_j)-x_j^\star\right).
\end{eqnarray*}
\normalsize
Then, the second conclusion is obtained by summing up the inequality above over the indices $j=1,\dots,p$.
\end{proof}

\begin{lemma}
\label{thm:ZhangXiaolemma2}
Given any $v\in\mathbb{R}^n$ and $\by\in\mathbb{R}^n$, if we uniformly and randomly choose a set of indices $I\subset\{1,2,\dots,n\}$ with $|I|=m$ and solve an $\hy\in\mathbb{R}^n$ with
\begin{eqnarray}
\label{eq:y_subproblem}
\hy_i=\left\{\begin{array}{ll}
\argmax_{\beta\in\mathbb{R}}\left\{\frac{1}{n}u_i\beta-\frac{\phi^*_i(\beta)}{n}-\frac{1}{2\sigma}(\beta-\by_i)^2\right\}&\text{if }i\in I\\
\by_i&\text{if }i\notin I,
\end{array}
\right.
\end{eqnarray}
\normalsize
then, any  $y\in\mathbb{R}^n$, we have
\begin{eqnarray*}
\nonumber
&&\left(\frac{n}{2m\sigma}+\frac{(n-m)\gamma}{2mn}\right)\|y-\by\|^2+\frac{n-m}{mn}\sum_{i=1}^n\left(\phi_i^*(y_i^{(t)})
-\phi_i^*(y_i)\right)
+\frac{1}{n}\mathbb{E}\left\langle u,\by+\frac{n}{m}(\hy-\by)-y\right\rangle
\\
&\geq&\left(\frac{n}{2m\sigma}+\frac{\gamma}{2m}\right)\mathbb{E}\|\hy-y\|^2
+\frac{n}{2m\sigma}\mathbb{E}\|\hy-\by\|^2
+\frac{1}{m}\sum_{i=1}^n\mathbb{E}\left(\phi_i^*(y_i^{(t+1)})
-\phi_i^*(y_i)\right)
\end{eqnarray*}
\normalsize
and
\begin{eqnarray*}
\nonumber
&&\left(\frac{n}{2m\sigma}+\frac{(n-m)\gamma}{mn}\right)\|y^\star-\by\|^2\\
&\geq&\left(\frac{n}{2m\sigma}+\frac{\gamma}{m}\right)\mathbb{E}\|\hy-y^\star\|^2
+\frac{n}{2m\sigma}\mathbb{E}\|\hy-\by\|^2-\frac{1}{n}\mathbb{E}\left\langle u-Ax^\star,\by+\frac{n}{m}(\hy-\by)-y^\star\right\rangle.
\end{eqnarray*}
\normalsize
where the expectation $\mathbb{E}$ is taken over $I$.
\end{lemma}

\begin{proof}
The proof is very similar to that of Lemma~\ref{thm:ZhangXiaolemma}, and thus, is omitted.
\end{proof}

\subsection{Convergence in distance to the optimal solution}

We use $\mathbb{E}_{t}$ to represent the expectation conditioned on $y^{(0)},x^{(0)},\dots,y^{(t)},x^{(t)}$, and
$\mathbb{E}_{t+}$ the expectation conditioned on $y^{(0)},x^{(0)},\dots,y^{(t)},x^{(t)},y^{(t+1)}$.
Lemma~\ref{thm:ZhangXiaolemma} and Lemma~\ref{thm:ZhangXiaolemma2} in the previous section provide the basis for the following proposition, which is the key to prove Theorem~\ref{thm:conv_strong_astrsym_switch}.
\begin{proposition}
\label{thm:mainprop_strong_asym}
Let $x^{(t)}$, $x^{(t+1)}$, $y^{(t)}$ and $y^{(t+1)}$ generated as in Algorithm~\ref{alg:aspdc_general} for $t=0,1,\dots$ with the parameters $\tau$ and $\sigma$ satisfying $\tau\sigma=\frac{nmq}{4p\Lambda}$. We have
\begin{eqnarray}
\nonumber
&&\left(\frac{p}{2q\tau}+\frac{(p-q)\lambda}{q}\right)\|x^\star-x^{(t)}\|^2
+\left(\frac{n}{2m\sigma}+\frac{(n-m)\gamma}{mn}\right)\|y^\star-y^{(t)}\|^2\\\nonumber
&&+\frac{\theta}{n}\left\langle  A(x^{(t)}-x^{(t-1)}),y^{(t)}-y^\star\right\rangle
+\frac{\theta\|x^{(t)}-x^{(t-1)}\|^2}{4\tau }\\\nonumber
&\geq&\left(\frac{p}{2q\tau}+\frac{p\lambda}{q}\right)\mathbb{E}_t\|x^{(t+1)}-x^\star\|^2+\left(\frac{n}{2m\sigma}
+\frac{\gamma}{m}\right)\mathbb{E}_t\|y^{(t+1)}-y^\star\|^2\\\nonumber
&&+\left(\frac{p}{2q\tau}-\frac{(n-m) p }{4n\tau q}\right)\mathbb{E}_t\|x^{(t+1)}-x^{(t)}\|^2
+\left(\frac{n}{2m\sigma}-\frac{\theta nq}{4\sigma mp}-\frac{n-m}{4\sigma m}\right)\mathbb{E}_t\|y^{(t+1)}-y^{(t)}\|^2\\\label{thm:prop1}
&&+\frac{p}{nq}\mathbb{E}_t\left\langle A^T(y^{(t+1)}-y^\star), x^{(t+1)}-x^{(t)}\right\rangle.
\end{eqnarray}
\normalsize
\end{proposition}

\begin{proof}
Let $x^{(t)}$, $x^{(t+1)}$ and $\by^{(t+1)}$ generated as in Algorithm~\ref{alg:aspdc_general}. By the second conclusion of Lemma \ref{thm:ZhangXiaolemma} and the tower property $\mathbb{E}_t\mathbb{E}_{t+}=\mathbb{E}_t$, we have
\begin{eqnarray}
\nonumber
\left(\frac{p}{2q\tau}+\frac{(p-q)\lambda}{q}\right)\|x^\star-x^{(t)}\|^2
&\geq&\left(\frac{p}{2q\tau}+\frac{p\lambda}{q}\right)\mathbb{E}_t\|x^{(t+1)}-x^\star\|^2
+\frac{p}{2q\tau}\mathbb{E}_t\|x^{(t+1)}-x^{(t)}\|^2\\\label{eq:lemmaoptcondexp3}\nonumber
&&
+\frac{1}{n}\mathbb{E}_t\left\langle A^T(\by^{(t+1)}-y^\star),x^{(t)}+\frac{p}{q}(x^{(t+1)}-x^{(t)})-x^\star\right\rangle.
\end{eqnarray}
\normalsize
Similarly, let $y^{(t)}$, $y^{(t+1)}$ and $\bx^{(t)}$ generated as in Algorithm~\ref{alg:aspdc_general}. By the second conclusion of Lemma \ref{thm:ZhangXiaolemma2}, we have
\begin{eqnarray}
\nonumber
\left(\frac{n}{2m\sigma}+\frac{(n-m)\gamma}{mn}\right)\|y^\star-y^{(t)}\|^2
&\geq&\left(\frac{n}{2m\sigma}+\frac{\gamma}{m}\right)\mathbb{E}_t\|y^{(t+1)}-y^\star\|^2
+\frac{n}{2m\sigma}\mathbb{E}_t\|y^{(t+1)}-y^{(t)}\|^2\\\label{eq:lemmaontcondeyn4}\nonumber
&&
-\frac{1}{n}\mathbb{E}_t\left\langle A(\bx^{(t)}-x^\star),y^{(t)}+\frac{n}{m}(y^{(t+1)}-y^{(t)})-y^\star\right\rangle.
\end{eqnarray}
\normalsize
Summing up these two inequalities, we have
\begin{eqnarray}
\nonumber
&&\left(\frac{p}{2q\tau}+\frac{(p-q)\lambda}{q}\right)\|x^\star-x^{(t)}\|^2+\left(\frac{n}{2m\sigma}+\frac{(n-m)\gamma}{mn}\right)\|y^\star-y^{(t)}\|^2\\\nonumber
&\geq&\left(\frac{p}{2q\tau}+\frac{p\lambda}{q}\right)\mathbb{E}_t\|x^{(t+1)}-x^\star\|^2+\left(\frac{n}{2m\sigma}
+\frac{\gamma}{m}\right)\mathbb{E}_t\|y^{(t+1)}-y^\star\|^2
+\frac{p}{2q\tau}\mathbb{E}_t\|x^{(t+1)}-x^{(t)}\|^2\\
&&+\frac{n}{2m\sigma}\mathbb{E}_t\|y^{(t+1)}-y^{(t)}\|^2
+\frac{1}{n}\mathbb{E}_t\left\langle A^T(\by^{(t+1)}-y^\star),x^{(t)}+\frac{p}{q}(x^{(t+1)}-x^{(t)})-x^\star\right\rangle\\\label{eq:prop11}
&&-\frac{1}{n}\mathbb{E}_t\left\langle A(\bx^{(t)}-x^\star),y^{(t)}+\frac{n}{m}(y^{(t+1)}-y^{(t)})-y^\star\right\rangle.\nonumber
\end{eqnarray}
\normalsize
By the definition of $\by^{(t+1)}$ in Algorithm~\ref{alg:aspdc_general}, we have
$\by^{(t+1)}-y^\star=y^{(t+1)}-y^\star+\frac{n-m}{m}(y^{(t+1)}-y^{(t)})$, which implies
\begin{eqnarray}
\nonumber
&&\frac{1}{n} \left\langle A^T(\by^{(t+1)}-y^\star),x^{(t)}+\frac{p}{q}(x^{(t+1)}-x^{(t)})-x^\star\right\rangle\\\nonumber
&=&\frac{1}{n} \left\langle A^T(y^{(t+1)}-y^\star+\frac{n-m}{m}(y^{(t+1)}-y^{(t)})),x^{(t)}+\frac{p}{q}(x^{(t+1)}-x^{(t)})-x^\star\right\rangle\\\label{eq:prop12}
&=&\frac{p}{nq} \left\langle A^T(y^{(t+1)}-y^\star),x^{(t+1)}-x^{(t)}\right\rangle
+\frac{1}{n}\left\langle A^T(y^{(t+1)}-y^\star),x^{(t)}-x^\star\right\rangle\\\nonumber
&&+\frac{(n-m) p}{nmq} \left\langle A^T(y^{(t+1)}-y^{(t)})),x^{(t+1)}-x^{(t)}\right\rangle
+\frac{n-m}{nm}\left\langle A^T(y^{(t+1)}-y^{(t)})),x^{(t)}-x^\star\right\rangle.
\end{eqnarray}
\normalsize
Similarly, by the definition of $\bx^{(t)}$ in Algorithm~\ref{alg:aspdc_general}, we have
$\bx^{(t)}-x^\star=x^{(t)}-x^\star+\theta(x^{(t)}-x^{(t-1)})$, which implies
\begin{eqnarray}
\nonumber
&&\frac{1}{n}\left\langle A(\bx^{(t)}-x^\star),y^{(t)}+\frac{n}{m}( y^{(t+1)}-y^{(t)})-y^\star\right\rangle\\\nonumber
&=&\frac{1}{n}\left\langle A(x^{(t)}-x^\star+\theta(x^{(t)}-x^{(t-1)})),y^{(t)}+\frac{n}{m}( y^{(t+1)}-y^{(t)})-y^\star\right\rangle\\\nonumber
&=&\frac{1}{m}\left\langle A(x^{(t)}-x^\star), y^{(t+1)}-y^{(t)}\right\rangle
+\frac{1}{n}\left\langle A(x^{(t)}-x^\star),y^{(t)}-y^\star\right\rangle\\\label{eq:prop13}
&&+\frac{\theta }{m}\left\langle A(x^{(t)}-x^{(t-1)}), y^{(t+1)}-y^{(t)}\right\rangle
+\frac{\theta}{n}\left\langle  A(x^{(t)}-x^{(t-1)}),y^{(t)}-y^\star\right\rangle.
\end{eqnarray}
\normalsize

According to~\eqref{eq:prop12} and~\eqref{eq:prop13}, the last two terms in the right hand side of \eqref{eq:prop11} within conditional expectation $\mathbb{E}_t$ can be represented as
\begin{eqnarray}
\nonumber
&&\frac{1}{n}\left\langle A^T(\by^{(t+1)}-y^\star),x^{(t)}+\frac{p}{q}( x^{(t+1)}-x^{(t)})-x^\star\right\rangle\\\nonumber
&&-\frac{1}{n}\left\langle A(\bx^{(t)}-x^\star),y^{(t)}+\frac{n}{m}( y^{(t+1)}-y^{(t)})-y^\star\right\rangle\\\nonumber
&=&\frac{p}{nq}\left\langle A^T(y^{(t+1)}-y^\star), x^{(t+1)}-x^{(t)}\right\rangle
+\frac{1}{n}\left\langle A^T(y^{(t+1)}-y^\star),x^{(t)}-x^\star\right\rangle\\\nonumber
&&+\frac{(n-m) p}{nmq} \left\langle A^T(y^{(t+1)}-y^{(t)})),x^{(t+1)}-x^{(t)}\right\rangle
+\frac{n-m}{nm}\left\langle A^T(y^{(t+1)}-y^{(t)})),x^{(t)}-x^\star\right\rangle\\\nonumber
&&-\frac{1}{m}\left\langle A(x^{(t)}-x^\star), y^{(t+1)}-y^{(t)}\right\rangle
-\frac{1}{n}\left\langle A(x^{(t)}-x^\star),y^{(t)}-y^\star\right\rangle\\\nonumber
&&-\frac{\theta }{m}\left\langle A(x^{(t)}-x^{(t-1)}), y^{(t+1)}-y^{(t)}\right\rangle
-\frac{\theta}{n}\left\langle  A(x^{(t)}-x^{(t-1)}),y^{(t)}-y^\star\right\rangle\\\nonumber
&=&\frac{p}{nq}\left\langle A^T(y^{(t+1)}-y^\star), x^{(t+1)}-x^{(t)}\right\rangle+\frac{(n-m) p}{nmq}\left\langle A^T(y^{(t+1)}-y^{(t)})), x^{(t+1)}-x^{(t)}\right\rangle\\
&&-\frac{\theta }{m}\left\langle A(x^{(t)}-x^{(t-1)}), y^{(t+1)}-y^{(t)}\right\rangle
-\frac{\theta}{n}\left\langle  A(x^{(t)}-x^{(t-1)}),y^{(t)}-y^\star\right\rangle.  \label{eq:prop3_3}
\end{eqnarray}
\normalsize
In the next, we establish some lower bounds for each of the four terms in \eqref{eq:prop3_3}.

Note that $x^{(t)}-x^{(t-1)}$ is a sparse vector with non-zero values only in the coordinates indexed by $J$. Hence, by Young's inequality, we have
\begin{eqnarray}
\nonumber
-\left\langle A(x^{(t)}-x^{(t-1)}), y^{(t+1)}-y^{(t)}\right\rangle
&\geq&
-\frac{\tau }{m}\|(A^J)^T(y^{(t+1)}-y^{(t)})\|^2-\frac{\|x^{(t)}-x^{(t-1)}\|^2}{4\tau/m}\\\nonumber
&\geq&-\frac{\tau \Lambda}{m}\|y^{(t+1)}-y^{(t)}\|^2-\frac{\|x^{(t)}-x^{(t-1)}\|^2}{4\tau/m}\\\label{eq:young3}
&=&-\frac{nq}{4\sigma p} \|y^{(t+1)}-y^{(t)}\|^2-\frac{\|x^{(t)}-x^{(t-1)}\|^2}{4\tau/m}.
\end{eqnarray}
\normalsize
Here, the second inequality is because $y^{(t+1)}-y^{(t)}$ has non-zero values only in the coordinates indexed by $I$ so that, by the definition \eqref{eq:Lambda} of $\Lambda$,
$$
\|(A^J)^T(y^{(t+1)}-y^{(t)})\|^2=\|(A_I^J)^T(y_I^{(t+1)}-y_I^{(t)})\|^2
\leq \Lambda\|y_I^{(t+1)}-y_I^{(t)}\|^2= \Lambda\|y^{(t+1)}-y^{(t)}\|^2,
$$
\normalsize
and the last equality is because $\tau\sigma=\frac{nmq}{4p\Lambda}$.

A similar argument implies
\begin{eqnarray}
\label{eq:young4}
\left\langle A^T(y^{(t+1)}-y^{(t)}), x^{(t+1)}-x^{(t)}\right\rangle
&\geq&-\frac{nq}{4\sigma p} \|y^{(t+1)}-y^{(t)}\|^2-\frac{\|x^{(t+1)}-x^{(t)}\|^2}{4\tau/m}.
\end{eqnarray}
\normalsize
Applying \eqref{eq:young3} and \eqref{eq:young4} to the right hand side of \eqref{eq:prop3_3}, we have
\begin{eqnarray}
\nonumber
&&\frac{1}{n}\left\langle A^T(\by^{(t+1)}-y^\star),x^{(t)}+\frac{p}{q}( x^{(t+1)}-x^{(t)})-x^\star\right\rangle\\\nonumber
&&-\frac{1}{n}\left\langle A(\bx^{(t)}-x^\star),y^{(t)}+\frac{n}{m}( y^{(t+1)}-y^{(t)})-y^\star\right\rangle\\\label{eq:prop3_4}
&\geq&\frac{p}{nq}\left\langle A^T(y^{(t+1)}-y^\star), x^{(t+1)}-x^{(t)}\right\rangle
-\frac{\theta}{n}\left\langle  A(x^{(t)}-x^{(t-1)}),y^{(t)}-y^\star\right\rangle\\\nonumber
&&-\left(\frac{\theta n q}{4\sigma mp}+\frac{n-m}{4\sigma m}\right)\|y^{(t+1)}-y^{(t)}\|^2-\frac{\theta\|x^{(t)}-x^{(t-1)}\|^2}{4\tau }-\frac{(n-m) p \|x^{(t+1)}-x^{(t)}\|^2}{4\tau n q}
\end{eqnarray}
\normalsize
The conclusion~\eqref{thm:prop1} is obtained by combining \eqref{eq:prop11} and the conditional expectation of \eqref{eq:prop3_4}.
\end{proof}

Based on Proposition~\ref{thm:mainprop_strong_asym}, we can prove Theorem~\ref{thm:conv_strong_astrsym_switch}.
\begin{proof}[\textbf{Theorem~\ref{thm:conv_strong_astrsym_switch}}]
We first show how to derive the forms of $\tau$ and $\sigma$ in \eqref{eq:opttau2} and \eqref{eq:optsigma2} from the last two equations in~\eqref{eq:thetatausigma}. Let
$Q_1=\frac{p}{2q\lambda\tau}$ and $Q_2=\frac{n^2}{2m\gamma\sigma}$.
The last two equations in~\eqref{eq:thetatausigma} imply
\begin{eqnarray*}
Q_1Q_2=\frac{pn^2}{4qm\lambda\gamma\tau\sigma}=\frac{(np)^2\Lambda}{(mq)^2n\lambda\gamma},
\quad Q_1+\frac{p}{q}=Q_2+\frac{n}{m}.
\end{eqnarray*}
\normalsize
Solving the values of $Q_1$ and $Q_2$ from these equations, we obtain
\begin{eqnarray}
\label{eq:T}
\frac{p}{2q\lambda\tau}&=Q_1=&\textstyle\frac{1}{2}\left(\frac{n}{m}-\frac{p}{q}\right)+\frac{1}{2}\sqrt{\left(\frac{n}{m}-\frac{p}{q}\right)^2+\frac{4(np)^2\Lambda}{(mq)^2n\lambda\gamma}},\\
\label{eq:S}
\frac{n^2}{2m\gamma\sigma}&=Q_2=&\textstyle\frac{1}{2}\left(\frac{p}{q}-\frac{n}{m}\right)+\frac{1}{2}\sqrt{\left(\frac{n}{m}-\frac{p}{q}\right)^2+\frac{4(np)^2\Lambda}{(mq)^2n\lambda\gamma}},
\end{eqnarray}
from which \eqref{eq:opttau2} and \eqref{eq:optsigma2} can be derived.

To derive the main conclusion of Theorem~\ref{thm:conv_strong_astrsym_switch} from Proposition~\ref{thm:mainprop_strong_asym}, we want to show that the following inequalities are satisfied by the choices for $\theta$, $\tau$ and $\sigma$ in~\eqref{eq:thetatausigma}.
\begin{eqnarray}
\label{eq:cond5}
\left(\frac{p}{2q\tau}+\frac{p\lambda}{q}\right)\frac{\theta q}{p}&\geq&\left(\frac{p}{2q\tau}+\frac{(p-q)\lambda}{q}\right),\\
\label{eq:cond6}
\left(\frac{n}{2m\sigma}+\frac{\gamma}{m}\right)\frac{\theta q}{p}&\geq&\left(\frac{n}{2m\sigma}+\frac{(n-m)\gamma}{mn}\right),\\
\label{eq:cond7}
\frac{n}{2m\sigma}-\frac{\theta nq}{4\sigma mp}-\frac{n-m}{4\sigma m}&\geq&0,\\
\label{eq:cond8}
\left(\frac{p}{2q\tau}-\frac{(n-m) p }{4n\tau q}\right)\frac{\theta q}{p}&\geq&\frac{\theta}{4\tau}.
\end{eqnarray}
\normalsize


In fact, \eqref{eq:cond5} holds since~\eqref{eq:T} implies\footnote{Here and when we show \eqref{eq:cond6}, we use the simple fact that $\sqrt{a^2+b^2}\leq a+b$ when $a\geq0$ and $b\geq0$.}
\begin{eqnarray*}
\frac{p}{2q\lambda\tau}+\frac{p}{q}
\leq\frac{1}{2}\left(\frac{n}{m}+\frac{p}{q}\right)+\frac{1}{2}\left|\frac{n}{m}-\frac{p}{q}\right|+\frac{np\sqrt{\Lambda}}{mq\sqrt{n\lambda\gamma}}
=\max\left\{\frac{p}{q},\frac{n}{m}\right\}+\frac{np\sqrt{\Lambda}}{mq\sqrt{n\lambda\gamma}}
\end{eqnarray*}
so that
\begin{eqnarray*}
\left(\frac{p}{2q\tau}+\frac{(p-q)\lambda}{q}\right)/\left(\frac{p}{2q\tau}+\frac{p\lambda}{q}\right)=
1-\frac{1}{\frac{p}{2q\lambda\tau}+\frac{p}{q}}
\leq1-\frac{1}{\max\left\{\frac{p}{q},\frac{n}{m}\right\}+\frac{np\sqrt{\Lambda}}{mq\sqrt{n\lambda\gamma}}} =\frac{\theta q}{p}.
\end{eqnarray*}
Similarly, \eqref{eq:cond6} holds since~\eqref{eq:S} implies
\begin{eqnarray*}
\frac{n^2}{2m\gamma\sigma}+\frac{n}{m}
\leq\frac{1}{2}\left(\frac{n}{m}+\frac{p}{q}\right)+\frac{1}{2}\left|\frac{n}{m}-\frac{p}{q}\right|+\frac{np\sqrt{\Lambda}}{mq\sqrt{n\lambda\gamma}}
=\max\left\{\frac{p}{q},\frac{n}{m}\right\}+\frac{np\sqrt{\Lambda}}{mq\sqrt{n\lambda\gamma}}
\end{eqnarray*}
so that
\begin{eqnarray*}
\left(\frac{n}{2m\sigma}+\frac{(n-m)\gamma}{mn}\right)/\left(\frac{n}{2m\sigma}+\frac{\gamma}{m}\right)
=1-\frac{1}{\frac{n^2}{2m\gamma\sigma}+\frac{n}{m}}
\leq1-\frac{1}{\max\left\{\frac{p}{q},\frac{n}{m}\right\}+\frac{np\sqrt{\Lambda}}{mq\sqrt{n\lambda\gamma}}} =\frac{\theta q}{p}.
\end{eqnarray*}
The inequality \eqref{eq:cond7} holds because
$\frac{n}{2m\sigma}-\frac{\theta nq}{4\sigma mp}-\frac{n-m}{4\sigma m}
\geq\frac{n}{2m\sigma}-\frac{ n}{4\sigma m}-\frac{n}{4\sigma m}=0$,
where we use the fact that $\frac{\theta q}{p}\leq1$.
The inequality \eqref{eq:cond8} holds because
$\left(\frac{p}{2q\tau}-\frac{(n-m) p }{4n\tau q}\right)\frac{\theta q}{p}
=\theta\left(\frac{1}{2\tau}-\frac{(n-m) }{4n\tau}\right)
\geq\theta\left(\frac{1}{2\tau}-\frac{1}{4\tau}\right)=\frac{\theta}{4\tau}$.

Applying the four inequalities \eqref{eq:cond5}, \eqref{eq:cond6}, \eqref{eq:cond7} and \eqref{eq:cond8} to the coefficients of \eqref{thm:prop1} from Proposition~\ref{thm:mainprop_strong_asym} leads to
$\mathbb{E}_t\Delta^{(t+1)}\leq\left(\frac{\theta q}{p}\right)\Delta^{(t)}$ for any $t\geq0$, where
\begin{eqnarray}
\nonumber
\Delta^{(t)}&=&\left(\frac{p}{2q\tau}+\frac{p\lambda}{q}\right)\|x^\star-x^{(t)}\|^2
+\left(\frac{n}{2m\sigma}+\frac{\gamma}{m}\right)\|y^\star-y^{(t)}\|^2\\\label{eq:Deltat}
&&+\frac{p}{qn}\left\langle  A(x^{(t)}-x^{(t-1)}),y^{(t)}-y^\star\right\rangle
+\left(\frac{p}{2q\tau}-\frac{(n-m) p }{4n\tau q}\right)\|x^{(t)}-x^{(t-1)}\|^2.
\end{eqnarray}
\normalsize
Applying this result recursively gives $\mathbb{E}\Delta^{(t)}\leq\left(\frac{\theta q}{p}\right)^t\Delta^{(0)}$ where
\begin{eqnarray*}
\Delta^{(0)}&=&\left(\frac{p}{2q\tau}+\frac{p\lambda}{q}\right)\|x^\star-x^{(0)}\|^2
+\left(\frac{n}{2m\sigma}+\frac{\gamma}{m}\right)\|y^\star-y^{(0)}\|^2
\end{eqnarray*}
\normalsize
because $(x^{(0)},y^{(0)})=(x^{(-1)},y^{(-1)})$.

Let $\tilde I$ be a uniformly random subset of $\{1,2,\dots,n\}$ with $|\tilde I|=m$, i.e., each index in $\{1,2,\dots,n\}$ is contained in $\tilde I$ with a probability of $\frac{m}{n}$. By Jensen's inequality and \eqref{eq:Lambda}, we have
$$
\frac{m^2}{n^2}\|(A^J)^T(y^{(t)}-y^\star)\|^2=\|\mathbb{E}(A_{\tilde I}^J)^T(y^{(t)}_{\tilde I}-y^\star_{\tilde I})\|^2\leq\mathbb{E}\|(A_{\tilde I}^J)^T(y^{(t)}_{\tilde I}-y^\star_{\tilde I})\|^2
=\frac{m\Lambda}{n}\|y^{(t)}-y^\star\|^2,
$$
\normalsize
where the expectation $\mathbb{E}$ is taken over $\tilde I$. This result further implies
\begin{eqnarray}
\label{eq:randomLambda}
\|(A^J)^T(y^{(t)}-y^\star)\|^2\leq\frac{n\Lambda}{m}\|y^{(t)}-y^\star\|^2.
\end{eqnarray}
\normalsize
Note that $x^{(t)}-x^{(t-1)}$ is a sparse vector with non-zero values only in the coordinates indexed by $J$. Hence, by Young's inequality, we have
\begin{eqnarray}
\nonumber
\left\langle A(x^{(t)}-x^{(t-1)}), y^{(t)}-y^\star\right\rangle
&\geq&
-\frac{\tau}{n}\|(A^J)^T(y^{(t)}-y^\star)\|^2-\frac{\|x^{(t)}-x^{(t-1)}\|^2}{4\tau/n}\\\nonumber
&\geq&-\frac{\tau \Lambda}{m}\|y^{(t)}-y^\star\|^2-\frac{\|x^{(t)}-x^{(t-1)}\|^2}{4\tau/n}\\\label{eq:young5}
&=&-\frac{nq}{4\sigma p} \|y^{(t)}-y^\star\|^2-\frac{\|x^{(t)}-x^{(t-1)}\|^2}{4\tau/n},
\end{eqnarray}
\normalsize
where the second inequality is because of \eqref{eq:randomLambda}
and the last equality is because $\tau\sigma=\frac{nmq}{4\Lambda p}$.
Applying~\eqref{eq:young5} to the right hand side of \eqref{eq:Deltat} leads to
\begin{eqnarray}
\nonumber
\Delta^{(t)}&\geq&\left(\frac{p}{2q\tau}+\frac{p\lambda}{q}\right)\|x^\star-x^{(t)}\|^2
+\left(\frac{n}{2m\sigma}+\frac{\gamma}{m}-\frac{1}{4\sigma}\right)\|y^\star-y^{(t)}\|^2\\\nonumber
&&
+\left(\frac{p}{2q\tau}-\frac{(n-m) p }{4n\tau q}-\frac{p}{4\tau q}\right)\|x^{(t)}-x^{(t-1)}\|^2\\\nonumber
&\geq&\left(\frac{p}{2q\tau}+\frac{p\lambda}{q}\right)\|x^\star-x^{(t)}\|^2
+\left(\frac{n}{4m\sigma}+\frac{\gamma}{m}\right)\|y^\star-y^{(t)}\|^2,
\end{eqnarray}
\normalsize
where the second inequalities holds because
$
\frac{n}{2m\sigma}-\frac{1}{4\sigma}
\geq\frac{n}{2m\sigma}-\frac{n}{4m\sigma}=\frac{n}{4m\sigma}
$
and
$
\frac{p}{2q\tau}-\frac{(n-m) p }{4n\tau q}-\frac{p}{4\tau q}\geq\frac{p}{2q\tau}-\frac{p}{4\tau q}-\frac{p}{4\tau q}=0.
$
Then, the conclusion of Theorem~\ref{thm:conv_strong_astrsym_switch} can be obtained as
\begin{eqnarray*}
&&\left(\frac{p}{2q\tau}+\frac{p\lambda}{q}\right)\mathbb{E}\|x^\star-x^{(t)}\|^2
+\left(\frac{n}{4m\sigma}+\frac{\gamma}{m}\right)\mathbb{E}\|y^\star-y^{(t)}\|^2
\leq\mathbb{E}\Delta^{(t)}
\leq\left(\frac{\theta q}{p}\right)^t\Delta^{(0)}\\
&\leq&\left(1-\frac{1}{\max\left\{\frac{p}{q},\frac{n}{m}\right\}+\frac{pR\sqrt{n}}{q\sqrt{m\lambda\gamma}}}\right)^t
\left[\left(\frac{p}{2q\tau}+\frac{p\lambda}{q}\right)\|x^\star-x^{(0)}\|^2
+\left(\frac{n}{2m\sigma}+\frac{\gamma}{m}\right)\|y^\star-y^{(0)}\|^2\right].
\end{eqnarray*}
\normalsize
\end{proof}

\subsection{Convergence of objective gap}
To establish the convergence of primal-dual gap (Theorem \ref{thm:conv_strong_astrsym_switch_gap}), we define the following two functions
\begin{eqnarray}
\label{eq:Pt}
\tilde P(x)&\equiv&g(x)+\frac{1}{n}(y^\star)^TAx-\left[g(x^\star)+\frac{1}{n}(y^\star)^TAx^\star\right],\\\label{eq:Dt}
\tilde D(y)&\equiv&\frac{1}{n}y^TAx^\star-\frac{1}{n}\sum_{i=1}^n\phi_i^*(y_i)-\left[\frac{1}{n}(y^\star)^TAx^\star-\frac{1}{n}\sum_{i=1}^n\phi_i^*(y_i^\star)\right].
\end{eqnarray}
\normalsize
Note that
\begin{eqnarray}
\label{eq:tildePD}
\tilde P(x)\geq\frac{\lambda}{2}\|x-x^\star\|^2\text{ and }\tilde D(y)\leq-\frac{\gamma}{2n}\|y-y^\star\|^2
\end{eqnarray}
\normalsize
and
\begin{eqnarray}
\label{eq:tildePDandPD}
\tilde P(x)-\tilde D(y)\leq P(x)-D(y)
\end{eqnarray}
\normalsize
for any $x\in\mathbb{R}^p$ and any $y\in\mathbb{R}^n$
because of \eqref{eq:def_sp} and the strong convexity of $g(x)$ and $\frac{1}{n}\sum_{i=1}^n\phi_i^*(y_i)$.
Then, we provide the following proposition, which is the key to prove Theorem~\ref{thm:conv_strong_astrsym_switch_gap}.

\begin{proposition}
\label{thm:mainprop_nostrong}
Let $x^{(t)}$, $x^{(t+1)}$, $y^{(t)}$ and $y^{(t+1)}$ generated as in Algorithm~\ref{alg:aspdc_general} for $t=0,1,\dots$ with the parameters $\tau$ and $\sigma$ satisfying $\tau\sigma=\frac{nmq}{4p\Lambda}$. We have
\begin{eqnarray}
\nonumber
&&\left(\frac{p}{2q\tau}+\frac{(p-q)\lambda}{2q}\right)\|x^\star-x^{(t)}\|^2
+\left(\frac{n}{2m\sigma}+\frac{(n-m)\gamma}{2mn}\right)\|y^\star-y^{(t)}\|^2\\\nonumber
&&\frac{\theta}{n}\left\langle  A(x^{(t)}-x^{(t-1)}),y^{(t)}-y^\star\right\rangle+\frac{\theta\|x^{(t)}-x^{(t-1)}\|^2}{4\tau }+\frac{p-q}{q}\tilde P(x^{(t)})-\frac{n-m}{m}\tilde D(y^{(t)})\\\nonumber
&\geq&\left(\frac{p}{2q\tau}+\frac{p\lambda}{2q}\right)\mathbb{E}_t\|x^{(t+1)}-x^\star\|^2
+\left(\frac{n}{2m\sigma}+\frac{\gamma}{2m}\right)\mathbb{E}_t\|y^{(t+1)}-y^\star\|^2
\\\nonumber
&&+\left(\frac{p}{2q\tau}-\frac{(n-m) p }{4n\tau q}\right)\mathbb{E}_t\|x^{(t+1)}-x^{(t)}\|^2
+\left(\frac{n}{2m\sigma}-\frac{\theta nq}{4\sigma mp}-\frac{n-m}{4\sigma m}\right)\mathbb{E}_t\|y^{(t+1)}-y^{(t)}\|^2\\\label{eq:prop2_main}
&&+\frac{p}{nq}\mathbb{E}_t\left\langle A^T(y^{(t+1)}-y^\star), x^{(t+1)}-x^{(t)}\right\rangle
+\frac{p}{q}\mathbb{E}_t\tilde P(x^{(t+1)})-\frac{n}{m}\mathbb{E}_t\tilde D(y^{(t+1)}).
\end{eqnarray}
\normalsize
\end{proposition}
\begin{proof}
Let $x^{(t)}$, $x^{(t+1)}$ and $\by^{(t+1)}$ generated as in Algorithm~\ref{alg:aspdc_general}. By the first conclusion of Lemma \ref{thm:ZhangXiaolemma} and the tower property $\mathbb{E}_t\mathbb{E}_{t+}=\mathbb{E}_t$, for any $x\in\mathbb{R}^p$,
\begin{eqnarray}
\label{eq:lemmaoptcondexp1}
&&\left(\frac{p}{2q\tau}+\frac{(p-q)\lambda}{2q}\right)\|x-x^{(t)}\|^2
+\frac{p-q}{q}\left(g(x^{(t)})-g(x)\right)\\\nonumber
&\geq&\left(\frac{p}{2q\tau}+\frac{p\lambda}{2q}\right)\mathbb{E}\|x^{(t+1)}-x\|^2
+\frac{p}{2q\tau}\mathbb{E}\|x^{(t+1)}-x^{(t)}\|^2+\frac{p}{q}\mathbb{E}\left(g(x^{(t+1)})-g(x)\right)\\\nonumber
&&+\frac{1}{n}\mathbb{E}\left\langle A^T\by^{(t+1)},x^{(t)}+\frac{p}{q}(x^{(t+1)}-x^{(t)})-x\right\rangle.
\end{eqnarray}
\normalsize
Let $y^{(t)}$, $y^{(t+1)}$ and $\bx^{(t)}$ generated as in Algorithm~\ref{alg:aspdc_general}. By the first conclusion of Lemma \ref{thm:ZhangXiaolemma2}, we have, for any $y\in\mathbb{R}^n$,
\begin{eqnarray}
\label{eq:lemmaontcondeyn2}
&&\left(\frac{n}{2m\sigma}+\frac{(n-m)\gamma}{2mn}\right)\|y-y^{(t)}\|^2+\frac{n-m}{mn}\sum_{i=1}^n\left(\phi_i^*(y_i^{(t)})
-\phi_i^*(y_i)\right)\\\nonumber
&\geq&\left(\frac{n}{2m\sigma}+\frac{\gamma}{2m}\right)\mathbb{E}\|y^{(t+1)}-y\|^2
+\frac{n}{2m\sigma}\mathbb{E}\|y^{(t+1)}-y^{(t)}\|^2
+\frac{1}{m}\sum_{i=1}^n\mathbb{E}\left(\phi_i^*(y_i^{(t+1)})
-\phi_i^*(y_i)\right)\\\nonumber
&&-\frac{1}{n}\mathbb{E}\left\langle A\bx^{(t)},y^{(t)}+\frac{n}{m}(y^{(t+1)}-y^{(t)})-y\right\rangle
\end{eqnarray}
\normalsize
Summing up the inequalities \eqref{eq:lemmaoptcondexp1} and \eqref{eq:lemmaontcondeyn2} and setting $(x,y)=(x^\star,y^\star)$ yield
\begin{eqnarray}
\nonumber
&&\left(\frac{p}{2q\tau}+\frac{(p-q)\lambda}{2q}\right)\|x^\star-x^{(t)}\|^2
+\left(\frac{n}{2m\sigma}+\frac{(n-m)\gamma}{2mn}\right)\|y^\star-y^{(t)}\|^2\\\label{eq:pdgap_1}
&&+\frac{p-q}{q}\left(g(x^{(t)})-g(x^\star)\right)+\frac{n-m}{mn}\sum_{i=1}^n\left(\phi_i^*(y_i^{(t)})-\phi_i^*(y_i^\star)\right)\\\nonumber
&\geq&\left(\frac{p}{2q\tau}+\frac{p\lambda}{2q}\right)\mathbb{E}_t\|x^{(t+1)}-x^\star\|^2
+\frac{p}{2q\tau}\mathbb{E}_t\|x^{(t+1)}-x^{(t)}\|^2
+\left(\frac{n}{2m\sigma}+\frac{\gamma}{2m}\right)\mathbb{E}_t\|y^{(t+1)}-y^\star\|^2\\\nonumber
&&+\frac{p}{q}\mathbb{E}_t\left(g(x^{(t+1)})-g(x^\star)\right)+\frac{1}{m}\sum_{i=1}^n\mathbb{E}_t\left(\phi_i^*(y_i^{(t+1)})
-\phi_i^*(y_i^\star)\right)
+\frac{n}{2m\sigma}\mathbb{E}_t\|y^{(t+1)}-y^{(t)}\|^2\\\nonumber
&&+\frac{1}{n}\mathbb{E}_t\left\langle A^T\by^{(t+1)},x^{(t)}+\frac{p}{q}(x^{(t+1)}-x^{(t)})-x^\star\right\rangle
-\frac{1}{n}\mathbb{E}_t\left\langle A\bx^{(t)},y^{(t)}+\frac{n}{m}(y^{(t+1)}-y^{(t)})-y^\star\right\rangle.
\end{eqnarray}
\normalsize

By the definitions of $\tilde P(x^{(t)})$, $\tilde D(y^{(t)})$, $\tilde P(x^{(t+1)})$, and $\tilde D(y^{(t+1)})$,  \eqref{eq:pdgap_1} is equivalent to

\begin{eqnarray}
\nonumber
&&\left(\frac{p}{2q\tau}+\frac{(p-q)\lambda}{2q}\right)\|x^\star-x^{(t)}\|^2
+\left(\frac{n}{2m\sigma}
+\frac{(n-m)\gamma}{2mn}\right)\|y^\star-y^{(t)}\|^2\\\nonumber
&&
+\frac{p-q}{q}\tilde P(x^{(t)})
+\frac{n-m}{m}\tilde D(y^{(t)})\\\nonumber
&\geq&\left(\frac{p}{2q\tau}+\frac{p\lambda}{2q}\right)\mathbb{E}_t\|x^{(t+1)}-x^\star\|^2
+\frac{p}{2q\tau}\mathbb{E}_t\|x^{(t+1)}-x^{(t)}\|^2\
+\left(\frac{n}{2m\sigma}+\frac{\gamma}{2m}\right)\mathbb{E}_t\|y^{(t+1)}-y^\star\|^2\\\label{eq:pdgap_4}
&&+\frac{p}{q}\mathbb{E}_t\tilde P(x^{(t+1)})
+\frac{n}{m}\mathbb{E}_t\tilde D(y^{(t+1)})
-\frac{1}{n}\mathbb{E}_t\left\langle A(\bx^{(t)}-x^\star),y^{(t)}+\frac{n}{m}(y^{(t+1)}-y^{(t)})-y^\star\right\rangle
\\\nonumber
&&+\frac{1}{n}\mathbb{E}_t\left\langle A^T(\by^{(t+1)}-y^\star),x^{(t)}+\frac{p}{q}(x^{(t+1)}-x^{(t)})-x^\star\right\rangle
+\frac{n}{2m\sigma}\mathbb{E}_t\|y^{(t+1)}-y^{(t)}\|^2.
\end{eqnarray}
\normalsize
which, together with \eqref{eq:prop3_4}, implies
\begin{eqnarray}
\nonumber
&&\left(\frac{p}{2q\tau}+\frac{(p-q)\lambda}{2q}\right)\|x^\star-x^{(t)}\|^2
+\left(\frac{n}{2m\sigma}+\frac{(n-m)\gamma}{2mn}\right)\|y^\star-y^{(t)}\|^2\\\nonumber
&&
+\frac{p-q}{q}\tilde P(x^{(t)})
-\frac{n-m}{m}\tilde D(y^{(t)})\\\nonumber
&\geq&\left(\frac{p}{2q\tau}+\frac{p\lambda}{2q}\right)\mathbb{E}_t\|x^{(t+1)}-x^\star\|^2
+\frac{p}{2q\tau}\mathbb{E}_t\|x^{(t+1)}-x^{(t)}\|^2
+\frac{p}{q}\mathbb{E}_t\tilde P(x^{(t+1)})
-\frac{n}{m}\mathbb{E}_t\tilde D(y^{(t+1)})\\\nonumber
&&+\left(\frac{n}{2m\sigma}+\frac{\gamma}{2m}\right)\mathbb{E}_t\|y^{(t+1)}-y^\star\|^2
+\frac{n}{2m\sigma}\mathbb{E}_t\|y^{(t+1)}-y^{(t)}\|^2\\\nonumber
&&
+\frac{p}{nq}\left\langle A^T(y^{(t+1)}-y^\star), x^{(t+1)}-x^{(t)}\right\rangle
-\frac{\theta}{n}\left\langle  A(x^{(t)}-x^{(t-1)}),y^{(t)}-y^\star\right\rangle\\\nonumber
&&-\left(\frac{\theta n q}{4\sigma mp}+\frac{n-m}{4\sigma m}\right)\|y^{(t+1)}-y^{(t)}\|^2-\frac{\theta\|x^{(t)}-x^{(t-1)}\|^2}{4\tau }-\frac{(n-m) p \|x^{(t+1)}-x^{(t)}\|^2}{4\tau n q}.
\end{eqnarray}
\normalsize
The conclusion of the proposition is obtained by organizing the terms of the inequality above.
\end{proof}

Based on Proposition~\ref{thm:mainprop_nostrong}, we now can prove Theorem~\ref{thm:conv_strong_astrsym_switch_gap}.

\begin{proof}[\textbf{Theorem~\ref{thm:conv_strong_astrsym_switch_gap}}]
We first show that the following inequalities are satisfied according to the choice for $\theta$ in \eqref{eq:thetagap} and the choices for $\tau$ and $\sigma$ in~\eqref{eq:thetatausigma}.
\begin{eqnarray}
\label{eq:cond9}
\left(\frac{p}{2q\tau}+\frac{p\lambda}{2q}\right)\frac{\theta q}{p}
&\geq&\left(\frac{p}{2q\tau}+\frac{(p-q)\lambda}{2q}\right),\\
\label{eq:cond10}
\left(\frac{n}{2m\sigma}+\frac{\gamma}{2m}\right)\frac{\theta q}{p}
&\geq&\left(\frac{n}{2m\sigma}+\frac{(n-m)\gamma}{2mn}\right),\\
\label{eq:cond11}
\frac{n}{2m\sigma}-\frac{\theta nq}{4\sigma mp}-\frac{n-m}{4\sigma m}&\geq&0,\\
\label{eq:cond12}
\left(\frac{p}{2q\tau}-\frac{(n-m) p }{4n\tau q}\right)\frac{\theta q}{p}&\geq&\frac{\theta}{4\tau},\\
\label{eq:cond13}
\frac{\theta q}{p}&\geq&\frac{p-q}{p},\\
\label{eq:cond14}
\frac{\theta q}{p}&\geq&\frac{n-m}{n}.
\end{eqnarray}
\normalsize

Since $\tau$ and $\sigma$ still satisfy \eqref{eq:thetatausigma} as in Theorem \ref{thm:conv_strong_astrsym_switch}, \eqref{eq:T} and \eqref{eq:S} are still satisfied. Therefore, we have
\begin{eqnarray*}
\frac{p}{q\lambda\tau}+\frac{p}{q}\leq\frac{n}{m}+\left|\frac{n}{m}-\frac{p}{q}\right|
+\frac{2np\sqrt{\Lambda}}{mq\sqrt{n\lambda\gamma}}
=2\max\left\{\frac{p}{q},\frac{n}{m}\right\}
+\frac{2np\sqrt{\Lambda}}{mq\sqrt{n\lambda\gamma}}
\end{eqnarray*}
\normalsize
according to \eqref{eq:T} so that, by the new choice for $\theta$ in \eqref{eq:thetagap},
\begin{eqnarray*}
\left(\frac{p}{2q\tau}+\frac{(p-q)\lambda}{2q}\right)/\left(\frac{p}{2q\tau}+\frac{p\lambda}{2q}\right)=
1-\frac{1}{\frac{p}{q\lambda\tau}+\frac{p}{q}}
\leq1-\frac{1}{2\max\left\{\frac{p}{q},\frac{n}{m}\right\}+\frac{2np\sqrt{\Lambda}}{mq\sqrt{n\lambda\gamma}}} =\frac{\theta q}{p}.
\end{eqnarray*}
\normalsize
Similarly, according to~\eqref{eq:S}, we have
\begin{eqnarray*}
\frac{n^2}{m\gamma\sigma}+\frac{n}{m}\leq\frac{p}{q}+\left|\frac{n}{m}-\frac{p}{q}\right|
+\frac{2np\sqrt{\Lambda}}{mq\sqrt{n\lambda\gamma}}
=2\max\left\{\frac{p}{q},\frac{n}{m}\right\}+\frac{2np\sqrt{\Lambda}}{mq\sqrt{n\lambda\gamma}}
\end{eqnarray*}
\normalsize
so that
\begin{eqnarray*}
\left(\frac{n}{2m\sigma}+\frac{(n-m)\gamma}{2mn}\right)/\left(\frac{n}{2m\sigma}+\frac{\gamma}{2m}\right)
=1-\frac{1}{\frac{n^2}{m\gamma\sigma}+\frac{n}{m}}
\leq1-\frac{1}{2\max\left\{\frac{p}{q},\frac{n}{m}\right\}
+\frac{2np\sqrt{\Lambda}}{mq\sqrt{n\lambda\gamma}}} =\frac{\theta q}{p}.
\end{eqnarray*}
\normalsize
Therefore, we have shown that \eqref{eq:cond9} and \eqref{eq:cond10} are satisfied.
The inequality \eqref{eq:cond11} holds because
\begin{eqnarray*}
\frac{n}{2m\sigma}-\frac{\theta nq}{4\sigma mp}-\frac{n-m}{4\sigma m}
\geq\frac{n}{2m\sigma}-\frac{ n}{4\sigma m}-\frac{n}{4\sigma m}=0,
\end{eqnarray*}
\normalsize
where we use the fact that $\frac{\theta q}{p}\leq1$. The inequality \eqref{eq:cond12} holds because
\begin{eqnarray*}
\left(\frac{p}{2q\tau}-\frac{(n-m) p }{4n\tau q}\right)\frac{\theta q}{p}
=\theta\left(\frac{1}{2\tau}-\frac{(n-m) }{4n\tau}\right)
\geq\theta\left(\frac{1}{2\tau}-\frac{1}{4\tau}\right)=\frac{\theta}{4\tau}.
\end{eqnarray*}
\normalsize
The inequalities \eqref{eq:cond13} and \eqref{eq:cond14} hold because
\begin{eqnarray*}
\max\left\{\frac{p-q}{p},\frac{n-m}{n}\right\}
\leq1-\frac{1}{2\max\left\{\frac{p}{q},\frac{n}{m}\right\}+\frac{2np\sqrt{\Lambda}}{mq\sqrt{n\lambda\gamma}}} =\frac{\theta q}{p}.
\end{eqnarray*}
\normalsize
Recall that $\tilde P(x)\geq0$ and $\tilde D(y)\leq0$ for any $x\in\mathbb{R}^p$ and any $y\in\mathbb{R}^n$. Therefore, applying the six inequalities \eqref{eq:cond9}, \eqref{eq:cond10}, \eqref{eq:cond11}, \eqref{eq:cond12}, \eqref{eq:cond13} and \eqref{eq:cond14} to the coefficients of \eqref{eq:prop2_main} from Proposition~\ref{thm:mainprop_nostrong} leads to
$\mathbb{E}_t\Delta^{(t+1)}\leq\left(\frac{\theta q}{p}\right)\Delta^{(t)}$ for any $t\geq0$, where
\begin{eqnarray}
\nonumber
\Delta^{(t)}&=&\left(\frac{p}{2q\tau}+\frac{p\lambda}{2q}\right)\|x^{(t)}-x^\star\|^2
+\left(\frac{n}{2m\sigma}+\frac{\gamma}{2m}\right)\|y^{(t)}-y^\star\|^2
+\frac{p}{q}\tilde P(x^{(t)})-\frac{n}{m}\tilde D(y^{(t)})\\\label{eq:gapDelta}
&&+\left(\frac{p}{2q\tau}-\frac{(n-m) p }{4n\tau q}\right)\|x^{(t)}-x^{(t-1)}\|^2+\frac{p}{nq}\left\langle A^T(y^{(t)}-y^\star), x^{(t)}-x^{(t-1)}\right\rangle.
\end{eqnarray}
\normalsize
Applying this result recursively gives $\mathbb{E}\Delta^{(t)}\leq\left(\frac{\theta q}{p}\right)^t\Delta^{(0)}$, where
\begin{eqnarray}
\label{eq:thm2delta0}
\Delta^{(0)}&=&\left(\frac{p}{2q\tau}+\frac{p\lambda}{2q}\right)\|x^{(0)}-x^\star\|^2
+\left(\frac{n}{2m\sigma}+\frac{\gamma}{2m}\right)\|y^{(0)}-y^\star\|^2\\\nonumber
&&
+\frac{p}{q}\tilde P(x^{(0)})-\frac{n}{m}\tilde D(y^{(0)})
\end{eqnarray}
\normalsize
because $(x^{(0)},y^{(0)})=(x^{(-1)},y^{(-1)})$.
Applying~\eqref{eq:young5} to the right hand side of \eqref{eq:gapDelta} leads to
\begin{eqnarray}
\nonumber
\Delta^{(t)}&\geq&\left(\frac{p}{2q\tau}+\frac{p\lambda}{2q}\right)\|x^\star-x^{(t)}\|^2
+\left(\frac{n}{2m\sigma}+\frac{\gamma}{2m}-\frac{1}{4\sigma}\right)\|y^\star-y^{(t)}\|^2\\\nonumber
&&+\left(\frac{p}{2q\tau}-\frac{(n-m) p }{4n\tau q}-\frac{p}{4\tau q}\right)\|x^{(t)}-x^{(t-1)}\|^2
+\frac{p}{q}\tilde P(x^{(t)})-\frac{n}{m}\tilde D(y^{(t)})\\\label{eq:gap Delta_lb}
&\geq&\frac{p}{q}\tilde P(x^{(t)})-\frac{n}{m}\tilde D(y^{(t)}).
\end{eqnarray}
\normalsize
Combining~\eqref{eq:thm2delta0}, $\mathbb{E}\Delta^{(t)}\leq\left(\frac{\theta q}{p}\right)^t\Delta^{(0)}$ and~\eqref{eq:gap Delta_lb} together, we obtain
\begin{eqnarray}
\label{eq:thm2eq1}
\min\left\{\frac{p}{q},\frac{n}{m}\right\}\mathbb{E}\left(\tilde P(x^{(t)})-\tilde D(y^{(t)})\right)\leq\frac{p\mathbb{E}\tilde P(x^{(t)})}{q}-\frac{n\mathbb{E}\tilde D(y^{(t)})}{m}\leq\mathbb{E}\Delta^{(t)}\leq\left(\frac{\theta q}{p}\right)^t\Delta^{(0)}.
\end{eqnarray}
\normalsize
In the next, we will establish the relationship between $\tilde P(x^{(t)})-\tilde D(y^{(t)})$ and the actual primal-dual objective gap $P(x^{(t)})-D(y^{(t)})$.

Because $\frac{1}{n}\sum_{i=1}^n\phi_i^*(y_i)$ is a $\frac{\gamma}{n}$-strong convex function of $y$, according to Theorem 1 in \cite{Nes:05}, the function defined as
\begin{eqnarray}
\label{eq:hatP}
\hat P(x)\equiv\max_{y\in\mathbb{R}^n}\left\{\frac{1}{n}y^TAx-\frac{1}{n}\sum_{i=1}^n\phi_i^*(y_i)\right\}
\end{eqnarray}
is a convex and smooth function of $x$. Moreover, its gradient $\nabla \hat P(x)$ is Lipschitz continuous with a Lipschitz constant of $\frac{n\|A\|^2}{n^2\gamma}= \frac{\|A\|^2}{n\gamma}$ and $\nabla \hat P(x^{\star})=\frac{1}{n}A^Ty^\star$. 
As a result, we have
\begin{eqnarray}
\label{eq:LiphatP}
\hat P(x^{(t)})\leq\hat P(x^{\star})+\left\langle\nabla \hat P(x^{\star}),x^{(t)}-x^\star\right\rangle+\frac{\|A\|^2}{2n\gamma}\|x^{(t)}-x^\star\|^2
\end{eqnarray}
\normalsize

According to the definition of the primal and dual objective functions \eqref{eq:erm} and \eqref{eq:erm_dual} and their relationship with the saddle-point problem \eqref{eq:sdp}, we have
\begin{eqnarray}
\nonumber
&&P(x^{(t)})=\max_{y\in\mathbb{R}^n}\left\{ g(x^{(t)})+\frac{1}{n}y^TAx^{(t)}-\frac{1}{n}\sum_{i=1}^n\phi_i^*(y_i)\right\}
=g(x^{(t)})+\hat P(x^{(t)})\\\nonumber
&\leq& g(x^{(t)})+\max_{y\in\mathbb{R}^n}\left\{\frac{1}{n}y^TAx^\star-\frac{1}{n}\sum_{i=1}^n\phi_i^*(y_i)\right\}
+\frac{1}{n}(y^\star)^TA(x^{(t)}-x^\star)+\frac{\|A\|^2}{2n\gamma}\|x^{(t)}-x^\star\|^2\\\label{eq:thm2eq2}
&=& g(x^{(t)})-\frac{1}{n}\sum_{i=1}^n\phi_i^*(y_i^\star)
+\frac{1}{n}(y^\star)^TAx^{(t)}+\frac{\|A\|^2}{2n\gamma}\|x^{(t)}-x^\star\|^2,
\end{eqnarray}
\normalsize
where the inequality is due to \eqref{eq:hatP} and \eqref{eq:LiphatP} and the last equality is due to \eqref{eq:def_sp}.

Similarly, because $g(x)$ is a $\lambda$-strong convex function of $x$, according to Theorem 1 in \cite{Nes:05} again, the function defined as
\begin{eqnarray}
\label{eq:hatD}
\hat D(y)\equiv\min_{x\in\mathbb{R}^p}\left\{\frac{1}{n}y^TAx+g(x)\right\}
\end{eqnarray}
is a concave and smooth function of $y$. Moreover, its gradient $\nabla \hat D(x)$ is Lipschitz continuous with a Lipschitz constant of $\frac{\|A\|^2}{n^2\lambda}$ and $\nabla \hat D(y^{\star})=\frac{1}{n}Ax^\star$. 
As a result, we have
\begin{eqnarray}
\label{eq:LiphatD}
\hat D(y^{(t)})\geq\hat D(y^{\star})+\left\langle\nabla \hat D(y^{\star}),y^{(t)}-y^\star\right\rangle-\frac{\|A\|^2}{2n^2\gamma}\|y^{(t)}-y^\star\|^2
\end{eqnarray}
\normalsize

With a derivation similar to \eqref{eq:thm2eq2}, we can show
\begin{eqnarray}
\nonumber
&&D(y^{(t)})=\min_{x\in\mathbb{R}^p}\left\{ g(x)+\frac{1}{n}(y^{(t)})^TAx-\frac{1}{n}\sum_{i=1}^n\phi_i^*(y_i^{(t)})\right\}
=\hat D(y^{(t)})-\frac{1}{n}\sum_{i=1}^n\phi_i^*(y_i^{(t)})\\\nonumber
&\geq&-\frac{1}{n}\sum_{i=1}^n\phi_i^*(y_i^{(t)})+\min_{x\in\mathbb{R}^p}\left\{ g(x)+\frac{1}{n}(y^\star)^TAx\right\}+\frac{1}{n}(y^{(t)}-y^\star)^TAx^\star-\frac{\|A\|^2}{2\lambda n^2}\|y^{(t)}-y^\star\|^2\\\label{eq:thm2eq3}
&=&-\frac{1}{n}\sum_{i=1}^n\phi_i^*(y_i^{(t)})+g(x^\star)+\frac{1}{n}(y^{(t)})^TAx^\star-\frac{\|A\|^2}{2\lambda n^2}\|y^{(t)}-y^\star\|^2.
\end{eqnarray}
\normalsize

Combining \eqref{eq:thm2eq2} and \eqref{eq:thm2eq3} and using the definitions \eqref{eq:Pt} and \eqref{eq:Dt}, we obtain
\begin{eqnarray}
\label{eq:thm2eq4}
P(x^{(t)})-D(y^{(t)})-\frac{\|A\|^2}{2n\gamma}\|x^{(t)}-x^\star\|^2-\frac{\|A\|^2}{2\lambda n^2}\|y^{(t)}-y^\star\|^2
\leq
\tilde P(x^{(t)})-\tilde D(y^{(t)}).
\end{eqnarray}
\normalsize

Applying \eqref{eq:thm2eq4} to the left hand side of \eqref{eq:thm2eq1}, we can show
\begin{eqnarray}
\label{eq:thm2eq5}
&&\min\left\{\frac{p}{q},\frac{n}{m}\right\}\mathbb{E}\left(P(x^{(t)})-D(y^{(t)})\right)\\\nonumber
&\leq&\left(\frac{\theta q}{p}\right)^t\Delta^{(0)}
+\min\left\{\frac{p}{q},\frac{n}{m}\right\}\mathbb{E}\left(\frac{\|A\|^2}{2n\gamma}\|x^{(t)}-x^\star\|^2+\frac{\|A\|^2}{2\lambda n^2}\|y^{(t)}-y^\star\|^2\right)
\end{eqnarray}
\normalsize

The property \eqref{eq:tildePD} of $\tilde P$ and $\tilde D$ and \eqref{eq:gap Delta_lb} imply
\begin{eqnarray}
\nonumber
&&\frac{\lambda p\mathbb{E}\|x^\star-x^{(t)}\|^2}{2q}+\frac{\gamma \mathbb{E}\|y^\star-y^{(t)}\|^2}{2m}\\\label{eq:thm2eq6}
&\leq&\frac{p\mathbb{E}\tilde P(x^{(t)})}{q}-\frac{n\mathbb{E}\tilde D(y^{(t)})}{m}
\leq\mathbb{E}\Delta^{(t)}\leq\left(\frac{\theta q}{p}\right)^t\Delta^{(0)}
\end{eqnarray}
\normalsize
which, together with \eqref{eq:thm2eq5}, further implies
\begin{eqnarray}
\nonumber
&&\min\left\{\frac{p}{q},\frac{n}{m}\right\}\mathbb{E}\left(P(x^{(t)})-D(y^{(t)})\right)\\
&\leq&\bigg\{1+\frac{\min\left\{\frac{p}{q},\frac{n}{m}\right\}\max\left\{\frac{\|A\|^2}{n\gamma},\frac{\|A\|^2}{\lambda n^2}\right\}}{\min\left\{\frac{\lambda p}{q},\frac{\gamma }{m}\right\}}\bigg\}\left(\frac{\theta q}{p}\right)^t\Delta^{(0)}.
\label{eq:thm2eq7}
\end{eqnarray}
\normalsize
It is from \eqref{eq:tildePDandPD} that
\begin{eqnarray*}
\frac{p}{q}\tilde P(x^{(0)})-\frac{n}{m}\tilde D(y^{(0)})
\leq\max\left\{\frac{p}{q},\frac{n}{m}\right\}\left(\tilde P(x^{(0)})-\tilde D(y^{(0)})\right)
\leq\max\left\{\frac{p}{q},\frac{n}{m}\right\}\left(P(x^{(0)})-D(y^{(0)})\right).
\end{eqnarray*}
\normalsize
Using this inequality and the definition \eqref{eq:thm2delta0} of $\Delta^{(0)}$,
we obtain
$$
\Delta^{(0)}\leq\left(\frac{p}{2q\tau}+\frac{p\lambda}{2q}\right)\|x^{(0)}-x^\star\|^2
+\left(\frac{n}{2m\sigma}+\frac{\gamma}{2m}\right)\|y^{(0)}-y^\star\|^2
+\max\left\{\frac{p}{q},\frac{n}{m}\right\}\left(P(x^{(0)})-D(y^{(0)})\right)
$$
\normalsize
Applying this inequality to the right hand size of \eqref{eq:thm2eq7}, we obtain the conclusion of Theorem~\ref{thm:conv_strong_astrsym_switch_gap} by the new definition~\eqref{eq:thetagap} of $\theta$.
\end{proof}

\bibliography{pdcg}
\bibliographystyle{plainnat}


\end{document}